\documentclass[twoside]{article}
\usepackage[preprint]{aistats2026}
\usepackage[utf8]{inputenc} 
\usepackage[T1]{fontenc}    
\usepackage{hyperref}       
\usepackage{url}            
\usepackage{booktabs}       
\usepackage{amsfonts}       
\usepackage{nicefrac}       
\usepackage[protrusion=false,expansion=true]{microtype}      
\usepackage{xcolor}         

\usepackage{hyperref}
\hypersetup{
    colorlinks=true,
    linkcolor=blue,
    filecolor=magenta,      
    urlcolor=green,
    citecolor=magenta,
    pdfpagemode=FullScreen,
    }

\usepackage{graphicx}
\usepackage{subfigure}
\usepackage{booktabs} 
\usepackage{xspace}

\usepackage{amsmath}
\usepackage{amssymb}
\usepackage{mathtools}
\usepackage{amsthm}
\usepackage[capitalize,noabbrev]{cleveref}
\crefname{equation}{Eq.}{Eqs.}
\usepackage{listings}
\usepackage{xcolor} 
\usepackage{enumitem}
\usepackage{fancyhdr}
\usepackage{algorithm}
\usepackage{algorithmic}

\lstdefinestyle{mystyle}{
    backgroundcolor=\color{gray!10},   
    commentstyle=\color{green},      
    keywordstyle=\color{black},       
    numberstyle=\tiny\color{gray},   
    stringstyle=\color{red},         
    basicstyle=\ttfamily\footnotesize, 
    breakatwhitespace=false,         
    breaklines=true,                 
    captionpos=b,                    
    keepspaces=true,                 
    numbers=none,                    
    numbersep=1pt,                   
    showspaces=false,                
    showstringspaces=false,          
    showtabs=false,                  
    tabsize=4                        
}
\usepackage{titlesec}
\usepackage{chngcntr}  
\usepackage{comment}

\titleformat{\subsubsection}[block]{\bfseries}{\thesubsubsection}{1em}{}

\theoremstyle{plain}
\newtheorem{theorem}{Theorem}

\newtheorem{lemma}{Lemma}

\newtheorem{corollary}{Corollary}

\theoremstyle{definition}
\newtheorem{definition}{Definition}
\newtheorem{assumption}{Assumption}

\theoremstyle{remark}
\newtheorem{remark}{Remark}
\DeclareMathOperator*{\argmax}{arg\,max}

\newcommand{\Mod}[1]{\ (\mathrm{mod}\ #1)}
\newcommand{\s}{&\mathrel{\phantom{=}}}
\def\term{NVLDB\xspace}
\def\util{u\xspace}
\def\link{g\xspace}
\def\arm{k\xspace}
\def\gauss{\mathcal{N}\xspace}
\def\loss{\mathcal{L}}

\def\stt{k_{t,1}\xspace}
\def\ndt{k_{t,2}\xspace}
\def\sti{k_{i,1}\xspace}
\def\ndi{k_{i,2}\xspace}
\def\opt{k_{t}^{*}\xspace}

\def\bigo{\mathcal{O}\xspace}
\def\bigol{\widetilde{\mathcal{O}}\xspace}
\def\nn{f\xspace}
\def\th{\theta\xspace}
\def\ww{W\xspace}
\def\rep{\phi\xspace}
\def\dim{d\xspace}
\def\len{L\xspace}
\def\wid{m\xspace}

\def\stdt{\sigma_t\xspace}

\def\estdt{\widehat{\sigma}_t\xspace}
\def\estdi{\widehat{\sigma}_i\xspace}
\def\wt{\zeta_{t}\xspace}
\def\wi{\zeta_{i\xspace}}

\def\cwct{a_{t}\xspace} 

\def\cs{\mathcal{C}\xspace} 
\def\xs{\mathcal{X}\xspace} 
\def\xb{G\xspace}
\def\thb{J\xspace}

\def\utilu{\alpha_{\util}\xspace}

\def\linkb{\beta_{\link}\xspace}
\def\linku{\alpha_{\link}\xspace}
\def\linkl{\kappa_{\link}\xspace}

\def\specmin{\lambda_{m}\xspace}
\def\repl{\kappa_{\rep}\xspace}
\def\ntk{\mathbf{H}\xspace}
\def\gs{n\xspace}
\def\lr{\eta\xspace}
\def\ff{\mathcal{Z}}
\def\gg{\mathcal{G}}
\def\epi{\mathcal{H}}
\def\lt{\left}
\def\rt{\right}
\def\no{\nonumber}
\def\diff{\|\mathbf{\util} - \widetilde{\mathbf{\util}}\|_{\ntk^{-1}}\xspace}

\usepackage[round]{natbib}

\begin{document}

%
\runningtitle{Neural Variance-aware Dueling Bandits with Deep Representation and Shallow Exploration}

\runningauthor{Oh, Park, Paik}

\twocolumn[

\aistatstitle{Neural Variance-aware Dueling Bandits \\ with Deep Representation
and Shallow Exploration}

\aistatsauthor{ Youngmin Oh$^{\ast}$ \And Jinje Park \And Taejin Paik }

\aistatsaddress{ InfiniTree \And Samsung Electronics \And Samsung Electronics } ]
\renewcommand{\thefootnote}{\fnsymbol{footnote}}
\footnotetext[1]{Corresponding author: \texttt{youngmin0.oh@gmail.com}}
\renewcommand{\thefootnote}{\arabic{footnote}}
\begin{abstract}


We introduce the first variance-aware algorithms for contextual dueling bandits that leverage shallow exploration strategies with neural networks for nonlinear utility approximation. A key theoretical challenge is the absence of a closed-form estimator, which led prior work to require an extremely large network width \( m \) (i.e., \( m = \widetilde{\Omega}(T^{14}) \)). We address this constraint with a novel analytical approach that combines iterative self-improvement with spectral analysis. Our analysis significantly reduces the network width requirement to \( m = \widetilde{\Omega}(T^{6}) \), and shows that our algorithms achieve a sublinear regret of 
$
\widetilde{\mathcal{O}}\left(d\sqrt{\sum_{t=1}^{T} \sigma_t^2} + \sqrt{dT}\right)
$
under both UCB and TS frameworks. Empirical results show that the proposed algorithms are not only computationally efficient and exhibit sublinear regret in practical settings, but also achieve state-of-the-art performance on both synthetic and real-world tasks.\footnote{Official code is available at \url{https://github.com/youngmin0oh/NVLDB-AISTATS2026}.}

\end{abstract}


%
%


\section{INTRODUCTION}

The contextual dueling bandits (CDBs) ~\citep{dudik2015contextual,firstduel} extend the standard contextual multi-armed bandit (CMAB) framework to handle pairwise preference feedback instead of numerical rewards. In each round, a learner observes a context, selects two arms to compare, and receives a preference outcome. This outcome is commonly sampled by a Bernoulli distribution, where the probability is a function of the arms' latent utilities under models such as Bradley--Terry--Luce~\citep{btmodel,btmodel2}.

Many existing studies on the CDB framework assume a linear utility function~\citep{lst, saha2021adversarial, saha2021optimal, feelgood, bui2024variance}.
However, the linearity assumption imposes a significant limitation on its ability to model real-world preferences, as real-world user preferences often exhibit complex, nonlinear interactions among features that a simple linear model cannot fully capture.

To address this limitation, \citet{ndb} recently proposed neural dueling bandit algorithms based on the Upper Confidence Bound (UCB) and Thompson Sampling (TS) frameworks. These algorithms employ a neural network to approximate a general nonlinear utility function and are proven to achieve sublinear cumulative average regret when the network width is sufficiently large. Despite these theoretical advancements, their approach requires constructing a Gram matrix from the gradients of all learnable parameters, which incurs substantial computational overhead and poses practical challenges for large-scale applications.

\begin{table*}[t]
\centering
\caption{Comparison of Bandit Algorithms: Regret, Assumptions, and Explicit Form of $\theta_t$}\vspace{0.05in}
\label{tab:bandit}
\begin{tabular}{lccc}
\toprule
\footnotesize \textbf{Name} & \multicolumn{1}{c}{\footnotesize \textbf{Regret}} & \textbf{\footnotesize Assumption of $\wid$} & \textbf{\footnotesize Explicit form } \\
 & \multicolumn{1}{c}{(\footnotesize without the term $\wid$)} & 
 \multicolumn{1}{c}{\footnotesize for sublinear regrets}& {\footnotesize of $\theta$} \\
\midrule
\footnotesize Neural Bandits~\citep{zhou2020neural} & \footnotesize $\bigol(d\sqrt{T})$ &\footnotesize $m = \widetilde{\Omega}(T^7)$ &\footnotesize None \vspace{0.05in}\\
\footnotesize Neural Linear Bandits~\citep{shallow} &\footnotesize $\bigol(d\sqrt{T})$ & \footnotesize  $m = \widetilde{\Omega}(T^3)$ & $ \footnotesize  \hat{\theta}_{t-1} = \mathbf{V}_{t-1}^{-1} \mathbf{b}_{t-1}$ \\
\footnotesize Neural-$\sigma^2$-LinearUCB~\citep{bui2024variance} & \footnotesize $\bigol\lt(\sqrt{dT}+d\sqrt{\sum_{t=1}^{T}\sigma_t^2}\rt)$ &\footnotesize $m = \widetilde{\Omega}(T^3)$  & \footnotesize  $\hat{\theta}_{t-1} = \mathbf{V}_{t-1}^{-1} \mathbf{b}_{t-1}$ \\
\midrule
\footnotesize Neural Dueling Bandits~\citep{ndb} &\footnotesize $\bigol(d\sqrt{T})$ &\footnotesize $m = \widetilde{\Omega}(T^{14})$ & \footnotesize None \vspace{0.05in}\\
\footnotesize Our Method (variance-agnostic) &\footnotesize $\bigol(d\sqrt{T})$ &\footnotesize $m = \widetilde{\Omega}(T^6)$ & \footnotesize None \\
\footnotesize \footnotesize Our Method (variance-aware) & \footnotesize $\bigol\lt(\sqrt{dT}+d\sqrt{\sum_{t=1}^{T}\sigma_t^2}\rt)$ &\footnotesize $m = \widetilde{\Omega}(T^6)$ & \footnotesize None \\
\bottomrule
\end{tabular}
\end{table*}

In the CMAB literature, computational challenges have been mitigated by adopting shallow exploration strategies, which construct Gram matrices using only the gradients of the final layer~\citep{bui2024variance, shallow}. Separately, variance-aware methods have been actively studied in linear CMABs and dueling bandits~\citep{variance-aware, zhou2022computationally, zhou2021nearlyvar, zhang2021improvedvar, zhang2022horizon} and were recently extended to neural bandits~\citep{bui2024variance}. However, to the best of our knowledge, the combination of variance-aware approaches with shallow exploration remains unexplored in the context of dueling bandits.

{\color{black}A key observation motivating variance-awareness in the dueling setting is that \emph{harder comparisons produce noisier feedback}: the Bernoulli variance $\link(\Delta\util_t)(1-\link(\Delta\util_t))$ is maximized when $\Delta\util_t \approx 0$---precisely when the comparison is hardest. Our variance-based weighting acts as an \emph{information filter}, down-weighting uncertain observations from small-margin pairs to achieve tighter, variance-dependent regret bounds.}

Motivated by these gaps, we propose a method we term \textbf{Neural Variance-Aware Linear Dueling Bandits} (\term). Our approach constructs the Gram matrix using only the gradients of the final layer's parameters, thereby improving computational efficiency compared to the method of \citet{ndb}. To achieve variance-awareness, \term estimates the variance of the underlying Bernoulli distribution, using the neural network as a nonlinear utility approximator. Furthermore, \term supports various arm selection strategies---including one asymmetric and two symmetric variants---within both the UCB and TS frameworks.

Under a significantly weaker condition on the neural network width $\wid$ than that required by previous work~\citep{ndb}, we prove that, under standard assumptions, the expected cumulative average regret for each of our arm selection strategies is bounded by
$
\bigol\left(d \sqrt{\sum_{t=1}^T \sigma_t^2} + d\sqrt{T}\right).
$
Here, $d$ is the contextual dimension, and $\sigma_t^2$ represents the variance of the Bernoulli preference model at round $t$.
We also establish that variance-agnostic variants of our algorithm achieve a sublinear regret bound of $\bigol\left( d \sqrt{T} \right)$ under the same standard assumptions.

A fundamental challenge that distinguishes the theoretical analysis of CDB from the CMAB setting is the \textbf{absence of a closed-form solution for the parameter estimate $\hat{\theta}_{t-1}$}. 
Analyses in the CMABs under shallow exploration literature (e.g., \citep{bui2024variance}) hinge on an explicit estimator, typically of the form $\hat{\theta}_{t-1} = \mathbf{V}_{t-1}^{-1} \mathbf{b}_{t-1}$. This closed-form results in a time-independent bound on the estimator's norm, $\|\hat{\theta}_{t-1}\|_2$, which is crucial for proving sublinear regret of neural bandits with a moderately sized neural network.

In the dueling bandit setting, however, no such solution exists, even with shallow exploration.
The absence of this analytical expression forced prior work~\citep{ndb} to assume an extremely large neural network width of $m = \widetilde{\Omega}\lt(T^{14}\rt)$ to guarantee sublinear cumulative regret. 
Such a strong requirement is inconsistent with results from related bandit literature~\citep{zhou2020neural, bui2024variance, shallow}.


In this paper, we overcome this critical limitation by introducing a novel analytical technique that combines a bootstrap argument~\citep{gilbarg1977elliptic} (iterative self-improvement) with spectral analysis to prove that the estimator norm $\|\hat{\theta}_{t-1}\|_2$ is bounded by a time-independent constant, even without an explicit form for the estimator, with a network width of $m = \widetilde{\Omega}(T^6)$. Consequently, we prove \textbf{sublinear regret} for \textbf{both} Upper Confidence Bound (UCB) and Thompson Sampling (TS) arm selections under standard assumptions that are customary in the neural (linear) bandit literature~\citep{shallow, bui2024variance}. Although the width condition is a stronger assumption compared to neural bandits under shallow exploration~\citep{bui2024variance, shallow}, the gap is originated from the absence of the closed form of $\th_{t-1}$ even under the shallow exploration in the case of dueling bandits. Nevertheless, this represents a significant improvement over the $m=\widetilde{\Omega}(T^{14})$ required by previous work~\citep{ndb}. A brief comparison with prior results is presented in Table~\ref{tab:bandit}. We demonstrate that these results are obtained under standard assumptions that are customary in the neural (linear) bandit literature. We provide further discussion in Section~\ref{sec:ra}, and defer the full proofs to the supplementary material.

Our contributions are summarized as follows:
\begin{enumerate}
\item We propose \term, the first framework for neural variance-aware linear dueling bandits with a shallow exploration strategy that significantly reduces computational overhead compared to full-network approaches~\citep{ndb}.

\item We provide a rigorous theoretical analysis proving that, under standard assumptions, \term achieves sublinear cumulative regret. This guarantee holds for both its variance-aware and variance-agnostic variants.

\item Our analysis establishes a required network width of $m=\widetilde{\Omega}(T^6)$. This is a substantial improvement over the $m=\widetilde{\Omega}(T^{14})$ required by prior work~\citep{ndb} and brings the requirement in line with that of standard neural bandits with shallow exploration.

\item Extensive experiments on synthetic and real-world datasets demonstrate that \term empirically achieves sublinear regret and consistently outperforms existing contextual dueling bandit methods.
\end{enumerate}

\section{RELATED RESULTS}
\paragraph{Linear Contextual Dueling Bandits.}
CDBs with a linear utility function have been the subject of active study in recent years~\citep{lst, saha2021optimal, variance-aware,  feelgood}. The UCB-based approaches remain effective in this formulation. Indeed, \citet{lst} and \citet{saha2021optimal} proposed UCB-based methods and showed that the cumulative regret is bounded by $\widetilde{\mathcal{O}}(d\sqrt{T})$ and  $\widetilde{\mathcal{O}}(\sqrt{dT})$. \citet{variance-aware} introduced a variance-aware action-elimination method and showed that the cumulative regret remains bounded. \citet{feelgood} utilized Thompson Sampling and proposed an algorithm called FGTS.CDB, showing that the cumulative regret is bounded by $\widetilde{\mathcal{O}}(d\sqrt{T})$. In addition, there exist lower bounds of order $\bigol(\sqrt{dT})$~\citep{saha2021optimal} for both the cumulative average and weak regret.

\paragraph{Neural Bandits.}  To the best of our knowledge, neural networks were first employed as nonlinear reward estimators in MABs~\citep{zhou2020neural, zhang2021neuralthompsonsampling, deb2023contextual, ban2021ee}.
\citet{ban2021ee} adds a second network for potential‑gain estimation but must retain all past checkpoints, whereas \citet{deb2023contextual} rely on perturbed‑prediction regression oracles that require several forward passes per round. \citet{zhou2020neural} and \citet{zhang2021neuralthompsonsampling} proposed neural UCB and TS algorithms, respectively, both achieving sublinear cumulative average regret with a sufficiently large network width and bounded effective dimension. To address computational challenges of previous results, \citet{shallow} introduced shallow exploration approach that leverages the gradients on the last layer's parameters to construct a Gram matrix. Building on these pioneering results, neural networks have been widely adopted in various bandit settings, such as combinatorial bandits~\citep{neuralcomb} and federated bandits~\citep{dai2022federated}. Moreover,  \citet{ndb} proposed 
both UCB and TS-based strategies utilizing neural networks to address contextual dueling bandits with nonlinear utility functions.

\paragraph{Variance-aware Linear Bandits.}
Variance-aware approaches for linear contextual bandits can be categorized into those that require knowledge of the true variance at each round \citep{zhou2021nearlyvar, zhou2021provablyvar} and those that do not \citep{zhao2023variance, kim2022improvedvar, zhang2021improvedvar}. Before the introduction of the method in~\citet{zhao2023variance},
methods that did not rely on the true variance often exhibited worse regret bounds compared to their variance-known counterparts. In contrast, \citet{zhao2023variance} achieves computational efficiency while matching the regret bounds of methods that assume variance knowledge, thus bridging the gap between these two lines of work. Such variance-aware approaches have also been applied to neural bandits~\citep{bui2024variance} and to contextual linear dueling bandits~\citep{variance-aware}, yielding sublinear cumulative average regret. However, to the best of our knowledge, there has been no prior work on neural variance-aware dueling bandits.
\section{PROBLEM SETTING}\label{sec:ps}

\paragraph{Contextual Dueling Bandits.} 

A learner observes a context set $X_{t}=\left\{x_{t,1},\cdots,x_{t,K}\right\}\subset \xs\subset \mathbb{R}^{d}$ at each round $0<t\leq T$, where $T$ is the total number of rounds, $K$ is the number of arms, and $\mathcal{X}$ is the overall context space. Then the learner chooses two contexts $\lt(x_{t,\stt}, x_{t,\ndt}\rt)$ in $X_t$ where $\stt, \ndt\in [K]$, which is equivalent to pull a tuple of arms $\lt(\stt, \ndt\rt)$.
As a result, the learner receives a binary preference feedback signal $o_t$. 
The feedback $o_t$ is sampled from a Bernoulli distribution $\mathcal{B}(p_t)$, where $p_t$ depends on the chosen pair of contexts. 
We assume that the probability $p_t$ is determined by a link function $\link:\mathbb{R} \rightarrow [0,1]$, depending on latent utilities.
The latent utility is defined by an unknown function $\util:\mathbb{R}^d \rightarrow \mathbb{R}$, which takes a context as input. In other words, $
    p_{t} = \mathbb{P}\lt(o_t=1 \mid x_{t,\stt},x_{t,\ndt}\rt)=\link\left(\Delta \util_{t}\right),$ where $\Delta \util_{t}:=\util\left(x_{t,\stt}\right)-\util\left(x_{t,\ndt}\right).$
We assume that the link function belongs to the exponential family \citep{lst}.
The exponential family contains various link functions, including the Bradley-Terry-Luce (BTL) model~\citep{btmodel, btmodel2}, the Thurstone-Mosteller model~\citep{thurstone2017law}, and the Exponential Noise model~\citep{noisemodel}.
{\color{black}Under this class, the Fisher Information and variance satisfy $I(\Delta\util_t)= c \cdot \sigma_t^2$ for a bounded constant $c>0$ (with $c=1$ for BTL, Thurstone-Mosteller, and Exponential Noise models)~\citep{mccullagh2019generalized}, so our inverse-variance weighting is equivalent to the inverse Fisher Information weighting~\citep{amari1998natural}.}
The gradient of the log-likelihood loss within this family has a particular form, which is utilized in the theoretical analysis of cumulative regret.

\paragraph{Utility Function.}
The linear utility function $\util$ has been extensively studied in contextual dueling bandits~\citep{saha2021optimal,lst, variance-aware, feelgood}. However, real-world utility structures are often far more complex. To address this, recent studies have explored nonlinear utility functions by approximating them with neural networks~\citep{ndb}. Building on this idea and following a similar approach to that of~\citet{shallow}, we assume that the utility function $\util$ can be expressed as an inner product between a feature map $\rep_*: \mathbb{R}^d \to \mathbb{R}^d$ and an unknown true parameter $\theta_* \in \mathbb{R}^d$, i.e.,
 \begin{align}
    \util\lt(x_{t,k}\rt)=\theta_{*}^\intercal \rep_*\lt(x_{t,k}\rt),\label{eq:kernel}
\end{align} 
which generalizes the linear setting when the feature map is the identity function.

\subsection{Objective}
The two most common types of instantaneous regret are average regret $r^{a}_t$ and weak regret $r^{w}_t$~\citep{saha2021optimal}:
\begin{align}
    r^{a}_t &= \frac{2\util\left(x_{t,\opt}\right)-\util\left(x_{t,\stt}\right)-\util\left(x_{t,\ndt}\right)}{2}, \nonumber\\
    r^{w}_t &= \util\left(x_{t,\opt}\right)-\max\left\{\util\left(x_{t,\stt}\right),\util\left(x_{t,\ndt}\right)\right\},\nonumber
\end{align} where $\opt = \arg\max_{\arm\in [K]} \util\left(x_{t,\arm}\right)$.  
Then we define the  cumulative average and weak regrets
as 
    $R^{a}_T = \sum_{t=1}^{T}r^{a}_t$ and 
    $R^{w}_T = \sum_{t=1}^{T}r^{w}_t\no
    $, respectively.
 An efficient algorithm for the stochastic contextual dueling bandits guarantees the sublinear increase of $R_t^a$ and $R_t^w$ on $T$, e.g., $\bigol(\sqrt{T}).$  
The weak and average cumulative regrets have individual advantages. For instance, the latter is more useful than the former when selecting the inferior action has significant side effects, e.g., clinical trials.  The objective of this paper is to minimize the expected cumulative average regret $R_T^a$ as \citep{lst}.
\section{METHOD}
\subsection{Nonlinear Utility Function Estimator} 

To achieve a sublinear expected cumulative average regret, it is crucial to accurately estimate the unknown utility function $\util$. While many existing works have relied on a linear structure~\citep{lst, feelgood, saha2021optimal, variance-aware}, i.e., $\util\left(x_{t,k}\right) = \th^\intercal_* x_{t,k}$, real-world utility functions are  often 
inherently nonlinear.  To approximate such nonlinear utility functions, we employ a fully connected neural network $\nn$ with the ReLU activation function $ \max\{x, 0\}$, as follows:
\begin{align}
    \nn\left(x;\th, \ww\right)&=\th^\intercal \rep\left(x;\ww\right),\label{eq:nn}\\ 
    \rep\left(x;\ww\right)&= \sqrt{\wid}\textrm{ReLU}\left(\ww^{L}(\cdots\textrm{ReLU}\left(\ww^{1} x\right)\cdots)\right), \no
\end{align}where $\ww = (\textrm{vec}(\ww^1), \cdots, \textrm{vec}(\ww^L))$ consists of learnable parameters as $\ww^{1} \in \mathbb{R}^{m\times d}$, $\ww^{l} \in \mathbb{R}^{m\times m}$ for $2 \leq l \leq L-1$, and $\ww^{L} \in \mathbb{R}^{d\times m}$. The parameter $\th \in \mathbb{R}^{d}$ is initialized as $\th = \th_{0}$, and the weight matrices are initialized as $\ww = \ww^i_{0}$ for $1 \leq i \leq L$, following the setup of prior neural bandit frameworks~\citep{shallow}. In detail, we initialize them as
\begin{align}
    \ww_0^l&=\begin{bmatrix} w & 0 \\ 0 & w \end{bmatrix}, \, \ww_0^L=
    \begin{pmatrix}
        \omega\\-\omega
    \end{pmatrix}^\intercal,\, \th_{0}\sim \gauss\left(0,\frac{1}{\dim}\right),\label{eq:init}
\end{align} for $1 \leq l \leq L-1$ by $w \sim \gauss\left(0, 4/\wid\right)$ and $\omega \sim \gauss\left(0, 2/\wid\right)$ where $\gauss$ denotes a Gaussian distribution. After receiving the binary signal $o_t$ from the arm selection $\lt(\stt, \ndt\rt)$ in each round $t$, the neural network $\nn$ is trained on historical data $\{x_{t,\sti}, x_{t,\ndi}, o_i\}_{i=1}^{t}$. The parameters $\th_{t}$ and $\ww_t$ are updated by minimizing the regularized log-likelihood loss $\loss_t$, defined as
\begin{align}       
\begin{aligned}
\loss_t\left(\th,\ww\right)&=-\sum_{i=1}^{t-1}\frac{\log \left(\link\lt(\left(-1\right)^{1-o_i}\Delta \nn_{i}\rt)\right)}{\wi^2} \\ \s + \frac{1}{2}\lambda\lt\|\th-\th_{0}\rt\|_2^2,
\end{aligned}\label{eq:total_loss}
\end{align} for some $\lambda>0$, where   
$
\Delta \nn_{i}=\nn\left(x_{i,\sti};\th, \ww\right)-\nn\left(x_{i,\ndi};\th, \ww\right)$ and
\begin{align}
\wi &= \left\{
\begin{aligned}
     &\max\{\estdi, \epsilon\} &\quad\textrm{variance-aware,}\\&1 &\quad\textrm{variance-agnostic,}
\end{aligned}\right. \label{eq:wi}
\end{align}
 with the estimated variance $|\widehat{\sigma}_{i}|^2=\link(\Delta\nn_{i})(1-\link(\Delta\nn_{i}))$ by $(\th_t, W_t)$ and a positive constant $\epsilon > 0$.

{\color{black}\begin{remark}[Role of Eqs.~\eqref{eq:total_loss}--\eqref{eq:wi}]
Eq.~\eqref{eq:total_loss} is a weighted negative log-likelihood~\citep{ueno2012weighted, dmochowski2010maximum}, the standard approach under heteroscedastic noise. The stability term $\epsilon$ in~\eqref{eq:wi} ensures bounded weights for the concentration inequalities (Lemma~\ref{lem:concentration_inequality}) and prevents loss divergence when the estimated variance approaches zero. Setting $\wi = 1$ reverts to a variance-agnostic approach with a looser $\bigol(d\sqrt{T})$ bound (Corollary~\ref{theorem:main_ucb2}).
\end{remark}}

\subsection{Arm Selection Methods.} 
\label{subsection:arm_strategy} 
We propose two types of methods: UCB-based~\citep{linucb} and TS-based~\citep{zhang2021neuralthompsonsampling} strategies. We assume access to a set of parameters $(\th_{t-1}, \ww_{t-1})$ learned from the previous round $t-1$, as well as the Gram matrix $V_{t-1}$:
\begin{align}
    V_{t-1} = \sum_{i=1}^{t-1}\frac{\Delta\rep_{i, \sti, \ndi}\Delta\rep_{i, \sti, \ndi}^\intercal}{\wi^2}+\lambda I,\label{eq:gram_matrix}
\end{align} where $\Delta\rep_{i, \arm, \arm'}:=\Delta\rep_{i, \arm, \arm', i}$. Here, 
$\Delta\rep_{i, \arm, \arm', j}:=\rep\lt(x_{i,\arm};\ww_{j-1}\rt)-\rep\lt(x_{i,\arm'};\ww_{j-1}\rt)$ for any $\arm, \arm'\in[K]$ and $i,j\in[T]$. The value $\wi$ in \cref{eq:gram_matrix} is identical to that in \cref{eq:total_loss}. 
Based on the Gram matrix, we propose three arm selection strategies: one asymmetric arm selection and two symmetric arm selections, applicable to both UCB- and TS-based methods.

\subsubsection{UCB-based Methods.}

\paragraph{Asymmetric Arm Selection (UCB-ASYM).} This  strategy selects the first arm $\stt$ greedily based on $\lt(\th_{t-1}, \ww_{t-1}\rt)$:
\begin{align}
    \stt = \argmax_{\arm\in [K]} \th_{t-1}^\intercal\rep(x_{t,k};\ww_{t-1}) \label{eq:ucb_uns_1st}.
\end{align} The second arm $\ndt$ is determined as the best competitor of $\stt$ with an exploration bonus, expressed as the upper confidence width  $\lt\|\Delta\rep_{t,\arm,\stt}\rt\|_{V_{t-1}^{-1}}$ as 
    $\ndt = \argmax_{\arm\in [K]} \th_{t-1}^\intercal\rep(x_{t,\arm};\ww_{t-1}) + \cwct\lt\|\Delta\rep_{t,\arm,\stt}\rt\|_{V_{t-1}^{-1}},\no$
where $\cwct > 0$ is a confidence width coefficient. 
The asymmetric arm selection strategy has been studied in~\citep{lst} and~\citep{ndb}, in the contexts of linear and nonlinear utility functions, respectively.

\paragraph{Optimistic Symmetric Arm Selection (UCB-OSYM).} In this strategy, the first and second arms, which maximize the sum of utilities including the exploration bonus, are selected as follows:
\begin{align}
    \stt, \ndt&=\argmax_{\arm,\arm'\in [K]\times[K]} \th_{t-1}^\intercal\rep_{t,\arm,\arm'}+\cwct\lt\|\Delta\rep_{t,\arm,\arm'}\rt\|_{V_{t-1}^{-1}},\nonumber
\end{align} with the confidence width coefficient $\cwct > 0$ and $\rep_{t,k,k'}=\rep_{t,k,k',t}$, where
    $\rep_{i,\arm,\arm',j}=\rep\lt(x_{i,\arm}; \ww_{j-1}\rt) + \rep\lt(x_{i,\arm'}; \ww_{j-1}\rt)$ for $i, j\in[T]$ and $\arm, \arm'\in[K].$
This strategy has been studied in the linear setting, as in \citep{variance-aware}.

\paragraph{Candidate-based Symmetric Arm Selection (UCB-CSYM).} This strategy first constructs a confidence set $\cs_t$ for each round $t$ for strong candidates as potential optimal arms: for the confidence width coefficient $\cwct>0$,
\begin{align}
    \cs_t = \lt\{\arm\mid \cwct\lt\|\Delta\rep_{t,\arm,\arm'}\rt\|_{V_{t-1}^{-1}}> \th_{t-1}^\intercal\Delta\rep_{t,\arm',\arm}, \forall \arm'\in[K]\rt\}.\no
\end{align}
A tuple of arms with the maximal exploration bonus is then selected {\color{black}from the confidence set $\cs_t$} as follows:
    $\stt, \ndt=\argmax_{\arm, \arm'
    \in{\color{black}\cs_t\times\cs_t}}\lt\|\Delta\rep_{t,\arm, \arm'}\rt\|_{V_{t-1}^{-1}}.\no$
This strategy type has been studied in the linear setting \citep{saha2021optimal}.

\subsubsection{Thompson Sampling Methods.}

\textbf{Asymmetric Arm Selection (TS-ASYM).} In this strategy, the first arm $\stt$ per round $t$ is greedily selected as in \cref{eq:ucb_uns_1st}.
Then, we sample each arm’s relative utility compared to the first arm from a Gaussian distribution with mean  
$\th_{t-1}^\intercal \Delta\rep_{t,\arm,\stt}$ and variance  
$\cwct^2 \left\| \Delta\rep_{t,\arm,\stt} \right\|_{V_{t-1}^{-1}}^2$,  
where $\cwct > 0$ is the confidence width coefficient, i.e.,
$
\widetilde{\util}_{t,\stt,\arm} \sim \gauss\lt(\th_{t-1}^\intercal\Delta\rep_{t,\arm,\stt}, \cwct^2\lt\|\Delta\rep_{t,\arm,\stt}\rt\|_{V_{t-1}^{-1}}^2\rt)$.
 Then, the second arm $\ndt$ in round $t$ is selected as $
    \ndt = \argmax_{\arm\in[K]}\widetilde{\util}_{t,\stt,\arm}.\nonumber
$
This strategy was studied in \citep{ndb}.

\textbf{Optimistic Symmetric Arm Selection (TS-OSYM).} 
This strategy constructs a Gaussian distribution, where the mean is defined as the sum of the estimated utilities of the pair of arms $\lt(\arm, \arm'\rt)$, and the variance is determined by the exploration bonus for the pair. Then, a value $\widetilde{\util}_{t,\arm,\arm'}$ is drawn from a Gaussian distribution as $
    \widetilde{\util}_{t,\arm,\arm'} \sim \gauss\left(\th_{t-1}^\intercal\rep_{t,\arm,\arm'}, \cwct^2\|\Delta\rep_{t,\arm,\arm'}\|^2_{V_{t-1}^{-1}}\right)\no
$ for $k,k'\in[K]$ where $\cwct > 0$ is the confidence width coefficient.  Then, the arms $\stt$ and $\ndt$ are finally selected as $
    \stt, \ndt = \argmax_{\arm,\arm'\in[K]\times[K]}\widetilde{\util}_{t,\arm,\arm'}.\no
$

\textbf{Candidate-based Symmetric Arm Selection (TS-CSYM).} In the same manner as the UCB-CSYM regime, a confidence set $\cs_t$ is constructed with the confidence width coefficient $\cwct>0$.
Then the pair $(\stt, \ndt)$ of arms is chosen from the set $\cs_t$ as $
(\stt, \ndt)=\argmax_{(i,j)\in C_t\times C_t}\widehat{\sigma}_{i,j},\no$ where
$$\widehat{\sigma}_{i,j}\sim\gauss\left(\|\Delta\rep_{t,i,j}\|^2_{V_{t-1}^{-1}}, \frac{\|\Delta\rep_{t,i,j}\|^4_{V_{t-1}^{-1}}}{4\log(Kt^2)}\right).$$


While TS-USYM aligns with the method in \citep{ndb}, our symmetric TS strategies (TS-OSYM and TS-CSYM) introduce a novel selection paradigm: instead of sequentially picking arms as in prior work~\citep{ndb, wu2016double}, both arms are simultaneously drawn from the same posterior. This shift not only differentiates our approach conceptually but also complicates the analysis, requiring more elaborate event constructions and careful use of concentration tools such as Azuma–Hoeffding.

{\color{black}\paragraph{Shallow Exploration.} Our approach constructs the Gram matrix $V_{t-1}$ using only the last-layer gradients ($\th\in\mathbb{R}^d$) instead of all parameters ($p=\bigol(m^2 L)$), reducing Gram matrix inversion from $\bigol(p^3)$ to $\bigol(d^3)$ and yielding ${\sim}28\times$ speedup over full-gradient methods. While treating hidden layers as a fixed feature extractor introduces a representation bias ($U_{t-1}^{-1}M_t$), our iterative self-improvement technique controls $\|\th_{t-1}\|_2$, ensuring this bias does not dominate the regret.}


\subsection{Neural Dueling Bandit Algorithms}
Building on the proposed UCB- and TS-based arm selection strategies, we develop \cref{alg:main}, which employs a neural network as a nonlinear estimator of the unknown utility function. At each round~$t$, the value~$\wt$ is defined exactly as in \cref{eq:wi} and incorporated into the loss function as shown in \cref{eq:total_loss}. While our method primarily relies on the estimated Bernoulli variance~$\estdt^2$, it can seamlessly accommodate the true variance~$\stdt^2$ when available from an oracle.



\begin{algorithm}[!h]
    \caption{Variance-aware Neural Dueling Bandits}
    \label{alg:main}
    \begin{algorithmic}[1]        
        \STATE Initialize the learnable parameters $\th_0,$ $\ww_0$ of a neural network $\nn$ defined in \cref{eq:init}; the Gram matrix $V_0 = \lambda I$ for some $\lambda>0$; the confidence interval coefficient $\cwct>0$.
        \STATE Set the episode length $\epi$ of $\nn$, the step size $\lr$, and the gradient steps $\gs$.
        \STATE Determine one of the arm selection strategies  in either UCB or Thompson sampling explained in \cref{subsection:arm_strategy}.
        \FOR{each round $t\in[T]$}
            \STATE Observe a context set $X_t=\left\{x_{t,1},\cdots,x_{t,K}\right\}\subset \xs.$            
            \STATE Choose $\stt, \ndt$ with $\cwct, X_t, V_{t-1}$ via the determined arm selection.
            \STATE Observe a preference (binary) signal $o_t$ by dueling of $x_{t,\stt}$ and $x_{t,\ndt}$.
            \IF{$t \equiv 0 \Mod{\epi}$}
                \STATE Update $(\th_t, \ww_t)$ by $\gs$ gradient steps via the loss $L_t(\th, \ww)$ in \cref{eq:total_loss} with the step size $\lr.$
                \STATE Update $\th_t$ again such that $\th_t$ minimizes \cref{eq:total_loss} while keeping $\ww_t$ fixed.
            \ENDIF
            \STATE Set $\wt$ by \cref{eq:wi} or by the oracle \STATE Update $V_t$ as $$V_t=V_{t-1} + \frac{\Delta\phi_{t,\stt,\ndt}^{~}\Delta\phi_{t,\stt,\ndt}^\intercal}{\wt^2}.$$
        \ENDFOR
    \end{algorithmic}
\end{algorithm}



\section{REGRET ANALYSIS}\label{sec:ra}
In this section, we provide theoretical results for the upper bounds of cumulative average regret for each arm selection method. 
To facilitate the analysis of our algorithm, we assume the following structure for the context set $\xs$, utility function $\util$ and the initial parameters $(\th_0, \ww_0)$ initialized by~\cref{eq:init}.

\begin{assumption} \label{assumption:utilbound}\label{assumption:linkbdd}\label{assumption:contextbdd}
For all $x_{t,\arm}$, the context vector is bounded as $\|x_{t,k}\|_2 = 1$ and $[x_{t,k}]_j=[x_{t,k}]_{j+\dim/2}$ for simplicity. The link function $\link$ is continuously differentiable,  takes values in $[-1, 1]$, and $\linkl$-Lipschitz, i.e.,  
$
|\link(x) - \link(x')| \leq \linkl |x - x'| \quad \text{for all } x, x' \in \mathbb{R}.
$  
{\color{black}Moreover, since the utility function is bounded and the contexts are bounded, the input to the link function remains within a compact effective domain. Within this bounded domain,} the derivative of $\link$ is strictly bounded away from zero:
$
0 < \linkb \leq \dot{\link}(x) \leq \linku,
$
for some positive constants $\linkb$ and $\linku$.  
Finally, the utility function is bounded in absolute value by 1, i.e., $|\util(x)| \leq 1$ for simplicity.
\end{assumption}
\begin{assumption}\label{assumption:phi_lipschitz}\label{assumption:thetastarbdd}
There exists a constant $\repl > 0$ such that
$
\left\|\frac{\partial \rep(x;\ww_{0})}{\partial\ww} - \frac{\partial \rep(x';\ww_{0})}{\partial\ww} \right\|_2 \leq \repl \left\|x - x'\right\|_2$, for all $x, x' \in \{x_{t,k}\}_{t\in[T], k\in[K]}.
$
The vector $\th_*$ in \cref{eq:kernel} is bounded in $L_2$ norm as $\|\th_*\|_2 \leq \thb$ for a constant $\thb > 0$.

\end{assumption}

To analyze the cumulative average regret, we define a neural tangent kernel (NTK) $\ntk$ as follows.

\begin{definition}
    $\ntk=\{\ntk_{i,j}\}_{i,j=1}^{TK} \in\mathbb{R}^{TK \times TK}$ is the neural tangent kernel (NTK) matrix  such that
    $
        \ntk_{i,j} = \frac{1}{2}\left(\widetilde{\Sigma}^L\lt(x_{t,k}, x_{t',k'}\rt)\right), \no 
    $ for all $t,t'\in[T]$ and $k,k'\in[K],$ where  for any $x,x'\in \mathbb{R}^{d},$
    \begin{align}
\begin{aligned}\widetilde{\Sigma}^0\lt(x,x'\rt)&=\Sigma^0\lt(x,x'\rt)=x^\intercal x',\\
\Lambda^l\lt(x,x'\rt) &= 
\begin{bmatrix}
\Sigma^{l-1}\lt(x,x\rt)\quad &\Sigma^{l-1}\lt(x,x'\rt)\\
\Sigma^{l-1}\lt(x',x\rt)\quad &\Sigma^{l-1}\lt(x',x'\rt)\\
\end{bmatrix},\\
\Sigma^l\lt(x,x'\rt)&=2\mathbb{E}_{(u,v)\sim \gauss(\textbf{0},\Lambda^{l-1}\lt(x,x'\rt))}\left[\sigma(u)\sigma(v)\right],\nonumber\\
\widetilde{\Sigma}^l\lt(x,x'\rt)&=2\widetilde{\Sigma}^{l-1}\lt(x',x'\rt)\mathbb{E}_{u,v}\left[\dot{\sigma}(u)\dot{\sigma}(v)\right] + \Sigma^{l}\lt(x,x'\rt).
\end{aligned}
    \end{align}
\end{definition}

\begin{assumption} \label{assumption:lambda_min}
     We assume that the minimum spectrum $\specmin\left(\ntk\right)$ is strictly positive.
\end{assumption}

Note that the assumptions listed above are standard assumptions commonly found in the literature on bandits~\citep{ndb, lst, shallow, bui2024variance}. We introduce the following notations: $\Delta\rep_{i,j}=\Delta\rep_{i,\sti, \ndi, j}$ for all $i,j\in[T]$ and 
$\mathbf{\util}=\left(\util\left(x_{1,1}\right),\cdots,\util\left(x_{T,K}\right)\right), \widetilde{\mathbf{\util}}=\left(\nn\left(x_{1,1};\th_0,\ww_0\right)\right),\cdots,\nn\left(x_{T,K};\th_{T-1},\ww_{T-1}\right).\nonumber$

{\color{black}We introduce two auxiliary quantities: the matrix $U_{t-1}$, a weighted Gram matrix using the link function derivative, and the vector $M_t$, capturing the bias from shallow exploration. The term $U_{t-1}^{-1}M_t$ explicitly accounts for this representation learning bias.}

Let
$(\theta_{t-1}, \ww_{t-1})$ be the parameters updated at round $t-1$
under the loss in~\cref{eq:total_loss} and
\begin{align}
    U_{t-1} &= \sum_{i=1}^{t-1}\dot{\link}(\ff_{i,i}(\overline{\th}_{i,t})) \frac{\Delta\rep_{i,i}\Delta\rep^{\intercal}_{i,i}}{\wi^2}+ \lambda' I, \no
    \\
    M_t &= \sum_{i=1}^{t-1}\left(\link(\th_{t-1}^\intercal\Delta\rep_{i,t}) - \link(\ff_{i,t}(\th_{t-1})) \right)\frac{\Delta\rep_{i,i}}{\wi^2},\no
\end{align} where $\lambda'=\linkb \lambda$ and  $\overline{\th}_{i,t}=c\th_{t-1} + (1-c)\th_*$ for some $c\in[0,1]$ satisfying
    $\link(\ff_{i,i}(\th_{t-1}))-\link(\ff_{i,i}(\th_*))=\dot{\link}(\ff_{i,i}(\overline{\th}_{i,t}))\Delta\rep_{i,i}^\intercal(\th_{t-1}-\th_*).$
Using the notations above, we prove a lemma that provides an upper bound on the confidence interval 
between $\th_{t-1}$ and $\th_{*}$, accounting for an additional bias term.

 \begin{lemma} \label{lem:concentration_inequality} 
 Let $\wid$ be a width of neural network as
$
    \wid=\textrm{poly}(\epi, \len,\linkb, \linku, \linkl, \lambda, \log(TK/\delta), 1/\delta, \|\mathbf{\util} - \widetilde{\mathbf{\util}}\|_{\ntk^{-1}})$ such that 
\begin{align}m=\widetilde{\Omega}(T^6), \label{eq:width}
\end{align} where $\widetilde{\Omega}$ is the asymptotic lower bound ignoring logarithm factors and all other parameters. 
    Then the following concentration inequality holds:
\begin{align}
&A_t:=\lt\|\th_{t-1}-\th_* + U_{t-1}^{-1} M_{t}\rt\|_{V_{t-1}}\label{eq:concentration_inequality_main}\\&=
\bigol\lt(\sqrt{d}+\epsilon^{-1}+\lt(\frac{d}{\epsilon^4}+d^3\rt)\frac{\sqrt{t}}{m^{1/6}}\|\mathbf{\util} - \widetilde{\mathbf{\util}}\|_{\ntk^{-1}}^{4/3}\rt)
\no
\end{align} ignoring all logarithmic and other terms.
\end{lemma} 
{\color{black}Intuitively, $A_t$ measures how far our parameter estimate $\th_{t-1}$ deviates from the true parameter $\th_*$, adjusted for the representation bias.} The last term $U_{t-1}^{-1} M_{t}$ of $A_t$ arises from the bias during representation learning. Furthermore, $\epsilon^{-1}$ cannot be made arbitrarily small while still ensuring reasonable bounds. Our proof employs a multi-step refinement strategy centered on the auxiliary function $\gg_t(\th)$:
$
\gg_t\left(\th\right)
= \sum_{i=1}^{t-1} \left[
\link\bigl(\ff_{i,i}\bigl(\th\bigr)\bigr)
- \link\bigl(\ff_{i,i}\bigl(\th_*\bigr)\bigr)
\right] \frac{\Delta\rep_{i,i}}{\wi^2}
+ \lambda\bigl(\th-\th_{0}\bigr),
$
where
$
\ff_{i,j}\left(\th\right)
= \th^\intercal \Delta\rep_{i,j}
+ \th_{0}^\intercal \nabla\Delta\rep_{i,1}\bigl(\ww_{*}-\ww_{0}\bigr).
$


The analysis proceeds iteratively. 
\textbf{First}, starting from the total loss~$\loss_t$ in \cref{eq:total_loss}, we establish an initial loose estimate 
$\|\th_{t-1}\|_2 = \mathcal{O}(\sqrt{t})$. Substituting this bound into the analysis of 
$\|\gg_t(\th_{t-1}) + M_t\|_{V_{t-1}^{-1}}$ yields an intermediate bound of order 
$\bigol(t/m^{1/6})$. Consequently, the resulting concentration inequality for~$A_t$ is limited by a term that scales with~$\sqrt{t}$ and cannot be improved by simply increasing the network width~$\wid$. 
\textbf{Second}, to address this limitation, we leverage the loose estimate together with a sufficiently large width~$\wid$ to obtain a sharper parameter bound of $\|\th_{t-1}\|_2$ depending only on $\dim$ and $\epsilon$. 
\textbf{Finally}, we substitute this improved estimate back into the analysis of 
$\|\gg_t(\th_{t-1}) + M_t\|_{V_{t-1}^{-1}}$ to derive a significantly tighter bound of order $\bigol(\sqrt{t}/m^{1/6})$ for the main concentration inequality in \cref{eq:concentration_inequality_main}. A high-level overview of the proof structure is illustrated in \cref{fig:theorem}, and the complete technical details are provided in \cref{sec:confidence_interval} of the supplementary material.

\begin{figure}[t!]
    \centering

    {
    \begin{tabular}{c}
    \includegraphics[width=0.9\columnwidth]{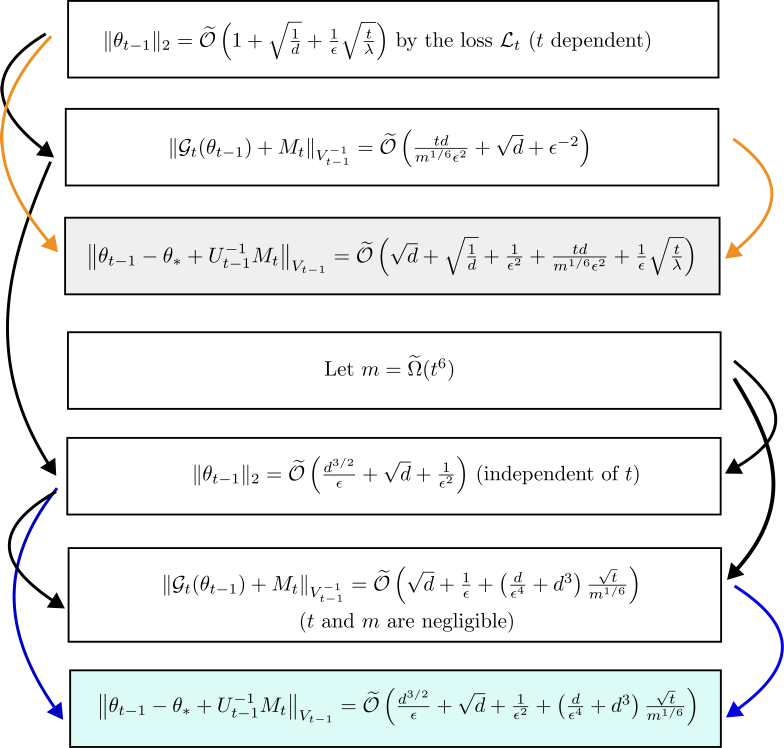}
    \end{tabular}}
    \caption{\small Illustration of the procedure in the proof of \cref{lem:concentration_inequality}. A naive analysis (\textcolor{orange}{orange line}) fails to tighten as $m$ increases (as $t$ increases). In contrast, our proposed analysis (\textcolor{black}{blue line}) establishes a bound by $\sqrt{t}/m^{1/6}$.
}
    \label{fig:theorem}
\end{figure}

Next, we establish a bound for each arm selection strategy, e.g., in the case of UCB-ASYM:
\begin{align}
    &\th_{t-1}^\intercal(2\Delta\rep_{t,\opt,\stt}+\Delta\rep_{t,\stt,\ndt}) \no\\&\leq \cwct\|\Delta\rep_{t,\stt,\ndt}\|_{V_{t-1}^{-1}}-\cwct\|\Delta\rep_{t,\opt,\stt}\|_{V_{t-1}^{-1}}\label{eq:ucb-arm1}
\end{align} 
where $\cwct>0$ is the confidence interval coefficient. To analyze the regret $r_t^a$ for the UCB methods, we decompose the regret term and apply our main concentration inequality from~\cref{lem:concentration_inequality}. This inequality is used in conjunction with strategy-specific bounds like~\cref{eq:ucb-arm1} and other theoretical tools, such as the elliptical potential lemma, to establish a final upper bound on the regret. The complete analysis is provided in~\cref{sec:UCB} of the supplementary material.


As a result, we can obtain the following result for variance-aware UCB-based algorithms.  
\begin{theorem}[Variance-aware UCB methods]\label{theorem:main_ucb}
    Let us assume \cref{assumption:contextbdd}-\cref{assumption:lambda_min}. Let the confidence width coefficient  $\cwct$ for the arm selection methods and the step size {\color{black}$\lr$} be chosen to satisfy the following:
    \begin{align}
        \cwct &= \frac{8}{\linkb}\left(\sqrt{\dim\log (\cwct')\log(4t^2/\delta)} + \frac{1}{\epsilon}\log(4t^2/\delta)\right) \no\\\s+ 2\sqrt{\lambda}\left(J
    +2\left(2+\sqrt{d^{-1}\log\left(\frac{1}{\delta}\right)}\right)\right)\no
,\\
        \lr &\leq C\left(\dim^2\wid \gs T^{6}\len^6\log(TK/\delta)\right)^{-1}\no
    \end{align} for some $C>0$ with
    $$
        \cwct'=\left(1+t\left(2\frac{\sqrt{\dim\log \gs\log(TK/\delta)}}{\sqrt{\dim\lambda}\epsilon}\right)^2\right).\nonumber
    $$ If the width $\wid$ of $\nn$ is sufficiently large as in \cref{lem:concentration_inequality}
    and $\epsilon=\bigo(1/\sqrt{d})$, then the cumulative average regret $R_T^a$ under any UCB-based arm selection method in \cref{subsection:arm_strategy} is upper-bounded as 
     \begin{align}
        R_T^a = \bigol\left(\sqrt{dT}+d\sqrt{\sum_{i=1}^{T}\sigma_i^2}+\frac{Td^{3}}{ m^{1/6}}\|\mathbf{\util} - \widetilde{\mathbf{\util}}\|_{\ntk^{-1}}^{4/3}\right).\label{eq:variance-aware}
    \end{align} ignoring all other hyperparameters with probability at least $1-\delta$ over the randomness of the initialization of the neural network $\nn$.
\end{theorem} The following result is about variance-agnostic UCB-based algorithms.
\begin{corollary}[Variance-agnostic UCB methods]\label{theorem:main_ucb2}
Assume the conditions in \cref{theorem:main_ucb} hold, but we fix $\wt=1$ (i.e., $\epsilon=1$) for all round $0\leq t \leq T$.  Then the cumulative regret $R^{a}_t$ under any UCB-based arm selection method in \cref{subsection:arm_strategy} is bounded as
    \begin{align}
        R_T^a = \bigol\left(d\sqrt{T} + \frac{Td^{3}}{ m^{1/6}}\|\mathbf{\util} - \widetilde{\mathbf{\util}}\|_{\ntk^{-1}}^{4/3}\right). \label{eq:variance-agnostic}
    \end{align} ignoring all other hyperparameters with probability at least $1-\delta$ over the randomness of learnable parameters in the neural network $\nn$
\end{corollary} 

For our TS-based methods, we again leverage the concentration inequality from~\cref{lem:concentration_inequality} to define good event sets. The construction for the TS-ASYM variant follows the established approach of~\citep{ndb}. However, our other proposed TS methods introduce a new analytical challenge, as they involve the simultaneous stochastic selection of two arms. This feature, distinct from prior work like~\citep{zhang2021neuralthompsonsampling, ndb}, requires a more intricate construction of the event sets and a more delicate application of the Azuma-Hoeffding inequality. Despite this added complexity, our analysis successfully yields regret bounds for the TS algorithms that are comparable to their UCB-based counterparts. A detailed derivation is provided in~\cref{sec:ts} of the supplementary material.



\begin{figure*}[t!]
    \centering
    \begin{tabular}{ccc}
    \includegraphics[width=0.31\textwidth]{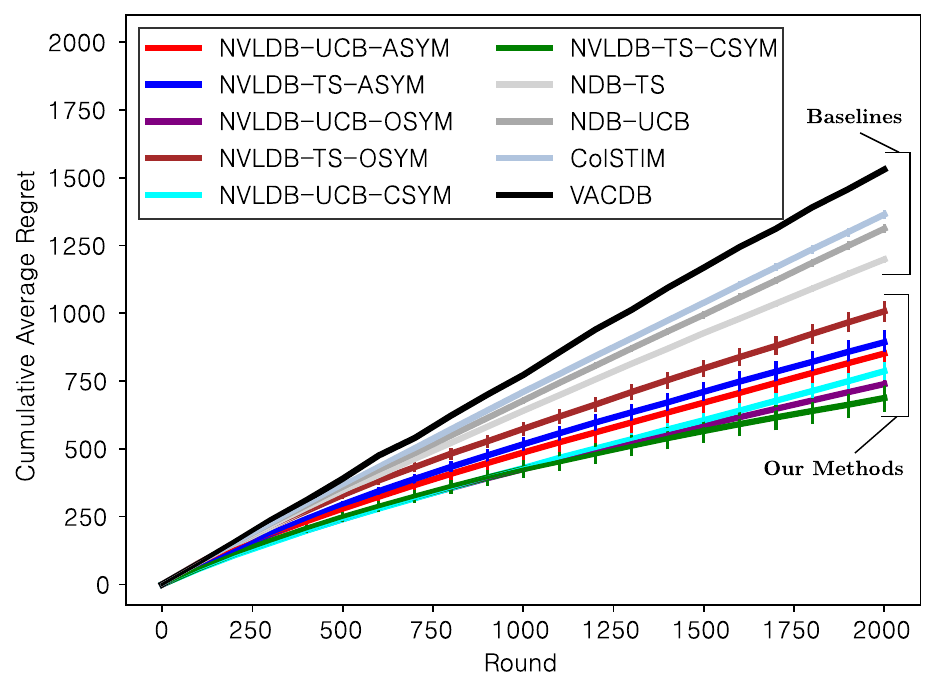}&
    \includegraphics[width=0.31\textwidth]{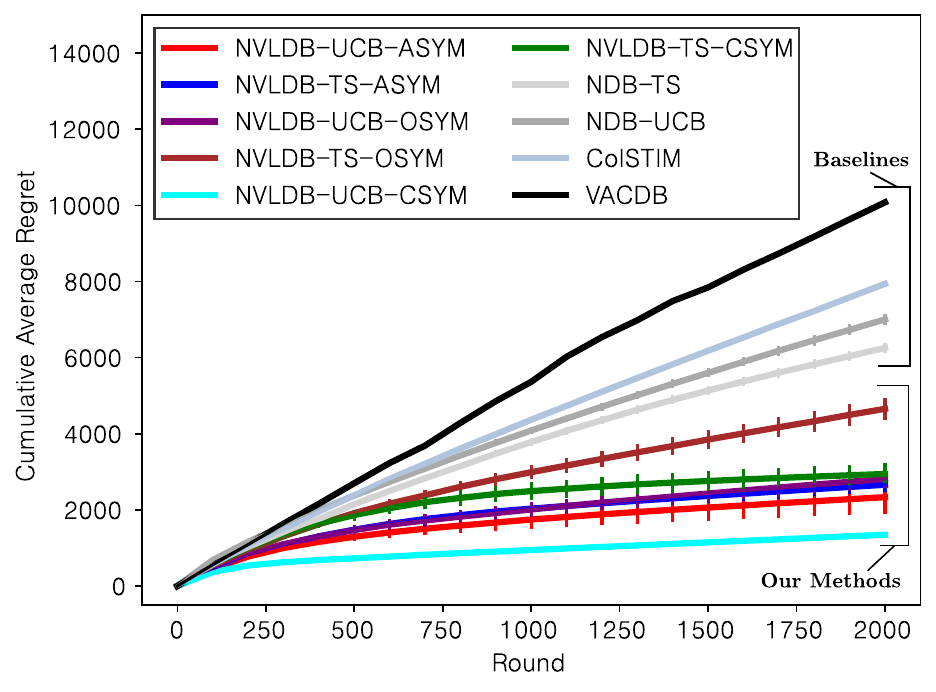}&
    \includegraphics[width=0.31\textwidth]{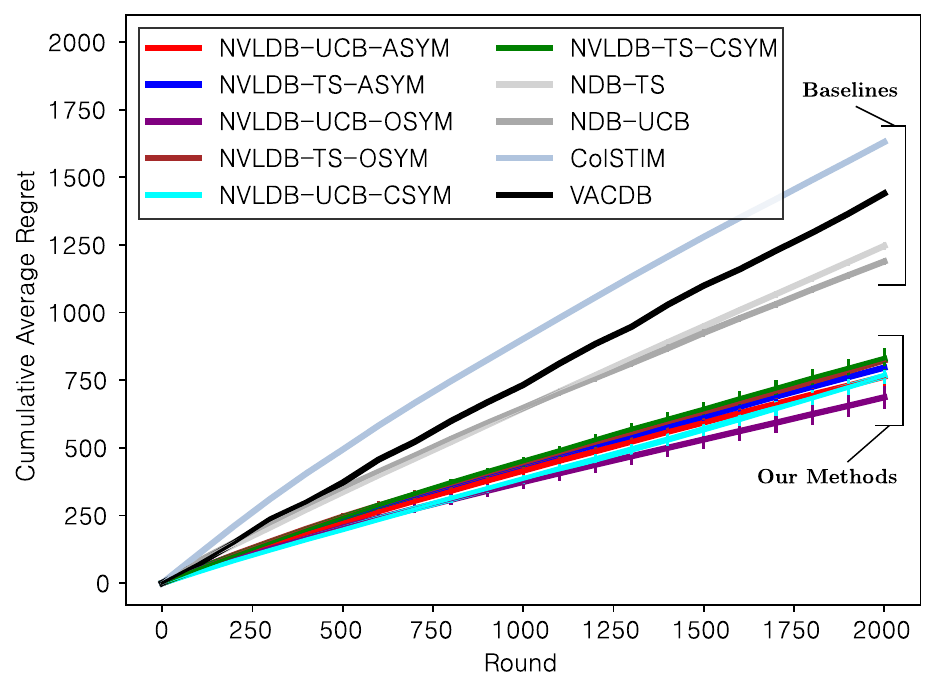}\\
    {\small (a): $\cos(3\Theta^\intercal x)$} &  {\small (b): $10(\Theta^\intercal x)^2$} & {\small (c): $x^\intercal\Theta\Theta^\intercal x$}
    \end{tabular}
    \caption{\small Comparison of cumulative average regret under (a)-(c): synthetic tasks with a context dimension of $\dim=5$ and a total of $K=5$ arms.
    Each experiment is repeated across 20 random seeds.
    \label{fig:main_result_all}}
\end{figure*}

\begin{theorem}[Variance-aware TS methods]\label{theorem:main_ts}
   Assume the conditions in \cref{theorem:main_ucb} hold. Then  the cumulative regret $R^{a}_t$ under any TS-based arm selection method in \cref{subsection:arm_strategy} is bounded as \cref{eq:variance-aware} with probability at least $1-\delta$ over the randomness of learnable parameters in the neural network $\nn$.
\end{theorem} 
In the case of variance-agnostic TS algorithms, we have the following result.
\begin{corollary}[Variance-agnostic TS methods]\label{theorem:main_ts2}
Assume the conditions in \cref{theorem:main_ucb2} hold. Then the cumulative regret $R^{a}_t$ under any TS-based arm selection method in \cref{subsection:arm_strategy} is bounded as \eqref{eq:variance-agnostic} with probability at least $1-\delta$ over the randomness of learnable parameters in the neural network $\nn$.
\end{corollary}

\cref{theorem:main_ucb}--\cref{theorem:main_ts2} establish that our UCB and TS methods achieve the same order of cumulative regret, up to logarithmic factors. To specialize these results, we adopt the common assumption that $\|\widetilde{\mathbf{\util}} - \mathbf{\util}\|_{\ntk^{-1}} = \mathcal{O}(1)$, which is theoretically justified in~\citep{shallow, bui2024variance}.
Under this assumption and with a network width $\wid$ satisfying~\cref{eq:width}, we have two key settings:  In the \textbf{variance-aware} setting, $R_T^a$ is bounded by $\mathcal{O}\lt(\sqrt{dT} + d\sqrt{\sum_{i=1}^{T}\sigma_i^2}\rt)$, which matches the bound established by~\citet{bui2024variance}. In the \textbf{variance-agnostic} setting, $R_T^a$ is bounded by $\mathcal{O}(d\sqrt{T})$, aligning with the result derived by~\citet{shallow}.

A key distinction arises between the linear and neural bandit settings. In the linear case, where representation learning is unnecessary, variance-aware methods can achieve cumulative regret on the order of $\mathcal{O}(d)$ when variances are negligible~\citep{zhou2022computationally, variance-aware}. This is because the model parameters can be estimated without incurring additional representation errors. In contrast, since the neural setting fundamentally relies on representation learning, accumulated estimation errors are unavoidable. These errors appear as an additional bias term in the regret analysis, preventing the achievement of $\widetilde{\mathcal{O}}(d)$ regret. Consequently, while neural networks provide greater expressive power than linear models, they also introduce a fundamental challenge that does not arise in the linear setting.

{\color{black}If $\|\th_{t-1}\|_2 \leq C$ could be established without additional assumptions, the width requirement would reduce to $\widetilde{\Omega}(T^3)$, matching neural CMABs~\citep{shallow, bui2024variance}. Our experiments confirm robust performance with moderate widths ($m=100$) and scalability with larger widths and higher dimensions (see \cref{sec:width_sensitivity,sec:high_dim}).}

\section{EXPERIMENTS}\label{sec:exp}

To empirically validate \term, we consider the following synthetic nonlinear utility functions---cosine, square, and matrix---as shown in \cref{fig:main_result_all}, which have been widely studied in prior work~\citep{ndb, zhang2021neuralthompsonsampling, zhou2020neural, dai2022federated}. We evaluate each arm-selection strategy introduced in \cref{subsection:arm_strategy} under \term. {\color{black}We use a fully connected network with depth $\len=3$, width $\wid=100$, learning rate $10^{-3}$, regularization $\lambda=1.0$, Adam optimizer, gradient steps $\gs=20$, and episode length $\epi=1$. Full details are in the supplementary material.}

To benchmark their performance, we compare them with the following baselines: VACDB~\citep{variance-aware}, ColSTIM~\citep{lst}---linear dueling bandits with and without variance-awareness, respectively---and Neural Dueling Bandits (NDB)~\citep{ndb} with both UCB and TS methods. We used the official implementations without modification.

\cref{fig:main_result_all} shows that the proposed NVLDB variants--under both UCB and TS methods--consistently outperform the baseline algorithms in cumulative average regret across all three tasks (with context dimension $\dim=5$ and $K=5$) while exhibiting sublinear cumulative regret under nonlinear utility functions.

{\color{black}
\paragraph{Scalability and Robustness.} To verify scalability, we conducted additional experiments with higher dimensions ($d=100$, $m=100$, square utility) and wider networks ($m\in\{500, 1000\}$, $d=5$, square utility) over $T=10{,}000$ rounds. All NVLDB variants maintain sublinear regret (decreasing $R(T)/T$), confirming robustness well beyond the settings in prior work~\citep{ndb, bui2024variance}. For instance, with $d=100$, NVLDB-UCB-ASYM achieves $R(10000)/10000 = 6.05 \pm 0.74$; with wider networks ($m=500$), NVLDB-UCB-OSYM achieves $0.39 \pm 0.04$. Detailed results are provided in \cref{sec:width_sensitivity} and \cref{sec:high_dim}.

\paragraph{Variance-Aware vs.\ Variance-Agnostic.}
On the UCI Statlog dataset under a deterministic preference model ($P(i\succ j)=1$ for $i>j$, negligible variance), NVLDB variants significantly outperform their variance-agnostic counterparts (NLDB), achieving near-zero cumulative average regret (e.g., NVLDB-UCB-OSYM: $0.017$ vs.\ NLDB-UCB-OSYM: $0.199$ at $T=200$), confirming the theoretical advantage of variance-aware confidence sets.
}

Additional results including real-world tasks and variance-agnostic comparisons are in the supplementary material.

\section{CONCLUSION}
We introduce \term, a family of neural variance-aware linear dueling bandit algorithms. Our approach employs a neural network to approximate nonlinear utility functions and estimates uncertainty by constructing a Gram matrix from the gradients with respect to the learnable parameters of the last layer. These gradients are weighted by the estimated variance of the Bernoulli distribution, computed via exponential-family link functions. We theoretically prove that both \term and its variance-agnostic variant achieve sublinear cumulative average regret across various arm selection strategies under both the UCB and TS frameworks. To the best of our knowledge, this is the first work to introduce variance-aware neural dueling bandits with shallow exploration. Furthermore, we empirically validate our method against existing baselines and demonstrate its superior performance on widely used synthetic and real-world tasks.

{\color{black}Our key technical contribution is a bootstrap argument with spectral analysis that bounds $\|\hat{\theta}_{t-1}\|_2$ by a time-independent constant, reducing the width requirement from $\widetilde{\Omega}(T^{14})$~\citep{ndb} to $\widetilde{\Omega}(T^6)$. We also introduce novel TS-based symmetric strategies (TS-OSYM, TS-CSYM). Closing the gap to $\widetilde{\Omega}(T^3)$~\citep{shallow, bui2024variance} remains future work. Finally, our inverse-variance weighting is equivalent to inverse Fisher Information under the exponential family, validated by strong empirical gains in low-variance regimes.}

\newpage
\appendix
\onecolumn
\counterwithin{figure}{section}
\counterwithin{table}{section}
\counterwithin{theorem}{section}
\counterwithin{corollary}{section}
\counterwithin{lemma}{section}
\counterwithin{assumption}{section}
\counterwithin{definition}{section}
\numberwithin{equation}{section}  
\begin{center}{\bf {\LARGE Supplementary Material:}}
\end{center}
\begin{center}{\bf {\Large Neural Variance-aware Dueling Bandits \\ with Deep Representation
and Shallow Exploration}}
\end{center}
\section*{ORGANIZATION OF THE SUPPLEMENTARY MATERIAL}
This document serves as a technical appendix to the main paper. It provides additional technical details and further experimental results omitted from the main text to facilitate a deeper understanding of the proposed method. The remainder of this document is organized as follows.

\paragraph{\cref{sec:preliminary}.}
We present theoretical results on the width of neural networks and several useful inequalities, which are instrumental in establishing the sublinear cumulative regret bounds of our method, \term.

\paragraph{\cref{sec:confidence_interval}.}
We provide a theoretical analysis of the confidence intervals for neural dueling bandits with shallow exploration and deep representation. The resulting bounds play a central role in both the UCB and Thompson Sampling (TS) frameworks.

\paragraph{\cref{sec:UCB}.}
We prove sublinear cumulative regret bounds under the UCB-based arm selection strategies. First, we analyze each arm selection strategy within the UCB regime and derive key inequalities. We then combine these with the concentration bounds from \cref{sec:confidence_interval} to establish our main results under the UCB framework.

\paragraph{\cref{sec:ts}.}
Applying the concentration bounds from \cref{sec:confidence_interval}, we prove that the TS framework achieves sublinear cumulative regret bounds. Specifically, we demonstrate that this sublinearity holds for each TS-based arm selection strategy.

\paragraph{\cref{sec:addexp}.}
We describe the experimental setup, including pseudocode, and provide additional experimental results that further validate the effectiveness of \term.

\paragraph{Supplementary Zip File Contents}
The supplementary zip file contains the following items:
\begin{itemize}
\item \textbf{supplementary-material.pdf}: Includes proofs of the theoretical results, detailed descriptions of the experimental setup, and additional experimental results (e.g., on real-world datasets).
\item \textbf{full-paper.pdf}: A consolidated version of the manuscript, including the main paper, references, checklist, and supplementary material.
\item \textbf{code.zip}: Contains the implementation of our proposed method, \term.
\end{itemize}

\section{PRELIMINARIES}\label{sec:preliminary}
This section provides a collection of mathematical notations and preliminaries that form the foundation of our theoretical analysis. For clarity and to ensure a self-contained exposition, we explicitly state all necessary definitions and results---some newly derived, others drawn from existing literature---even if they repeat material from the main manuscript. These components, though some may seem elementary, are instrumental in constructing our proofs and establishing the final cumulative average regret bounds.

Throughout this paper, we suppress constant factors and dependencies on all hyperparameters within the $\bigol(\cdot)$ notation, explicitly indicating only the dependencies on the key parameters of interest, such as $T$, $\dim$, $\wid$, and $\frac{1}{\delta}$. Likewise, our asymptotic lower-bound notation, $\widetilde{\Omega}(\cdot)$, omits these dependencies when appropriate.

In the remaining sections, we will frequently use the  notation $\rep\lt(x;\ww\rt)$ under weights $\ww$ for $x\in\xs$, as the deep neural network representation. Here, $\xs$ is a context space. Then a neural network $\nn$ can be expressed as 
\begin{align}
    \nn\lt(x;\th, \ww\rt)=\th^\intercal\rep\lt(x;\ww\rt). \label{eq:app-nn}
\end{align} with the last layer's parameter $\th\in\mathbb{R}^{d}.$
We also define notations for $\rep$ already used in the main manuscript:
\begin{align}
\begin{aligned}
    \rep_{t,\arm,i}&=\rep\lt(x_{t,\arm}; \ww_{i-1}\rt)\\
    \Delta\rep_{i,\arm,\arm',j}&=\rep\lt(x_{i,\arm};\ww_{j-1}\rt)-\rep\lt(x_{i,\arm'};\ww_{j-1}\rt),\\
\Delta\rep_{i,\arm,\arm'}:=\Delta\rep_{i,\arm,\arm',i}&=\rep\lt(x_{i,\arm};\ww_{i-1}\rt)-\rep\lt(x_{i,\arm'};\ww_{i-1}\rt),\\
    \Delta\rep_{i,j} := \Delta\rep_{i,\sti,\ndi,j}&=\rep\lt(x_{i,\sti};\ww_{j-1}\rt)-\rep\lt(x_{i,\ndi};\ww_{j-1}\rt),\\    \Delta\rep_{i}:=\Delta\rep_{i,\sti,\ndi,i}&=\rep\lt(x_{i,\sti};\ww_{i-1})-\rep(x_{i,\ndi};\ww_{i-1}\rt),
\end{aligned} \label{eq:not}
\end{align} for $t, i, j \in [T] = \{1, \dots, T\}$ and $\arm, \arm' \in [K]$, where $K$ is the total number of arms, $\ww_{i-1}$ and $\ww_{j-1}$ denote the parameters obtained at rounds $i-1$ and $j-1$, respectively. Moreover, $(\sti, \ndi)\in[K]\times[K]$ denotes the pair of arms selected by the learner at round $i$. We use $\nabla$ to denote the gradient with respect to $\ww$, i.e., $\nabla = \nabla_{\ww}$. As we mentioned in the main manuscript, we denote $\ww_0$ and $\th_0$ as the initialization of $\ww$ and $\th$, respectively. In detail, for $1 \leq l \leq L-1$, where $\len$ is the depth of $\nn$,
we initialize the learnable parameters as
\begin{align}
    \ww_0^l&=\begin{bmatrix} w & 0 \\ 0 & w \end{bmatrix}, \quad \ww_0^L=
    \begin{pmatrix}
        \omega^\intercal\\-\omega^\intercal
    \end{pmatrix}^\intercal,\quad \th_{0}\sim \gauss\left(0,\frac{1}{\dim}\right),\label{eq:app-init}
\end{align} where $w \sim \gauss\left(0, 4/\wid\right)$ and $\omega \sim \gauss\left(0, 2/\wid\right)$. We also redefine a utility function $\util$ assuming
\begin{equation}
    \util\lt(x_{t,k}\rt) = \theta_{*}^{\intercal} \rep_{*}\lt(x_{t,k}\rt),
    \label{eq:app-kernel}
\end{equation}
where $\theta_{*} \in \mathbb{R}^{D}$ is the true parameter vector and $\rep_{*}: \mathcal{X} \rightarrow \mathbb{R}^{\
dim}$ is the true representation function, with $\dim$ being the dimension of the representation space. Finally,   a link function $\link$ belongs to an exponential family~\citep{lst}, e.g., the Bradley–Terry model~\citep{btmodel,btmodel2}: $\link(x)=\frac{1}{1+\exp(-x)}$. 

 We assume that $(\th_{t-1}, \ww_{t-1})$ denotes the learned parameters from the previous round $t-1$ by the loss:
\begin{align}       
\begin{aligned}
\loss_t\left(\th,\ww\right)&=-\sum_{i=1}^{t-1}\frac{\log \left(\link\lt(\left(-1\right)^{1-o_i}\Delta \nn_{i}\rt)\right)}{\wi^2}+ \frac{1}{2}\lambda\lt\|\th-\th_{0}\rt\|_2^2,
\end{aligned}\label{eq:app-total_loss}
\end{align} for some $\lambda>0$ and $\Delta f_i = f(x_{i,\sti};\th,\ww) - f(x_{i,\ndi};\th,\ww)$, as well as the Gram matrix $V_{t-1}$:
\begin{align}
    V_{t-1} = \sum_{i=1}^{t-1}\frac{\Delta\rep_{i, \sti, \ndi}\Delta\rep_{i, \sti, \ndi}^\intercal}{\wi^2}+\lambda I.\label{eq:app-gram_matrix}
\end{align} where
\begin{align}
\wi &= \left\{
\begin{aligned}
     &\max\{\estdi, \epsilon\} &\quad\textrm{variance-aware,}\\&1 &\quad\textrm{variance-agnostic,}
\end{aligned}\right. \label{eq:app-wi}
\end{align}

We restate our assumptions stated in the main manuscript for self-contained supplementary material.

\begin{assumption} 
 \label{assumption:app-utilbound}\label{assumption:app-linkbdd}\label{assumption:app-contextbdd}
For all $x_{t,\arm}$, the context vector is bounded as $\|x_{t,k}\|_2 = 1$ and $[x_{t,k}]_j=[x_{t,k}]_{j+\dim/2}$ for simplicity. The link function $\link$ is continuously differentiable,  takes values in $[-1, 1]$, and $\linkl$-Lipschitz, i.e.,  
$
|\link(x) - \link(x')| \leq \linkl |x - x'| \quad \text{for all } x, x' \in \mathbb{R}.
$  
Moreover, for all $x \in \xs$, the derivative of $\link$ is bounded as  
$
0 < \linkb \leq \dot{\link}(x) \leq \linku,
$  
for some positive constants $\linkb$ and $\linku$.  
Finally, the utility function is bounded in absolute value by 1, i.e., $|\util(x)| \leq 1$ for simplicity.
\end{assumption}

\begin{assumption}
\label{assumption:app-phi_lipschitz}\label{assumption:app-thetastarbdd}
There exists a constant $\repl > 0$ such that
$
\left\|\frac{\partial \rep(x;\ww_{0})}{\partial\ww} - \frac{\partial \rep(x';\ww_{0})}{\partial\ww} \right\|_2 \leq \repl \left\|x - x'\right\|_2$, for all $x, x' \in \{x_{t,k}\}_{t\in[T], k\in[K]}.
$
The vector $\th_*$ in \cref{eq:app-kernel} is bounded in $L_2$ norm as $\|\th_*\|_2 \leq \thb$ for a constant $\thb > 0$.
\end{assumption}

To analyze the cumulative average regret, we define a neural tangent kernel (NTK) $\ntk$ as follows.

\begin{definition}
    $\ntk=\{\ntk_{i,j}\}_{i,j=1}^{TK} \in\mathbb{R}^{TK \times TK}$ is the neural tangent kernel (NTK) matrix  such that
    $
        \ntk_{i,j} = \frac{1}{2}\left(\widetilde{\Sigma}^L\lt(x_{t,k}, x_{t',k'}\rt)\right), \no 
    $ for all $t,t'\in[T]$ and $k,k'\in[K],$ where  for any $x,x'\in \mathbb{R}^{d},$
    \begin{align}
\begin{aligned}\widetilde{\Sigma}^0\lt(x,x'\rt)&=\Sigma^0\lt(x,x'\rt)=x^\intercal x',\\
\Lambda^l\lt(x,x'\rt) &= 
\begin{bmatrix}
\Sigma^{l-1}\lt(x,x\rt)\quad &\Sigma^{l-1}\lt(x,x'\rt)\\
\Sigma^{l-1}\lt(x',x\rt)\quad &\Sigma^{l-1}\lt(x',x'\rt)\\
\end{bmatrix},\\
\Sigma^l\lt(x,x'\rt)&=2\mathbb{E}_{(u,v)\sim \gauss(\textbf{0},\Lambda^{l-1}\lt(x,x'\rt))}\left[\sigma(u)\sigma(v)\right],\nonumber\\
\widetilde{\Sigma}^l\lt(x,x'\rt)&=2\widetilde{\Sigma}^{l-1}\lt(x',x'\rt)\mathbb{E}_{u,v}\left[\dot{\sigma}(u)\dot{\sigma}(v)\right] + \Sigma^{l}\lt(x,x'\rt).
\end{aligned}
    \end{align}
\end{definition}

\begin{assumption} 
\label{assumption:app-lambda_min}
     We assume that the minimum spectrum $\specmin\left(\ntk\right)$ is strictly positive.
\end{assumption}

Due to \cref{eq:app-init}, $L_2$-boundedness for $\th_0$ is obtained.
\begin{lemma}\label{lemma:inittheta0} Let $\th_{0}\sim  \gauss(0,1/\dim)$ as in \eqref{eq:app-init}. Then $\th_{0}$ is $L_2$-bounded as follows:
\begin{align}
    \lt\|\th_{0}\rt\|_2 \leq 2\left(2+\sqrt{\dim^{-1}\log\left(\frac{1}{\delta}\right)}\right) \label{eq:theta0}\end{align} with probability at least $1-\delta$. 
\end{lemma}
Combining  \cref{assumption:app-thetastarbdd} and \eqref{eq:theta0}, it is also obvious that 
\begin{align}
    \|\th_*-\th_{0}\|_{V_{t-1}^{-1}}\leq \frac{1}{\sqrt{\lambda}}\left(J
    +2\left(2+\sqrt{d^{-1}\log\left(\frac{1}{\delta}\right)}\right)\right),\label{eq:fact5}
\end{align} for $V_{t-1}$ defined in \cref{eq:app-gram_matrix}, with probability at least $1-\delta.$

{\color{black}\textbf{Key neural network lemmas.} The following three lemmas (from prior work) are the building blocks for our analysis: Lemma~\ref{lemma:app-linear} establishes that the utility function can be linearized around $\ww_0$ with a controlled residual ($\ww_*$ exists close to $\ww_0$); Lemma~\ref{lemma:app-upper} bounds the parameter drift $\|\ww_t-\ww_0\|_2$ under gradient descent, ensuring the NTK regime holds; and Lemma~\ref{lemma:app-linear2} quantifies the linearization error of the representation $\rep$.}

We have the following three lemmas for a sufficiently wide network as \cref{eq:app-nn}.
\begin{lemma}[\citet{shallow}]\label{lemma:app-linear}
    Assume that \cref{assumption:app-lambda_min} holds. Let
$\ntk$ be the neural tangent kernel~\citep{jacot2018neural} and
    \begin{align}
        \mathbf{\util} &= \left(\util\left(x_{1,1}\right) ,\cdots,\util\left(x_{T,K}\right)\right)\in \mathbb{R}^{TK}\nonumber\\
        \widetilde{\mathbf{\util}} &= \left(\nn\left(x_{1,1};\th_{0},\ww_{0}\right),\cdots,\nn\left(x_{T,K};\th_{T-1},\ww_{T-1}\right)\right) \in \mathbb{R}^{TK}. \nonumber
    \end{align} 
    Then there exists $\ww_{*}$ such that 
    \begin{align}
        \lt\|\ww_*-\ww_0\rt\|_2 \leq \frac{1}{\sqrt{\wid}}\sqrt{\left(\mathbf{\util}-\widetilde{\mathbf{\util}}\right)^T \ntk^{-1} \left(\mathbf{\util}-\widetilde{\mathbf{\util}}\right)} = \frac{1}{\sqrt{\wid}}\lt\|\mathbf{\util} - \widetilde{\mathbf{\util}}\rt\|_{\ntk^{-1}},\no
    \end{align} and
    \begin{align}
        \util\left(x_{t,\arm}\right)=\th_*^\intercal \rep(x_{t,\arm};\ww_{t-1}) + \th_{0}^\intercal \left(\nabla \rep_{t,\arm,1}\right)\left(\ww_*-\ww_{0}\right),\no
    \end{align} for all $\arm\in[K]$ and $t\in [T]$.
\end{lemma}

\begin{lemma}[\citet{shallow}]\label{lemma:app-upper}
\textcolor{black}{Assume that \cref{assumption:app-utilbound} holds.} For any round index $t \in [T]$ in the $q$-th epoch, i.e., $t = (q - 1)\epi + i $  for some $i \in [\epi]$ where $\epi$ is the episode length. If the step size $\lr$ satisfies
\[
\lr \leq \frac{C_0}{\dim^2 \wid \gs T^{6} \len^6 \log(TK/\delta)},
\] for some $C_0>0$ and the width $\wid$ of the neural network satisfies 
\[
m \geq \max\{\len \log(TK/\delta), d\len^2 \log(\wid/\delta), \delta^{-6} \epi^{18} \len^{16} \log^3(TK)\},
\]
then, the following inequalities hold with probability at least $1 - \delta$:
\begin{align}
\begin{aligned}
\lt\|\ww_t - \ww_0\rt\|_2 &\leq \frac{\delta^{3/2}}{\wid^{1/2} T \gs^{9/2} \len^6 \log^3(\wid)},
\\
\lt\|\nabla \rep(x_{t,\arm}; \ww_0)\rt\|_F &\leq C_1 \sqrt{\dim\len\wid},
\\
\lt\|\rep(x; \ww_t)\rt\|_2 &\leq \sqrt{\dim \log(\gs) \log(TK/\delta)}, \quad \forall x\in \xs,
\end{aligned} \no
\end{align}
for all $t \in [T]$ and $\arm \in [K]$.
\begin{remark}
    \cref{lemma:app-upper} imposes a stricter requirement of $\lr=\bigol(T^{-6})$, in contrast to the $\lr=\bigol(T^{-5.5})$ stated in the original paper~\citep{shallow} (ignoring all other parameters). This necessity arises from the fact that the parameter estimate $\widehat{\theta}_{t-1}$ lacks a closed-form solution in the contextual dueling bandit setting. This is in contrast to the multi-armed bandit case, where a linear utility model typically admits such a closed form.
\end{remark}
\end{lemma}
\begin{lemma}[\citet{cao2019generalization}]\label{lemma:app-linear2}
Let $\ww, \ww' \in B(\ww_0, \omega)$ for some radius $\omega > 0$. Let the neural network $\nn$ be given as in \eqref{eq:app-nn} and let the width $\wid$ and the radius $\omega$  satisfy
\begin{align}
\wid &\geq C_0 \max\{\dim\len^2 \log(\wid/\delta), \omega^{-4/3}L^{-8/3} \log(TK) \log(\wid/(\omega \delta))\},\no\\
\omega &\leq C_1 \len^{-5} (\log \wid)^{-3/2},\no
\end{align} for some $C_0>0, C_1>0$. Then the following inequality holds:
\[
\lt|\rep(x; \ww) - \widehat{\rep}(x; \ww)\rt| \leq C_2 \omega^{4/3} \len^3 \dim^{-1/2} \sqrt{\wid \log \wid},
\] for all $x \in \{x_{t,\arm}\}_{t \in [T], k \in [K]}$ with probability at least $1 - \delta$, where $\widehat{\rep}(x; \ww)$ is the linearization of $\rep(x; \ww)$ at $\ww'$ as follows:
\[
\hat{\rep}(x; \ww) = \rep(x; \ww') + \lt\langle \nabla \rep(x; \ww'), \ww - \ww' \rt\rangle.
\]
\end{lemma}

We now estimate the difference of $\rep(x;\ww)-\rep(x;\ww')$ for some $\ww$ and $\ww'$ using the results we introduced.

\begin{lemma}\label{lemma:app-diffphiw} Under the same conditions in \cref{lemma:app-linear2}, there exist constants $C_1, C_2>0$ such that
\begin{align}
\|\rep(x;\ww_{t-1})-\rep(x;\ww_{i-1})\|_2&\leq  C_1 \left(\wid^{-1/6}\len^3\dim^{1/2}\sqrt{\log \wid}\diff^{4/3} +  \frac{\sqrt{\dim}\delta^{3/2}}{T\gs^{9/2}\len^{11/2}\log^3\wid}\right),\no
\\
    \|\rep(x;\ww_{t-1})-\rep(x;\ww_{*})\|_2&\leq  C_2\lt(\lt(\wid^{-1/6}\len^3\dim^{1/2}\sqrt{\log \wid} + \dim^{1/2}L^{1/2}\rt)\diff^{4/3} \rt),\no
\end{align} with probability at least $1-\delta$.
\end{lemma}
\begin{proof}
    Let $\widehat{\rep}(x;\ww)$ be the linearization of $\rep(x;\ww)$ at $\ww'$ resulting from \cref{lemma:app-linear2} as $$
\widehat{\rep}(x;\ww) = \rep(x;W')+\langle\nabla_W\rep(x;\ww'),\ww-\ww'\rangle.
$$
    Since $\rep(x;\ww_{0})\equiv 0$ for all $x\in \xs$, we can compute $\rep(x;\ww_{t-1})-\rep(x;\ww_{i-1})$ as follows:
\begin{align}
\begin{aligned}
    \rep(x;\ww_{t-1})-\rep(x;\ww_{i-1}) &= \rep(x;\ww_{t-1}) - \rep(x;\ww_{0}) - (\rep(x;\ww_{i-1}) - \rep(x;\ww_{0}))
    \\&=\rep(x;\ww_{t-1}) - \widehat{\rep}(x;\ww_{t-1}) - (\rep(x;\ww_{i-1}) - \widehat{\rep}(x;\ww_{i-1}))\\ \s + \langle\nabla \rep(x;\ww_{0}),\ww_{t-1}-\ww_{0}\rangle - \langle\nabla \rep(x;\ww_{0}),\ww_{i-1}-\ww_{0}\rangle.
\end{aligned} \label{eq:A.3}
\end{align} Therefore, using \cref{eq:A.3}, we can show that there exist constants $C, C'>0$ such that
\begin{align}
    \lt\|\rep(x;\ww_{t-1})-\rep(x;\ww_{i-1})\rt\|_2
    &\leq C \left(w^{4/3}\len^3\dim^{1/2}\sqrt{\wid\log \wid} + \frac{\sqrt{\dim}\delta^{3/2}}{T\gs^{9/2}\len^{11/2}\log^3\wid}\right) \nonumber\\
    &\leq C' \left(\wid^{-1/6}\len^3\dim^{1/2}\sqrt{\log \wid}\|\mathbf{\util} - \widetilde{\mathbf{\util}}\|_{\ntk^{-1}} +  \frac{\sqrt{\dim}\delta^{3/2}}{T\gs^{9/2}\len^{11/2}\log^3\wid}\right),\nonumber
\end{align} where the first inequality follows from applying \cref{lemma:app-upper} and \cref{lemma:app-linear2} together 
with the Cauchy--Schwarz inequality. The second inequality is derived using the estimate 
\(w = \bigol\lt(m^{-\tfrac12}\,\|\util - \widetilde{\util}\|_{\ntk^{-1}}\rt)\). Similarly, we can show that there exists $C''>0$ such that 
\begin{align}
    \lt\|\rep(x;\ww_{t-1})-\rep(x;\ww_{*})\rt\|_2\leq C''\left(\lt(\wid^{-1/6}L^3d^{1/2}\sqrt{\log m} + d^{1/2}L^{1/2}\rt)\diff^{4/3}\right)\nonumber.
\end{align} This completes the proof.
\end{proof}
We also estimate the difference of $\rep(x;\ww)-\rep(x';\ww)$.
\begin{lemma}\label{lemma:diffphix}
Under the same conditions in \cref{lemma:app-linear2}, there exist  constants $C_1, C_2>0$ such that
\begin{align}
    &\lt\|\rep(x;\ww_{t-1})-\rep(x';\ww_{t-1})\rt\|_2\no \\ &\leq C_1\left(\lt\|\mathbf{\util} - \widetilde{\mathbf{\util}}\rt\|_{\ntk^{-1}}^{4/3}L^3d^{1/2}\sqrt{\log m}/(m^{1/6}) + 2G \repl \frac{\delta^{3/2}}{\sqrt{m}Tn^{9/2}\log^3m}\right)\label{eq:diffphi_xx'} \\&
    \lt\|\rep(x;\ww_{*})-\rep(x';\ww_{*})\rt\|_2\no\\&\leq C_2\left(\lt\|\mathbf{\util} - \widetilde{\mathbf{\util}}\rt\|_{\ntk^{-1}}^{4/3}L^3d^{1/2}\sqrt{\log m}/(m^{1/6}) + 2G \repl \frac{\|\mathbf{\util} - \widetilde{\mathbf{\util}}\|_{\ntk^{-1}}}{\sqrt{m}}\right)\label{eq:diffphi_xx''}
\end{align}
\end{lemma}
\begin{proof} Using \cref{lemma:app-linear2}, $\rep(x;\ww_{t-1})-\rep(x';\ww_{t-1})$ can be decomposed into
    \begin{align}
    \rep(x;\ww_{t-1})-\rep(x';\ww_{t-1}) &= \rep(x;\ww_{t-1}) - \rep(x;\ww_{0}) - (\rep(x';\ww_{t-1}) - \rep(x';\ww_{0}))\nonumber
    \\&=\rep(x;\ww_{t-1}) - \widehat{\rep}(x;\ww_{t-1}) - (\rep(x';\ww_{i-1}) - \widehat{\rep}(x';\ww_{i-1}))\nonumber\\ \s + \langle\nabla (\rep(x;\ww_{0})-\rep(x';\ww_{0})),\ww_{t-1}-\ww_{0}\rangle. \no
\end{align} due to $\rep(x;\ww_0)=0$ for any $x\in\xs$. Using \cref{assumption:app-phi_lipschitz} and \cref{lemma:app-upper}, we have
\begin{align}
    \langle\nabla (\rep(x;\ww_{0})-\rep(x';\ww_{0})),\ww_{t-1}-\ww_{0}\rangle\leq 2\xb \repl \frac{\delta^{3/2}}{\sqrt{m}Tn^{9/2}\log^3m}.\nonumber
\end{align} Furthermore, using
\cref{lemma:app-linear2} to estimate the difference of $\rep(x;\ww_{t-1})-\widehat{\rep}(x;\ww_{t-1}))$ and $\rep(x';\ww_{t-1})-\widehat{\rep}(x';\ww_{t-1}))$, we can show that
\begin{align}
    \|\rep(x;\ww_{t-1})-\rep(x';\ww_{t-1})\|_2\leq w^{4/3}L^3\dim^{1/2}\sqrt{\wid\log \wid} + 2\xb \repl \frac{\delta^{3/2}}{\sqrt{m}Tn^{9/2}\log^3m}.\label{eq:phix-phix'2}
\end{align} By substituting  $w=\bigol\lt(\|\frac{\diff}{\sqrt{m}}\rt)$ into \cref{eq:phix-phix'2}, we can show \cref{eq:diffphi_xx'}. The case of $\ww_{*}$, i.e., \cref{eq:diffphi_xx''}, is shown in a similar manner. This completes the proof.
\end{proof}

We also provide the following elliptic potential lemma.
\begin{lemma}[\citet{lattimore2020bandit}]\label{lemma:epl}
Let $\{y_t\}_{t=1}^{\infty}$ be a sequence in $\mathbb{R}^d$ and $\lambda > 0$. Assume $\|y_t\|_2 \leq B$ for all $t$ and $\lambda \geq \max\{1, B^2\}$ for some $B > 0$. Let $A_t = \lambda I + \sum_{s=1}^t y_s y_s^\intercal.$ Then, the following inequalities hold:
\begin{align}
    \det(A_t) &\leq (\lambda + t B^2 /\dim)^\dim, \no\\
    \sum_{t=1}^T \|y_t\|_{A_t^{-1}}^2 &\leq 2 \log \frac{\det(A_T)}{\det(\lambda I)} \leq 2\dim \log\left(1 + \frac{TB^2}{\lambda \dim}\right).\no
\end{align}
\end{lemma}

We define the filtration that is used to apply the Bernstein-type inequality in martingale regime.
\begin{definition}[Filtration]
Let $(\Omega, \mathcal{F}, \mathbb{P})$ be a probability space. A \emph{filtration} is a family 
$\{\mathcal{F}_t\}_{t \in I}$ of sub-$\sigma$-algebras of $\mathcal{F}$ indexed by a totally ordered set \(\mathcal{T}\) (often time), 
such that for all $s, t \in \mathcal{T}$ with $s \le t$, we have
\[
  \mathcal{F}_s \;\subseteq\; \mathcal{F}_t \;\subseteq\; \mathcal{F}.
\]
\end{definition}
The sub-$\sigma$-algebra \(\mathcal{F}_t\) can be interpreted as 
the information available up to time \(t\), and \(\mathcal{F}_s\subseteq \mathcal{F}_t\) 
reflects the fact that one cannot lose information over time. Using the filtration, we can obtain the following Bernstein-type inequality.

\begin{lemma}[\citet{zhou2021nearlyvar}]\label{lemma:berstein}
Let $\{\mathcal{F}_t\}_{t=1}^\infty$ be a filtration, and $\{y_t, r_t\}_{t \geq 1}$ a stochastic process such that $y_t \in \mathbb{R}^d$ is $\mathcal{F}_t$-measurable and $r_t \in \mathbb{R}$ is $\mathcal{F}_{t+1}$-measurable. Fix $R, G, \sigma, \lambda > 0$, and $\th^* \in \mathbb{R}^d$. For $t \geq 1$, suppose that $r_t, y_t$ satisfy
\begin{align*}
    |r_t| &\leq R, \\
    \mathbb{E}[r_t \mid \mathcal{F}_t] &= 0, \\
    \mathbb{E}[r_t^2 \mid \mathcal{F}_t] &\leq \sigma^2, \\
    \|y_t\|_2 &\leq \xb.
\end{align*}
Then, the following inequality holds with probability at least $1 - \delta$:
\[
\left\|
\sum_{i=1}^t y_i r_i
\right\|_{A_{t}^{-1}}\leq 8 \sigma \sqrt{d \log\left(1 + \frac{t G^2}{d \lambda}\right) \log\left(\frac{4t^2}{\delta}\right)} + 4 R \log\left(\frac{4t^2}{\delta}\right)
\]
 for any $t \geq 1$, where $A_t = \lambda I + \sum_{i=1}^{t-1} y_i y_i^\intercal$.

\end{lemma}
The following lemma provides an inequality for the product of $L_2$ boundedness for the product of a Gram matrix and extra term.
\begin{lemma}[\citet{shallow}]\label{lemma:extraterm}
Let $\{y_t\}_{t=1,2,\ldots}$ be a real-valued sequence such that $|y_t| \leq D$ for some constant $D > 0$.
Let $A_t = \lambda I + \sum_{s=1}^t y_s y_s^\intercal$ such that  $y_t \in \mathbb{R}^d$ and $\|y_t\|_2 \leq B$ for all $t \geq 1$ and some constants $\lambda, B > 0$. Then the following inequality holds:
\[
\left\| A_t^{-1} \sum_{s=1}^t \phi_s \zeta_s \right\|_2 \leq 2Dd, 
\] for all $t\geq 1.$
\end{lemma}

\section{PROOF OF CONFIDENCE INTERVAL LEMMA}\label{sec:confidence_interval}
Let 
$(\theta_{t-1}, \ww_{t-1})$ be the parameters updated at round $t-1$ 
under the loss in~\cref{eq:app-total_loss} and
\begin{align}
\begin{aligned}
    U_{t-1} &= \sum_{i=1}^{t-1}\dot{\link}(\ff_{i,i}(\overline{\th}_{i,t})) \frac{\Delta\rep_{i,i}\Delta\rep^{\intercal}_{i,i}}{\wi^2}+ \lambda' I, 
    \\
    M_t &= \sum_{i=1}^{t-1}\left(\link(\th_{t-1}^\intercal\Delta\rep_{i,t}) - \link(\ff_{i,t}(\th_{t-1})) \right)\frac{\Delta\rep_{i,i}}{\wi^2},
\end{aligned}\label{eq:diffg}
\end{align} where $\lambda'=\linkb \lambda$ and  $\overline{\th}_{i,t}=c\th_{t-1} + (1-c)\th_*$ for some $c\in[0,1]$ satisfying
    $\link(\ff_{i,i}(\th_{t-1}))-\link(\ff_{i,i}(\th_*))=\dot{\link}(\ff_{i,i}(\overline{\th}_{i,t}))\Delta\rep_{i,i}^\intercal(\th_{t-1}-\th_*).$
Using the notations above, we restate a lemma that provides an upper bound on the confidence interval 
between $\th_{t-1}$ and $\th_{*}$, accounting for an additional bias term.

 \begin{lemma} \label{lem:app-concentration_inequality} 
 Let $\wid$ be a width of neural network as
$
    \wid=\textrm{poly}(\epi, \len,\linkb, \linku, \linkl, \lambda, \log(TK/\delta), 1/\delta, \|\mathbf{\util} - \widetilde{\mathbf{\util}}\|_{\ntk^{-1}})$ such that 
\begin{align}m=\widetilde{\Omega}(T^6), \label{eq:app-width}
\end{align} where $\widetilde{\Omega}$ is the asymptotic lower bound ignoring logarithm factors and all other parameters. 
    Then the following concentration inequality holds:
\begin{align}
&A_t:=\lt\|\th_{t-1}-\th_* + U_{t-1}^{-1} M_{t}\rt\|_{V_{t-1}}\label{eq:app-concentration_inequality_main}\\&=
\bigol\lt(\sqrt{d}+\epsilon^{-1}+\lt(\frac{d}{\epsilon^4}+d^3\rt)\frac{\sqrt{t}}{m^{1/6}}\|\mathbf{\util} - \widetilde{\mathbf{\util}}\|_{\ntk^{-1}}^{4/3}\rt)
\no
\end{align} ignoring all logarithmic and other terms.
\end{lemma}

The existence of $\overline{\th}_{i,t}$ satisfying \cref{eq:diffg} in~\cref{lem:app-concentration_inequality} is true due to the mean-value theorem~\citep{rudin1964principles}. This lemma is identical to the lemma introduced in the main manuscript. To prove~\cref{lem:app-concentration_inequality}, we estimate the left-hand side by showing a loose upper bound for $\|\th_{t-1}\|_2$.

{\color{black}\textbf{Proof roadmap (Bootstrap / Iterative Self-Improvement).} The proof of \cref{lem:app-concentration_inequality} proceeds in three stages:
\begin{enumerate}
    \item \textbf{Loose bound (Lemma~\ref{lemma:app_theta_t_bound}):} Using only the loss function, we obtain $\|\th_{t-1}\|_2 = \mathcal{O}(\sqrt{t/\lambda}/\epsilon)$. This bound grows with time and is insufficient for sublinear regret.
    \item \textbf{First refinement (Lemma~\ref{lemma:gt+mt}):} Substituting the loose bound into the decomposition of $\|\gg_t(\th_{t-1})+M_t\|_{V_{t-1}^{-1}}$ yields an intermediate bound of order $\bigol(td/(\epsilon^2 m^{1/6}))$. For a sufficiently large width $m = \widetilde{\Omega}(t^6)$, this allows us to derive a \emph{time-independent} bound $\|\th_{t-1}\|_2 = \bigol(d^{3/2} + \epsilon^{-2})$ (Eq.~\ref{eq:tight_theta_t-1}).
    \item \textbf{Final refinement (Lemma~\ref{lemma:gt+mt2}):} Substituting the improved time-independent bound back yields the tight result $\|\gg_t(\th_{t-1})+M_t\|_{V_{t-1}^{-1}} = \bigol(\sqrt{t}/m^{1/6})$, which gives the final concentration inequality.
\end{enumerate}
This iterative strategy is necessary because dueling bandits lack a closed-form estimator---the key distinction from standard neural bandits where $\hat{\theta}_{t-1} = V_{t-1}^{-1}b_{t-1}$ trivially gives a time-independent norm bound.}

{\color{black}\textbf{Stage 1: Loose bound on $\|\th_{t-1}\|_2$.} The following lemma provides an initial, time-dependent bound by comparing the loss at $(\th_{t-1}, \ww_{t-1})$ with the loss at the initialization $(\th_0, \ww_0)$. Since the initialized representation $\rep(x;\ww_0)=0$, the initial loss reduces to $\frac{t}{\epsilon^2}\log 2$, giving a simple upper bound on $\|\th_{t-1}-\th_0\|_2$.}

\begin{lemma} \label{lemma:app_theta_t_bound}
For the given $\th_{0}\sim \gauss(0,1/\dim)$ and $\th_{t-1}$  updated by \eqref{eq:app-total_loss}, the following holds:
    \begin{align}
    \lt\|\th_{t-1}\rt\|_2 \leq 2\left(2+\sqrt{\dim^{-1}\log\left(\frac{1}{\delta}\right)}\right)  + \frac{1}{\epsilon}\sqrt{\frac{t}{\lambda}\log4}.\label{eq:app-theta_t_bounded1}
\end{align} 
\end{lemma}
\begin{proof}
First, note that for any $x$, we have $\rep\bigl(x;\ww_0\bigr) = 0$ by \cref{eq:app-init}. Hence,
\[
\link\bigl(\th_0^\intercal \Delta\rep_{t,1}\bigr) 
\;=\; 
\link(0) 
\;=\; 
\frac{1}{2},
\] for a link function $\link$. 
By substituting the above into \eqref{eq:app-total_loss}, we obtain
\[
    \frac{\lambda}{2}\,\|\th_{t-1} - \th_0\|_2^2 
    \leq
    \loss_{t}\bigl(\th_{t-1}, \ww_{t-1}\bigr) 
    \leq
    \loss_{t}\bigl(\th_0, \ww_0\bigr) 
    \leq
    \frac{t}{\epsilon^2}\log 2.
\]
It follows that
\begin{align}
     \|\th_{t-1}\|_2 &\leq \|\th_{0}\|_2 + \frac{1}{\epsilon}\sqrt{\frac{t}{\lambda}\log4} \nonumber\\
    &\leq 2\left(2+\sqrt{\dim^{-1}\log\left(\frac{1}{\delta}\right)}\right)  + \frac{1}{\epsilon}\sqrt{\frac{t}{\lambda}\log4}. \no
\end{align}
where the second inequality uses \cref{lemma:inittheta0}. This completes the proof.
\end{proof}
{\color{black}\textbf{Auxiliary functions.} We now introduce the auxiliary function $\ff_{i,j}$ and $\gg_t$ that decompose the confidence interval into manageable components. The function $\ff_{i,j}(\th)$ linearizes the utility around the initial weights $\ww_0$, while $\gg_t(\th)$ aggregates the link function differences across all rounds---it plays the role of a ``gradient-like'' quantity whose norm controls the concentration inequality.}

We define the function $\ff_{i,j}:\mathbb{R}^{\dim}\rightarrow \mathbb{R}$ for $i, j \in [T]$:
\begin{align}
    \ff_{i,j}\left(\th\right)
    &= \th^\intercal \Delta\rep_{i,j} 
       + \th_{0}^\intercal \nabla\Delta\rep_{i,1}\bigl(\ww_{*}-\ww_{0}\bigr),
    \label{eq:F}
\end{align}
where $\Delta\rep_{i,1}$ and $\Delta\rep_{i,j}$ are given in~\cref{eq:not}. 
Using $\ff_{i,j}$, we additionally define the following auxiliary  function $\gg_t:\mathbb{R}^\dim\rightarrow\mathbb{R}^\dim$:
\begin{align}
    \gg_t\left(\th\right)
    &= \sum_{i=1}^{t-1}
       \link\bigl(\ff_{i,i}\bigl(\th\bigr)\bigr)
             - \link\bigl(\ff_{i,i}\bigl(\th_*\bigr)\bigr)
       \frac{\Delta\rep_{i,i}}{\wi^2}
       + \lambda\bigl(\th-\th_{0}\bigr).\no
\end{align}
We establish the relationship between this confidence interval 
and the function $\gg_t$. 

{\color{black}\textbf{Key decomposition.} The following lemma shows that the confidence interval $\|\th_{t-1}-\th_*+U_{t-1}^{-1}M_t\|_{V_{t-1}}$ is controlled by $\|\gg_t(\th_{t-1})+M_t\|_{V_{t-1}^{-1}}$. The core idea is that $\gg_t(\th_{t-1})-\gg_t(\th_*)+M_t = U_{t-1}(\th_{t-1}-\th_*+U_{t-1}^{-1}M_t)$ via the mean-value theorem, and $U_{t-1}\geq \linkb V_{t-1}$, so the $V_{t-1}$-norm of the left side lower-bounds the $V_{t-1}$-norm of the right side.}

 \begin{lemma} \label{lem:app-decom} Let $\theta_{*}$ be the parameter defined in~\cref{eq:app-kernel}, and let
$(\theta_{t-1}, \ww_{t-1})$ be the parameters updated at round $t-1$
under the loss in~\cref{eq:app-total_loss}. 
Then the following inequality holds:
\begin{align}
\lt\|\th_{t-1}-\th_* + U_{t-1}^{-1} M_{t}\rt\|_{V_{t-1}}&\leq \frac{1}{\linkb}
    \lt\|\gg_t\left(\th_{t-1}\right)+M_{t} - \lambda\left(\th_*-\th_{0}\right)\rt\|_{V_{t-1}^{-1}} \nonumber\\
    &\leq \frac{1}{\linkb}\left(\lambda\left\|\th_*-\th_{0}\right\|_{V_{t-1}^{-1}}+\lt\|\gg_t(\th_{t-1})+M_t\rt\|_{V_{t-1}^{-1}}\right).\nonumber
\end{align}
\end{lemma}

\begin{proof}
 By using the mean value theorem~\citep{rudin1964principles}, there exists $\overline{\th}_{i,t}=c\th_{t-1} + (1-c)\th_*$ for some $c\in[0,1]$ such that
\begin{align}
    \link\lt(\ff_{i,i}\lt(\th_{t-1}\rt)\rt)-\link\lt(\ff_{i,i}\lt(\th_*\rt)\rt)=\dot{\link}\lt(\ff_{i,i}\lt(\overline{\th}_{i,t}\rt)\rt)\Delta\rep_{i,i}^\intercal\lt(\th_{t-1}-\th_*\rt).
    \nonumber
\end{align}As a result, $\gg_t(\th_{t-1})$ can be expressed as
\begin{align}
    \gg_t(\th_{t-1})=\sum_{i=1}^{t-1}\left[\dot{\link}(\ff_{i,i}(\overline{\th}_{i,t}))\right]\frac{\Delta\rep_{i, i}\Delta\rep_{i, i}^\intercal}{\wi^2}(\th_{t-1}-\th_*) + \lambda'(\th_{t-1}-\th_{0}).\label{eq:gt2}
\end{align} 
It is obvious that \begin{align}
    \gg_t\left(\th_*\right) = \lambda'\left(\th_*-\th_{0}\right),\label{eq:gt0}
\end{align} and
\begin{align}
    \frac{d\ff_{s,t}(\th)}{d\th} = \Delta\rep_{s,t},\nonumber
\end{align} for $s,t\in[T].$
Therefore, we can obtain
\begin{align}
    \gg_t\left(\th_{t-1}\right)-\gg_t\left(\th_{*}\right)+M_t&=\gg_t\left(\th_{t-1}\right) - \lambda'\left(\th_*-\th_{0}\right) + M_{t}\label{eq:ineq0}\\&=\sum_{i=1}^{t-1}\dot{\link}(\ff_{i,i}(\overline{\th}_{i,t}))\frac{\Delta\rep_{i, i}\Delta\rep_{i, i}^\intercal}{\wi^2}(\th_{t-1}-\th_*) + \lambda'(\th_{t-1}-\th_{*})+M_t\nonumber\\&= U_{t-1}\left(\th_{t-1}-\th_* + U_{t-1}^{-1}M_{t}\right).\nonumber
\end{align} where the first and the second equalities result from \cref{eq:gt0} and \cref{eq:gt2}, respectively. Since $V_{t-1}$ is the Gram matrix in \eqref{eq:app-gram_matrix}, we can compute
\begin{align}
    U_{t-1}\geq \linkb V_{t-1}=\linkb\left(\sum_{i=1}^{t-1}\frac{\Delta\rep_{i,i}\Delta\rep^{\intercal}_{i,i}}{\wi^2} + \lambda I\right)\geq\lambda' I.\label{eq:ineq1}
\end{align}  As a result, using \cref{eq:ineq0,eq:ineq1}, we can show the following result:
\begin{align}
    &\|\gg_t\left(\th_{t-1}\right) - \lambda'\left(\th_*-\th_{0}\right) + M_{t}\|_{V_{t-1}^{-1}}^2\\&=\left(\th_{t-1}-\th_* + U_{t-1}^{-1}M_{t}\right)^\intercal U_{t-1} V_{t-1}^{-1}  U_{t-1}\left(\th_{t-1}-\th_* + U_{t-1}^{-1}M_{t}\right)\nonumber\\ &\geq  \linkb^2\left(\th_{t-1}-\th_* + U_{t-1}^{-1}M_{t}\right)^\intercal V_{t-1}\left(\th_{t-1}-\th_* + U_{t-1}^{-1}M_{t}\right)\nonumber\\
    &=\linkb^2\|\th_{t-1}-\th_* + U_{t-1}^{-1}M_{t}\|^2_{V_{t-1}}.\no
\end{align} This completes the proof.
\end{proof}
Due to the mean-value theorem~\citep{rudin1964principles}, $M_t$ can also be expressed as \begin{align}
    M_t=-\sum_{i=1}^{t-1}\dot{\link}(\widetilde{\ff}_{i,t})(\th_{0}^\intercal\nabla \Delta\rep_{i,1}(\ww_{*}-\ww_{0}))\frac{\Delta\rep_{i,i}}{\wi^2}, \label{eq:mt2-2}
\end{align} where  $\widetilde{\ff}_{i,t}=c \ff_{i,t}(\th_{t-1}) + (1-c) \th^\intercal_{t-1}\Delta\rep_{i,t}$ for some $0\leq c \leq 1$.

{\color{black}\textbf{Bias term $U_{t-1}^{-1}M_t$.} The term $M_t$ captures the representation learning bias---the discrepancy between using the current features $\rep_{i,i}$ and the linearized features $\rep_{i,t}$. The following lemma bounds $\|U_{t-1}^{-1}M_t\|_2$, showing it scales as $\bigol(d^{3/2}\epsilon^{-1}\diff^{4/3})$ and vanishes for large $m$.}

We estimate $U_{t-1}^{-1}M_t$ in the following lemma.
\begin{lemma} \label{lemma:um}With probability at least $1-\delta$, the following inequality holds:
    \begin{align}
    \|U_{t-1}^{-1}M_t\|_2&\leq 2\dim \repl \xb \frac{\utilu}{\epsilon}\left(2+\sqrt{\dim^{-1}\log\left(\frac{1}{\delta}\right)}\right)\sqrt{\dim\len}\diff^{4/3}.\label{eq:mt1}\\
    &=\bigol(d^{3/2}\epsilon^{-1}\diff^{4/3}).\no
\end{align} 
\end{lemma}
\begin{proof} $M_t$ in \cref{eq:mt2-2} can be expressed as 
\begin{align}
\begin{aligned}
    M_t&=\sum_{i=1}^{t-1}\frac{d_{i,t}}{\wi}\frac{\Delta\rep_{i,i}}{\wi},
\end{aligned}\no
\end{align} where
\begin{align}
    d_{i,t}=-\dot{\link}(\widetilde{\ff}_{i,t})(\th_{0}^\intercal\nabla \Delta\rep_{i,\sti,\ndi,1}(\ww_{*}-\ww_{0})).\nonumber
\end{align} Notice that 
\begin{align}
    \linkb V_{t-1}\leq U_{t-1},\no
\end{align} by \cref{eq:ineq1}. Thus,
    \begin{align}
        \|U_{t-1}^{-1}M_t\|_2\leq \frac{1}{\linkb}\|V_{t-1}^{-1}M_t\|_2. \nonumber
    \end{align}
We need to show that $d_{i,t}$ is bounded. Indeed, we can show that 
    \begin{align}
    \begin{aligned}
        |d_{i,t}|&=\left|\dot{\link}(\widetilde{\ff}_{i,t})(\th_{0}^\intercal\nabla \Delta\rep_{i,\sti,\ndi,1}(\ww_{*}-\ww_{0}))\right|\\&\leq \linku \|\th_{0}\|_2 \|\nabla \Delta\rep_{i,\sti,\ndi,1}\|_F \|\ww_{*}-\ww_{0}\|_2 
        \\
        &\leq \linku\left(2+\sqrt{d^{-1}\log\left(\frac{1}{\delta}\right)}\right)\|\nabla \Delta\rep_{i,\sti,\ndi,1}\|_F\frac{\diff^{4/3}}{\sqrt{m}}\\&\leq 2\repl \xb \utilu\left(2+\sqrt{\dim^{-1}\log\left(\frac{1}{\delta}\right)}\right)\sqrt{\dim\len}\diff^{4/3}.
    \end{aligned} \nonumber
    \end{align} Here, the first inequality follows from the matrix inequality 
and the existence of an upper bound \(\linku\) for \(\dot{\link}\). 
The second inequality uses \cref{lemma:app-linear} and \cref{lemma:inittheta0}, 
and the third follows from \cref{lemma:app-upper}. 
Finally, applying \cref{lemma:extraterm}, we can complete the proof.
\end{proof}

{\color{black}\textbf{Stage 2: Decomposition of $\gg_t(\th_{t-1})+M_t$.} We now decompose $\gg_t(\th_{t-1})+M_t$ into six sub-terms $\gg_{t,1,1}$ through $\gg_{t,1,6}$, plus $\gg_{t,2,1}$, $\gg_{t,2,2}$, and $\gg_{t,3}$. The key insight is that the first-order optimality condition $\frac{\partial \loss_t}{\partial\th}(\th_{t-1},\ww_{t-1})=0$ causes $\gg_{t,1,1}+\gg_{t,2,1}=0$ (Eq.~\ref{eq:a1b1}) and $\gg_{t,1,2}+M_t=0$ (Eq.~\ref{eq:b12}), so these terms vanish exactly. The remaining terms arise from (i) the feature drift $\Delta\rep_{i,t}-\Delta\rep_{i,i}$ between different rounds, (ii) the link function evaluated at different parameter values, and (iii) the martingale noise $\varepsilon_i$. Each is bounded using the neural network approximation guarantees and the elliptic potential lemma.}

Now we estimate $\|\gg_t(\th_{t-1})+M_t\|_{V_{t-1}^{-1}}.$
By \cref{lemma:app-linear2} and \cref{eq:F}, we can obtain
\begin{align}
    \link\left(\ff_{i,i}\left(\th_*\right)\right) = \link\left(\util\left(x_{i,\sti}\right)-\util\left(x_{i,\ndi}\right)\right) = o_i-\varepsilon_i,\label{eq:ot_varepsilon}
\end{align} where $o_i$ is a binary signal and $\varepsilon_i$ is the remaining term at round $i$. Using \cref{eq:ot_varepsilon}, we can decompose $\gg_t(\th_t)+M_t$ as follows:
\begin{align}
    \gg_t(\th_{t-1})+M_t&=\sum_{i=1}^{t-1}u\left(\ff_{i,i}\left(\th_{t-1}\right)\right)-u\left(\ff_{i,i}\left(\th_*\right)\right)\frac{\Delta\rep_{i,i}}{\wi^2} +\lambda'\left(\th_{t-1}-\th_{0}\right) + M_t \nonumber\\&=\gg_{t,1}+\gg_{t,2}+\gg_{t,3} + M_t\nonumber
\end{align} where
\begin{align}
\begin{aligned}
    \gg_{t,1} &= \sum_{i=1}^{t-1}u\left(\ff_{i,i}\left(\th_{t-1}\right)\right)\frac{\Delta\rep_{i,i}}{\wi^2}+\lambda'\left(\th_{t-1}-\th_{0}\right),\\
    \gg_{t,2} &= \sum_{i=1}^{t-1}\left(-o_i\right)\frac{\Delta\rep_{i,i}}{\wi^2},\\
    \gg_{t,3}&= \sum_{i=1}^{t-1}\frac{\varepsilon_i}{\wi}\frac{\Delta\rep_{i,i}}{\wi}.
\end{aligned}\label{eq:b123}
\end{align}
We can further decompose $\gg_{t,1}, \gg_{t,2}, \gg_{t,3}$ into 
\begin{align}
    \gg_{t,1} &= \gg_{t,1,1}+\gg_{t,1,2}+\gg_{t,1,3}+\gg_{t,1,4}+\gg_{t,1,5}+\gg_{t,1,6},\nonumber\\
    \gg_{t,2} &= \sum_{i=1}^{t-1}\frac{-o_i}{\wi^2}\Delta\rep_{i,t} + \sum_{i=1}^{t-1}\frac{o_i}{\wi^2}\left(\Delta\rep_{i,t}-\Delta\rep_{i,i}\right)=\gg_{t,2,1}+\gg_{t,2,2}.\nonumber
\end{align} where
\begin{align}
    \gg_{t,1,1} &= \sum_{i=1}^{t-1}\link(\th_{t-1}^\intercal\Delta\rep_{i,t})\frac{\Delta\rep_{i,t}}{\wi^2} + \lambda\left(\th_{t-1}-\th_{0}\right), \nonumber\\
    \gg_{t,1,2} &= \sum_{i=1}^{t-1}\link(\ff_{i,t}(\th_{t-1})) - \link(\th_{t-1}^\intercal\Delta\rep_{i,t})\frac{\Delta\rep_{i,i}}{\wi^2},\nonumber\\
    \gg_{t,1,3} &= \sum_{i=1}^{t-1}\link(\ff_{i,t}(\th_{t-1})) - \link(\th_{t-1}^\intercal\Delta\rep_{i,t})\frac{(\Delta\rep_{i,t}-\Delta\rep_{i,i})}{\wi^2},\nonumber\\
    \gg_{t,1,4} &= -\sum_{i=1}^{t-1} \link\left(\ff_{i,t}\left(\th_{t-1}\right)\right) - \link\left(\ff_{i,i}\left(\th_{t-1}\right)\right)\frac{\Delta\rep_{i,i}}{\wi^2}, \nonumber\\
    \gg_{t,1,5} &= -\sum_{i=1}^{t-1}\link\left(\ff_{i,t}\left(\th_{t-1}\right)\right) - \link\left(\ff_{i,i}\left(\th_{t-1}\right)\right)\frac{(\Delta\rep_{i,t}-\Delta\rep_{i,i})}{\wi^2}, \nonumber\\
    \gg_{t,1,6}  &= -\sum_{i=1}^{t-1}\link\left(\ff_{i,i}\left(\th_{t-1}\right)\right)\frac{\left(\Delta\rep_{i,t} - \Delta\rep_{i,i}\right)}{\wi^2}.\nonumber
\end{align}

Notice that $\th_{t-1}$ and $\ww_{t-1}$ are updated to satisfy $$\frac{\partial\mathcal{L}_t}{\partial\th}(\th_{t-1},\ww_{t-1})=0,$$ where $\loss_t$ is defined in \cref{eq:app-total_loss}. Then we can obtain
\begin{align}
    \frac{\partial\mathcal{L}_t}{\partial\th}(\th_{t-1},\ww_{t-1})=\gg_{t,1,1}+\gg_{t,2,1} = 0 \label{eq:a1b1}
\end{align} Moreover, by the definition of $M_t$,
it holds that 
\begin{align}
    \gg_{t,1,2}+M_{t}=0.\label{eq:b12}
\end{align} Thanks to \cref{eq:a1b1} and \cref{eq:b12}, we can obtain
\begin{align}
    \|\gg_{t,1,1}+\gg_{t,2,1}\|_{V_{t-1}^{-1}}=0,\quad \|\gg_{t,1,2}+M_t\|_{V_{t-1}^{-1}}=0. \label{eq:fact1}
\end{align}

Now, we estimate the terms $\|\gg_{t,1,3}+\gg_{t,1,5}+\gg_{t,1,6}+\gg_{t,2,2}\|_{V_{t-1}^{-1}}$.  
We can rearrange $\gg_{t,1,3}+\gg_{t,1,5}+\gg_{t,1,6}+\gg_{t,2,2}$ as
\begin{align*}
\gg_{t,1,3}+\gg_{t,1,5}+\gg_{t,1,6}+\gg_{t,2,2} &=\sum_{i=1}^{t-1}o_i\frac{\left(\Delta\rep_{i,t}-\Delta\rep_{i,i}\right)}{\wi^2} -  \sum_{i=1}^{t-1} \link(\th_{t-1}^\intercal\Delta\rep_{i,t})\frac{(\Delta\rep_{i,t}-\Delta\rep_{i,i})}{\wi^2}\\
&= \sum_{i = 1}^{t-1}\lt(o_i-\link(\th_{t-1}^\intercal\Delta\rep_{i, t})\rt)\frac{(\Delta\rep_{i, t} - \Delta\rep_{i, i})}{\wi^2}.
\end{align*}

It is clear that  $0< \link(x) < 1$, and  $0\leq o_i\leq 1.$
Therefore, we can obtain
\begin{align}
    \|\gg_{t,1,3}+\gg_{t,1,5}+\gg_{t,1,6}+\gg_{t,2,2}\|_2 \leq \frac{1}{\epsilon^2}\sum_{i=1}^{t-1}\|\Delta\rep_{i, t} - \Delta\rep_{i, i}\|_2,\no
\end{align} due to the definition of $\wi$ (in \cref{eq:app-wi}).  Thus, applying \cref{lemma:app-diffphiw} to $\|\Delta\rep_{i, t} - \Delta\rep_{i, i}\|_2$, we can show the following inequality: 
\begin{align}
&\|\gg_{t,1,3}+\gg_{t,1,5}+\gg_{t,1,6}+\gg_{t,2,2}\|_{V_{t-1}^{-1}}\leq\sqrt{\frac{1}{\lambda}}\left(\|\gg_{t,1,3}+\gg_{t,1,5}+\gg_{t,1,6}+\gg_{t,2,2}\|_2\right)\label{eq:fact2}\\
    &\leq \frac{C}{\epsilon^2}\sqrt{\frac{1}{\lambda}}t\left(\|\mathbf{\util} - \widetilde{\mathbf{\util}}\|_{\ntk^{-1}}^{4/3}\len^3\dim^{1/2}\sqrt{\log \wid}/(\wid^{1/6}) + 2\xb \repl \frac{\delta^{3/2}}{\sqrt{\wid}T\gs^{9/2}\log^3\wid}\right)\no\\
    &=\bigol\lt(\frac{d^{1/2}t\diff^{4/3}}{m^{1/6}\epsilon^{2}}\rt)\no
\end{align} for some constant $C>0$.

Now we estimate $\|\gg_{t,1,4}\|_{V_{t-1}^{-1}}$. By the mean-value theorem~\citep{rudin1964principles},
the following holds:
\begin{align}
    |\link(\ff_{i,t}(\th_{t-1}))-\link(\ff_{i,i}(\th_{t-1}))| &= |\dot{\link}(\overline{\ff}_{i,t})(\ff_{i,t}(\th_{t-1})-\ff_{i,i}(\th_{t-1}))|\nonumber\\
    &\leq \linku|\ff_{i,t}(\th_{t-1})-\ff_{i,i}(\th_{t-1})|,\no
\end{align} where $\overline{\ff}_{i,t} = c\ff_{i,t}(\th_{t-1})+(1-c)\ff_{i,i}(\th_{t-1})$ for some $c\in[0,1].$ Further, we compute that there exists a constant $C>0$ such that 
\begin{align}
\begin{aligned}
    &|\ff_{i,t}(\th_{t-1})-\ff_{i,i}(\th_{t-1})| = \left|\th_{t-1}^\intercal\Delta\rep_{i,i}-\th_{t-1}^\intercal\Delta\rep_{i,t}\right| \\&\leq 2C\|\th_{t-1}\|_2 
    \left(\|\mathbf{\util} - \widetilde{\mathbf{\util}}\|_{\ntk^{-1}}^{4/3}L^3d^{1/2}\sqrt{\log m}/(m^{1/6}) + 2G \repl \frac{\delta^{3/2}}{\sqrt{m}Tn^{9/2}\log^3m}\right)\\
    &\leq 2C\left(\sqrt{\frac{2t}{\lambda}} + 2\left(2+\sqrt{d^{-1}\log\left(\frac{1}{\delta}\right)}\right)\right)\\
    \s\times\left(\|\mathbf{\util} - \widetilde{\mathbf{\util}}\|_{\ntk^{-1}}^{4/3}\len^3\dim^{1/2}\sqrt{\log \wid}/(\wid^{1/6}) + 2\xb \repl \frac{\delta^{3/2}}{\sqrt{\wid}T\gs^{9/2}\log^3\wid}\right). 
\end{aligned}\label{eq:fit-fii}
\end{align} with probability at least \(1 - \delta\), where the first equality follows from the definition 
of \(\ff_{i,t}\) in \cref{eq:F}, the second inequality uses the Cauchy--Schwarz inequality 
and \cref{lemma:app-diffphiw}, and the third inequality results from \cref{lemma:app_theta_t_bound}. Therefore, there exists a constant $C>0$ such that 
\begin{align}
    &\|\gg_{t,1,4}\|_{V_{t-1}^{-1}} \leq \frac{\linku}{\epsilon}\sqrt{\frac{1}{\lambda}}\sum_{i=1}^{t-1}|\ff_{i,t}(\th_{t-1})-\ff_{i,i}(\th_{t-1})|\|\Delta\rep_{i,i}/\wi\|_{V_{t-1}^{-1}}\label{eq:fact3}\\
    &\leq C\sqrt{td}\frac{\linku}{\epsilon}\sqrt{\frac{1}{\lambda}} \left(\sqrt{\frac{2t}{\lambda}} + 2\left(2+\sqrt{d^{-1}\log\left(\frac{1}{\delta}\right)}\right)\right)\sqrt{\log n\log(TK/\delta)}\no\\
    \s\times\left(\|\mathbf{\util} - \widetilde{\mathbf{\util}}\|_{\ntk^{-1}}^{4/3}\len^3\dim^{1/2}\sqrt{\log \wid}/(\wid^{1/6}) + 2\xb \repl \frac{\delta^{3/2}}{\sqrt{\wid}T\gs^{9/2}\log^3\wid}\right)\no
    \\&=\bigol\lt(\frac{dt}{m^{-1/6}\epsilon}\lt(\diff^{4/3}\rt)\rt)\no
\end{align} where the second inequality is driven by \cref{eq:fit-fii} and the elliptic potential lemma in \cref{lemma:epl} to estimate the summation of $\|\Delta\rep_{i,i}/\wi\|_{V_{t-1}^{-1}}$.

{\color{black}\textbf{Martingale noise term $\gg_{t,3}$.} This term captures the stochastic noise from the Bernoulli preference feedback. Since $\varepsilon_i/\wi$ forms a bounded martingale difference sequence, we apply a Bernstein-type inequality to obtain a $\bigol(\sqrt{d}+\epsilon^{-1})$ bound---this is the variance-aware component that exploits the inverse-variance weighting.}

Finally, we estimate $\gg_{t,3}$. Since $|\varepsilon_i/\wi|\leq\frac{1}{\epsilon}$, $\mathbb{E}\lt[\varepsilon_i/\wi\mid\mathcal{F}_t\rt]=0$, and $\mathbb{E}\lt[(\varepsilon_i/\wi)^2\mid \mathcal{F}_t\rt]\leq 1,$ we can obtain the following bound by \cref{lemma:berstein}:
\begin{align}
    \|\gg_{t,3}\|_{V_{t-1}^{-1}}&\leq \frac{4}{\epsilon}\log(4t^2/\delta) \label{eq:fact4}\\ \s + 4\sqrt{\dim\log\left(1+t\left(2\frac{\sqrt{\dim\log \gs\log(TK/\delta)}}{\epsilon}\right)^2/(\dim\lambda)\right)\log(4t^2/\delta)} \no
    \\
    &=\bigol(\sqrt{d} + \epsilon^{-1})\no
\end{align} by hiding the dependence on all other hyperparameters within the $\bigol$ notation by \cref{lemma:epl}. 
Let us define
\begin{align}\mathcal{J}_1&=\|\gg_{t,3}\|_{V_{t-1}^{-1}} \label{eq:N_1}\\
\mathcal{K}_1(t)&=\|\gg_{t,1,1}+\gg_{t,2,1}+\gg_{t,1,1}+M_t+\gg_{t,1,3}+\gg_{t,1,5}+\gg_{t,1,6}+\gg_{t,2,2}+\gg_{t,1,4}\|_{V_{t-1}^{-1}}.\label{eq:M_1}
\end{align}

Combining the inequalities in \cref{eq:fact1,eq:fact2,eq:fact3,eq:fact4}, we can derive the following lemma.

\begin{lemma}\label{lemma:gt+mt} Let the conditions in \cref{lemma:app-linear}, \cref{lemma:app-upper}, and \cref{lemma:app-linear2} be satisfied. Under~\cref{eq:app-theta_t_bounded1} for $\th_{t-1}$, we can obtain
\begin{align}
    \begin{aligned}
    \|\gg_t(\th_{t-1})+M_t\|_{V_t^{-1}} \leq \mathcal{J}_{1} + \mathcal{K}_1(t)   
    \end{aligned},\no
\end{align} where
\begin{align}
    \mathcal{J}_{1} &\leq \frac{4}{\epsilon}\log(4t^2/\delta)\no\\
    \s + 4\sqrt{\dim\log\left(1+t\left(2\frac{\sqrt{\dim\log \gs\log(TK/\delta)}}{\epsilon}\right)^2/(\dim\lambda)\right)\log(4t^2/\delta)} \no\\
    &=\bigol\lt(\sqrt{d}+\epsilon^{-1}\rt), \no\\
\mathcal{K}_{1}(t)&\leq\frac{4}{\epsilon^2 }\sqrt{\frac{1}{\lambda}}t\left(\|\mathbf{\util} - \widetilde{\mathbf{\util}}\|_{\ntk^{-1}}^{4/3}\len^3\dim^{1/2}\sqrt{\log \wid}/(\wid^{1/6}) + 2\xb \repl \frac{\delta^{3/2}}{\sqrt{\wid}T\gs^{9/2}\log^3\wid}\right)\no\\\s+C\sqrt{t}\frac{\linku}{\epsilon^2}\sqrt{\frac{1}{\lambda}} \left(\sqrt{\frac{2t}{\lambda}} + 2\left(2+\sqrt{d^{-1}\log\left(\frac{1}{\delta}\right)}\right)\right)\sqrt{d\log n\log(TK/\delta)}\no\\        
    \s\times\left(\|\mathbf{\util} - \widetilde{\mathbf{\util}}\|_{\ntk^{-1}}^{4/3}\len^3\dim^{1/2}\sqrt{\log \wid}/(\wid^{1/6}) + 2\xb \repl \frac{\delta^{3/2}}{\sqrt{\wid}T\gs^{9/2}\log^3\wid}\right)\no
    \\
&=
\bigol\lt(\frac{td}{\wid^{1/6}\epsilon^2}\left(\diff^{4/3}\right)\rt)\no
\end{align} by hiding the dependence on all other hyperparameters within the $\bigol$ notation 
for some constant $C>0$ with probability at least $1-\delta$.
\end{lemma}

To conclude, we can obtain
\begin{align}
\|\gg_t(\th_{t-1})+M_t\|_{V_t^{-1}}\leq \bigol\lt(\frac{t}{m^{1/6}\epsilon^2}\|\mathbf{\util} - \widetilde{\mathbf{\util}}\|_{\ntk^{-1}}^{4/3}+\sqrt{d}+\epsilon^{-2}\rt).\label{eq:app-gm}
\end{align}  
{\color{black}\textbf{Width condition $m=\widetilde{\Omega}(t^6)$.} We now choose $m$ large enough so that the $t$-dependent term $\frac{t}{m^{1/6}\epsilon^2}\diff^{4/3}$ becomes negligible compared to $\sqrt{d}+\epsilon^{-2}$. This requires $m^{1/6}\gtrsim t\cdot(\text{poly}(d,\epsilon,\diff))$, i.e., $m=\widetilde{\Omega}(t^6)$.}

Now we determine the neural network width $\wid$ such that \cref{eq:app-gm} can be expressed as 
\begin{align}
    \|\gg_t(\th_{t-1})+M_t\|_{V_t^{-1}}\leq \bigol\lt(\sqrt{d}+\epsilon^{-2}\rt). \label{eq:app-wid-order}
\end{align}
\cref{eq:app-wid-order} is indeed true if $m\geq Ct^{6}$ for some $C=C(\lambda, \linku, L, 1/\epsilon, \sqrt{d}, \diff)$, i.e.,
\begin{align}
m=\widetilde{\Omega}(t^6), \label{eq:app-m-eq}
\end{align}ignoring all other parameters. 
 Importantly, the required network width $\wid$ in our analysis scales more favorably with respect to $t$ compared to the condition in Eq.~(8) of \citet{ndb}. While their regret bound holds only under a strong overparameterization assumption where $m$ must grow at least on the order of $t^{14}$, our result remains valid under a substantially milder requirement, with $m$ only needing to grow on the order of $t^6$, thereby highlighting the theoretical advantage of our approach.

Because $\lambda_{m}(V_{t-1})\geq \lambda,$ it is obvious that 
\begin{align}
&\lambda^{1/2} \|\th_{t-1}-\th_* + U_{t-1}^{-1}M_{t}\|_2 \leq \|\th_{t-1}-\th_* + U_{t-1}^{-1}M_{t}\|_{V_{t-1}}\no\\
&\leq \frac{1}{\linkb}\left(\lambda\left\|\th_*-\th_{0}\right\|_{V_{t-1}^{-1}}+\lt\|\gg_t(\th_{t-1})+M_t\rt\|_{V_{t-1}^{-1}}\right)\no
\\
&=\bigol\lt(\sqrt{d}+\epsilon^{-2}\rt)\no
\end{align} if \cref{eq:app-m-eq} is satisfied.
The second inequality results from \cref{lem:app-decom}. Then we can obtain
\begin{align}
 \|\th_{t-1}\|_2 &\leq \|\th_*\|_2 + \|U_{t-1}^{-1}M_{t}\|_2 \label{eq:tight_theta_t-1}\\
 \s + \frac{1}{\linkb}\lambda^{-1/2}\left(J
    +2\left(2+\sqrt{d^{-1}\log\left(\frac{1}{\delta}\right)}\right)\right) + \bigol\lt(\sqrt{d}+\epsilon^{-2}\rt)\no\\
    &\leq \bigol(d^{3/2}\|\mathbf{\util} - \widetilde{\mathbf{\util}}\|_{\ntk^{-1}}+\sqrt{d} + \epsilon^{-2})\no
\end{align} if \cref{eq:app-m-eq} is satisfied, where the last inequality results from \cref{lemma:um} and the $L_2$-boundedness of $\th_*$. \cref{eq:tight_theta_t-1} denotes that for a sufficiently large width $\wid$, $\th_{t-1}$ can be bounded by $L_2$ norm.

{\color{black}\textbf{Stage 3: Spectral perturbation analysis for the final refinement.} Having obtained the time-independent bound $\|\th_{t-1}\|_2 = \bigol(d^{3/2}+\epsilon^{-2})$, we now refine the bounds on the remaining terms. The key tools are the Courant--Fischer theorem (Lemma~\ref{lem:courant1}) and a Gram matrix perturbation bound (Lemma~\ref{lem:courant2}). These allow us to relate the $V_{t-1}^{-1}$-norm to a perturbed Gram matrix $\widetilde{U}_{t-1}^{-1}$-norm, absorbing the feature drift into a $(1+\bigol(t m^{-1/3}/\lambda))$ multiplicative factor. This spectral technique is what ultimately improves the bound from $\bigol(t/m^{1/6})$ to $\bigol(\sqrt{t}/m^{1/6})$.}

For a more refined estimation, we provide the following results based on spectral perturbation analysis.
\begin{lemma}\label{lem:courant1}
 Let $V, W \in \mathbb{R}^{d \times d}$ be symmetric positive definite (SPD) matrices. Then for all $x \in \mathbb{R}^d$, the following inequality holds:
\[
\lambda_{m}(W^{-1} V) \|x\|_{W}^2 \le \|x\|_V^2 \le \lambda_{M}(W^{-1} V) \|x\|_{W}^2.
\]
\end{lemma}
\begin{proof}
Since $V$ and $W$ are SPD, the matrix $W^{-1} V$ is also SPD and hence diagonalizable with positive real eigenvalues. Consider the generalized Rayleigh quotient:
\[
R(x) := \frac{x^\intercal V x}{x^\intercal W x} = \frac{\|x\|_V^2}{\|x\|_{W}^2}.
\]
By the Courant–Fischer Minimax theorem~\citep{horn2012matrix}, the following inequality holds for all nonzero $x \in \mathbb{R}^d$:
\[
\lambda_{m}(W^{-1} V) \le R(x) \le \lambda_{M}(W^{-1} V).
\]
Multiplying both sides by $\|x\|_{W}^2 = x^\intercal W x$ gives:
\[
\lambda_{m}(W^{-1} V) \cdot \|x\|_{W}^2 \le \|x\|_V^2 \le \lambda_{M}(W^{-1} V) \cdot \|x\|_{W}^2.
\]
Taking square roots (since all terms are positive), we obtain:
\[
\sqrt{\lambda_{m}(W^{-1} V)} \cdot \|x\|_{W} \le \|x\|_V \le \sqrt{\lambda_{M}(W^{-1} V)} \cdot \|x\|_{W}.
\]
\end{proof}
\begin{lemma}\label{lem:courant2}
Let $y_1, \dots, y_t \in \mathbb{R}^\dim$ be vectors satisfying $\|y_i\|_2 \leq B$. Define $z_i := y_t - y_i$, and assume that $\|z_i\|_2 \leq \bigol(m^{-1/6})$ for all $i$. Let the regularized Gram matrices be
\[
V := \lambda I + \sum_{i=1}^n y_i y_i^\intercal, \qquad W := \lambda I + \sum_{i=1}^n z_i z_i^\intercal,
\]
with $\lambda > 0$. 
Then the following bound holds:
\[
\|W V^{-1}\| = \sup_{x \neq 0} \frac{x^\intercal W x}{x^\intercal V x} \leq 1+\bigol\lt(\frac{tm^{-1/3}}{\lambda}\rt).
\]
\end{lemma}
\begin{proof}
Since $\|z_i\|_2 \leq C m^{-1/6}$, we have
\[
\|z_i z_i^\intercal\| = \|z_i\|_2^2 \leq C^2 m^{-1/3}.
\]
Therefore,
\[
\left\| \sum_{i=1}^t x_i x_i^\intercal \right\| \leq \sum_{i=1}^t \|x_i x_i^\intercal\| \leq t C^2 m^{-1/3}.
\]
Let $A := \sum_{i=1}^t y_i y_i^\intercal$ and $B := \sum_{i=1}^t z_i z_i^\intercal$. Then we write
\[
WV^{-1} = (\lambda I + B)^{-1} (\lambda I + A).
\]

We consider the Rayleigh quotient:
\[
\lambda_{M}(WV^{-1}) = \max_{\|x\| = 1} \frac{x^\intercal (\lambda I + B) z}{x^\intercal (\lambda I + A) x}.
\]

Now, for all unit vectors $x \in \mathbb{R}^d$,
\[
x^\intercal (\lambda I + B) x \leq \lambda + \|B\| \leq \lambda + t C^2 m^{-1/3},
\] for some constant $C>0$
and
\[
x^\intercal (\lambda I + A) x \geq \lambda.
\]

Therefore,
\[
\frac{x^\intercal (\lambda I + B) x}{x^\intercal (\lambda I + A) x} \leq \frac{\lambda + t C^2 m^{-1/3}}{\lambda} = 1 + \frac{t C^2}{\lambda m^{1/3}}.
\]

We conclude that
\[
\lambda_{M}(W^{-1}V) \leq 1 + \frac{t C^2}{\lambda m^{1/3}},
\]
which gives the bound:
\[
\lambda_{M}(W^{-1}V) = 1 + \mathcal{O}\left(\frac{t}{m^{1/3}}\right).
\]
This completes the proof.
\end{proof}

{\color{black}\textbf{Refined bound with the time-independent $\|\th_{t-1}\|_2$.} The following lemma replaces the loose $\sqrt{t}$-dependent parameter norm in Lemma~\ref{lemma:gt+mt} with the time-independent bound from Eq.~\eqref{eq:tight_theta_t-1}. This substitution, combined with the spectral perturbation bounds from Lemmas~\ref{lem:courant1}--\ref{lem:courant2}, reduces $\mathcal{K}_1(t)$ from $\bigol(td/(\epsilon^2 m^{1/6}))$ to $\bigol(\sqrt{t}d^2/(\epsilon^2 m^{1/6}))$---the critical improvement that yields sublinear regret.}

Utilizing $\text{\cref{lem:courant1}}$ and $\text{\cref{lem:courant2}}$, we establish the following lemma, which holds for a sufficiently large width $\wid$ satisfying $\text{\cref{eq:app-m-eq}}$,\textbf{ yielding a finer estimation result:}

\begin{lemma}\label{lemma:gt+mt2} Let the conditions in~\cref{lemma:gt+mt} be satisfied. Additionally, if the width $\wid$ satisfies \cref{eq:app-m-eq}, we can obtain
\begin{align}
    \begin{aligned}
    \|\gg_t(\th_{t-1})+M_t\|_{V_t^{-1}} \leq \mathcal{J}_{1} + \mathcal{K}_1(t)   
    \end{aligned},\no
\end{align} where $\mathcal{J}_{1}$ and  $\mathcal{K}_1(t)$ are defined in \cref{eq:N_1} and \cref{eq:M_1}, respectively, such that
\begin{align}
    \mathcal{J}_{1} &= \bigol\lt(\sqrt{d}+\epsilon^{-1}\rt), \no\\
\mathcal{K}_{1}(t)&=\bigol\lt(t^{1/2}\wid^{-1/6}\lt(\frac{d^2}{\epsilon^2}+d^3\rt)\diff^{4/3}+\dim^{1/2}\rt)\no,
\end{align}  with probability at least $1-\delta$.
\end{lemma}
\begin{proof}
It suffices to consider  the terms $\gg_{t,1,4}$ and 
$$
\mathcal{B}_t := \gg_{t,1,3}+\gg_{t,1,5}+\gg_{t,1,6}+\gg_{t,2,2}
$$ which are
appeared in \cref{lemma:gt+mt}. Using~\cref{eq:tight_theta_t-1} by~\cref{eq:app-m-eq}, we can replace the result in \cref{eq:fit-fii} as follows:
\begin{align}
\begin{aligned}
    &|\ff_{i,t}(\th_{t-1})-\ff_{i,i}(\th_{t-1})| = \left|\th_{t-1}^\intercal\Delta\rep_{i,i}+\th_{t-1}^\intercal\Delta\rep_{i,t}\right| \\&\leq C\|\th_{t-1}\|_2 
    \left(\|\mathbf{\util} - \widetilde{\mathbf{\util}}\|_{\ntk^{-1}}^{4/3}L^3d^{1/2}\sqrt{\log m}/(m^{1/6}) + 2G \repl \frac{\delta^{3/2}}{\sqrt{m}Tn^{9/2}\log^3m}\right)\\
    &= \bigol\left((d^2+d\epsilon^{-2})\|\mathbf{\util} - \widetilde{\mathbf{\util}}\|_{\ntk^{-1}}^{4/3}m^{-1/6}\right), \no
\end{aligned}\no
\end{align} for some constant $C>0$ with probability at least \(1 - \delta\). Therefore, we can compute
\begin{align}
\begin{aligned}
    &\|\gg_{t,1,4}\|_{V_{t-1}^{-1}} \leq \frac{\linku}{\epsilon}\sqrt{\frac{1}{\lambda}}\sum_{i=1}^{t-1}|\ff_{i,t}(\th_{t-1})-\ff_{i,i}(\th_{t-1})|\|\Delta\rep_{i,i}/\wi\|_{V_{t-1}^{-1}}\\
    &=\bigol\lt(\sqrt{td}\left((d^2+ d\epsilon^{-2})\|\mathbf{\util} - \widetilde{\mathbf{\util}}\|_{\ntk^{-1}}^{4/3}m^{-1/6}\right)\rt)
\end{aligned}\label{eq:fact6}
\end{align} 
with probability at least $1-\delta$, where the last inequality is driven by \cref{eq:fit-fii} and the elliptic potential lemma in \cref{lemma:epl} to estimate the summation of $\|\Delta\rep_{i,i}\|_{V_{t-1}^{-1}}$.
Besides, note that 
\begin{align*}
\mathcal{B}_t &= \sum_{i = 1}^{t-1}\lt(o_i-\link(\th_{t-1}^\intercal\Delta\rep_{i, t})\rt)\frac{(\Delta\rep_{i, t} - \Delta\rep_{i, i})}{\wi^2}.
\end{align*} 

Combining \cref{lem:courant1} and \cref{lem:courant2}, we can obtain
$$
\|\mathcal{B}_t\|_{V_{t-1}^{-1}} \leq \bigol\lt(\lt(1 +  \left(\frac{t}{\epsilon^4 m^{1/3}}\diff^{4/3}\right)\rt)\|\mathcal{B}_t\|_{\widetilde{U}_{t-1}^{-1}}\rt)
$$ where
$$
\widetilde{U}_{t-1}=\lambda I + \sum_{i=1}^{t-1}\frac{(\Delta\rep_{i, t} - \Delta\rep_{i, i})}{\wi^2}\frac{(\Delta\rep_{i, t} - \Delta\rep_{i, i})^\intercal}{\wi^2}.
$$
Since 
\begin{align}
    \mathbb{E}\lt[\lt(o_i-\link(\th_{t-1}^\intercal\Delta\rep_{i, t})\rt)\mid \mathcal{F}_{t}\rt]=0,\no\\
\mathbb{E}\lt[\lt(o_i-\link(\th_{t-1}^\intercal\Delta\rep_{i, t})\rt)^2\mid \mathcal{F}_{t}\rt]\leq 1,\no
\end{align}
we can apply \cref{lemma:berstein}, resulting in  
\begin{align}
    \|\mathcal{B}_t\|_{\widetilde{U}_{t-1}^{-1}}= \bigol\lt(\lt(1 + \left(\frac{t}{\epsilon^4 m^{1/3}}\right)\rt)d^{1/2}\rt).\label{eq:fact7}
\end{align} Combining \cref{eq:fact6,eq:fact7} and the other estimation results in \cref{lemma:gt+mt}, we can obtain
\begin{align}
\mathcal{K}_{1}(t)&=\bigol\lt(\lt(d\lt(\frac{d}{\epsilon^2}+d^2\rt)\frac{t^{1/2}}{\wid^{1/6}}+\frac{d^{1/2}}{\epsilon^4}\frac{t}{m^3}\rt)\|\mathbf{\util} - \widetilde{\mathbf{\util}}\|_{\ntk^{-1}}^{4/3}+\dim^{1/2}\rt). \label{eq:temp}
\end{align}  Notice that
    \begin{align}
        \frac{t}{m^{1/3}}=\lt(\frac{t^{1/2}}{m^{1/6}}\rt)^2=\lt(\frac{t}{m^{1/6}}\rt)\lt(\frac{t^{1/2}}{m^{1/6}}\rt)\frac{1}{\sqrt{t}}=\bigol\lt(\frac{t^{1/2}}{m^{1/6}}\rt)\label{eq:mixed}
    \end{align} due to \cref{eq:app-width}. Using \cref{eq:mixed}, we can complete the proof. 

\end{proof}

\begin{proof}[Proof of \cref{lem:app-concentration_inequality}] Thanks to \cref{lem:app-decom}, the following holds:
    \begin{align}
\lt\|\th_{t-1}-\th_* + U_{t-1}^{-1} M_{t}\rt\|_{V_{t-1}}&\leq \frac{1}{\linkb}\left(\lambda\left\|\th_*-\th_{0}\right\|_{V_{t-1}^{-1}}+\lt\|\gg_t(\th_{t-1})+M_t\rt\|_{V_{t-1}^{-1}}\right)\nonumber\\
&\leq \mathcal{J}'_1+\mathcal{K}'_1(t)\no
\end{align} where 
\begin{align}
    \mathcal{J}'_1 &= \mathcal{J}_1\frac{1}{\linkb}+\|\th_*-\th_0\|_{V_{t-1}^{-1}}\label{eq:n1'}\\
    \mathcal{K}'_1&= \frac{1}{\linkb}\mathcal{K}_1. \label{eq:m1'}
\end{align} Here, $\mathcal{J}_{1}$ and  $\mathcal{K}_1(t)$ are defined in \cref{eq:N_1} and \cref{eq:M_1}. 
Due to \cref{eq:app-m-eq}, we can apply \cref{lemma:gt+mt2} to estimate $\lt\|\gg_t(\th_{t-1})+M_t\rt\|_{V_{t-1}^{-1}}$. Since $\left\|\th_*-\th_{0}\right\|_{V_{t-1}^{-1}}$ is bounded independently of $T, \wid, \epsilon$ and $\diff$, it is obvious that
\begin{align}
    \mathcal{J}'_{1} &= \bigol\lt(d^{1/2}+\epsilon^{-1}\rt), \label{eq:n1'inequ}\\
\mathcal{K}'_{1}(t)&=\bigol\lt(\lt(\frac{d^2}{\epsilon^2}+d^3\rt)\frac{t^{1/2}}{\wid^{1/6}}\|\mathbf{\util} - \widetilde{\mathbf{\util}}\|_{\ntk^{-1}}^{4/3}+\dim^{1/2}\rt). \label{eq:m1'inequ}
\end{align} 
This completes the proof. 
\end{proof}

Although there exist several theoretical results on concentration inequalities for generalized linear bandits~\citep{kveton2020randomized, olddog, chen1999strong, li2017provably} and contextual linear dueling bandits~\citep{lst, saha2021optimal}, to the best of our knowledge, this is the first work that combines spectral perturbation analysis (the Courant–Fischer Minimax theorem) with a bootstrap argument. This refined approach is crucial for achieving a tight regret bound. A standard analysis that combines a concentration inequality with the elliptic potential lemma typically introduces an extra $\sqrt{T}$ factor, resulting in a suboptimal regret term on the order of $\bigol(T^{3/2}m^{-1/6})$. In contrast, our analysis circumvents this issue, yielding a tighter and more favorable bound that scales as $\bigol(Tm^{-1/6})$.

\section{THEORETICAL ANALYSIS OF UCB FRAMEWORK}\label{sec:UCB}

{\color{black}\textbf{Overview.} This section proves sublinear cumulative regret for all UCB-based arm selection strategies. The proof structure is: (1) derive strategy-specific bounds on $\th_{t-1}^\intercal(\cdot)$ relating the estimated utility gap to the exploration bonus (\cref{sec:UCB-ARM}); (2) decompose the instantaneous regret $r_t^a$ into representation bias, estimated utility gap, and parameter estimation error terms; (3) apply the concentration inequality from \cref{lem:app-concentration_inequality} and the elliptic potential lemma to sum the per-round bounds across $T$ rounds.}

\subsection{Arm Selection Strategies}\label{sec:UCB-ARM}
In this section, we consider three arm-selection strategies: 
Asymmetric Arm Selection (UCB-ASYM), Optimistic Symmetric Arm Selection (UCB-OSYM), 
and Candidate-Based Symmetric Arm Selection (UCB-CSYM).

We begin by analyzing the asymmetric strategy (UCB-ASYM):

\begin{lemma} [Asymmetric Arm Selection (UCB-ASYM)]
    Let the asymmetric arm selection strategy be defined as
\begin{align}
    \stt&=\argmax_{i \in [K]}\th_{t-1}^\intercal\rep(x_{t,i};\ww_{t-1}),\label{eq:app-strategy1-1}\\
    \ndt&=\argmax_{i\in [K]}\left(\th_{t-1}^\intercal\rep(x_{t,i};\ww_{t-1}) +  \cwct\|\Delta\rep_{t,i,\stt}\|_{V_{t-1}^{-1}}\right),\label{eq:app-strategy1-2}
\end{align} for some $\cwct>0$. Then we can obtain
\begin{align}
    \th_{t-1}^\intercal(2\Delta\rep_{t,\opt,\stt}+\Delta\rep_{t,\stt,\ndt}) &\leq \cwct\|\Delta\rep_{t,\stt,\ndt}\|_{V_{t-1}^{-1}}-\cwct\|\Delta\rep_{t,\opt,\stt}\|_{V_{t-1}^{-1}}.\label{eq:unsymmetry}
\end{align}
\end{lemma}
\begin{proof}
    Notice that
\begin{align}
    \th_{t-1}^\intercal\left(
    \Delta\rep_{t,\opt,\stt} + \Delta\rep_{t,\opt,\ndt}
    \right)
    &=2\th_{t-1}^\intercal(\Delta\rep_{t,\opt,\stt})+
    \th_{t-1}^\intercal(\Delta\rep_{t,\stt,\ndt}).\nonumber
\end{align}
\cref{eq:app-strategy1-1} results in 
\begin{align}
    \th_{t-1}^\intercal\Delta\rep_{t,\opt,\stt}\leq 0.\label{eq:arm1}
\end{align} 
In addition, \cref{eq:app-strategy1-2} results in 
\begin{align}
     &\th_{t-1}^\intercal\rep\left(x_{t,\opt};\ww_{t-1}\right) +\cwct\|\Delta\rep_{t,\opt,\stt}\|_{V_{t-1}^{-1}} \label{eq:arm2}\\&\leq  \th_{t-1}^\intercal \rep\left(x_{t,\ndt};\ww_{t-1}\right) + \cwct\|\Delta\rep_{t,\ndt,\stt}\|_{V_{t-1}^{-1}}.\no
\end{align}
Adding \cref{eq:arm1} and \cref{eq:arm2} leads to \cref{eq:unsymmetry}. This completes the proof.
\end{proof}
 Now we analyze two symmetric arm selection strategies as follows:
\begin{lemma}[Optimistic Symmetric Arm Selection (UCB-OSYM)]
    Let the optimistic symmetric arm selection strategy be defined as
\begin{align}
    (\stt,\ndt)=\argmax_{(i,j)\in[K]\times[K]}\left[\th_{t-1}^\intercal\rep(x_{t,i};\ww_{t-1}) +\th_{t-1}^\intercal\rep(x_{t,j};\ww_{t-1}) + \cwct\|\Delta\rep_{t,i,j}\|_{V_{t-1}^{-1}}\right],\label{eq:osym}
\end{align} for some $\cwct>0$.  Then,
\begin{align}
    &2\th_{t-1}^\intercal\rep(x_{t,\opt};\ww_{t-1}) - \th_{t-1}^\intercal\rep(x_{t,\stt};\ww_{t-1}) - \th_{t-1}^\intercal\rep(x_{t,\ndt};\ww_{t-1})\label{eq:sym1}\\
    &\leq -\cwct\|\Delta\rep_{t,\opt,\stt}\|_{V_{t-1}^{-1}} -\cwct\|\Delta\rep_{t,\opt,\ndt}\|_{V_{t-1}^{-1}}  + 2\cwct\|\Delta\rep_{t,\stt,\ndt}\|_{V_{t-1}^{-1}}.\no
\end{align} 
\end{lemma}
\begin{proof}
    Using the strategy as in \cref{eq:osym}, we can obtain
\begin{align}
\th_{t-1}^\intercal\rep(x_{t,\opt};\ww_{t-1}) &+ \th_{t-1}^\intercal\rep(x_{t,\stt};\ww_{t-1}) + \cwct\|\Delta\rep_{t,\opt,\stt}\|_{V_{t-1}^{-1}}\nonumber\\
    &\leq \th_{t-1}^\intercal\rep(x_{t,\stt};\ww_{t-1}) + \th_{t-1}^\intercal\rep(x_{t,\ndt};\ww_{t-1})  + \cwct\|\Delta\rep_{t,\stt,\ndt}\|_{V_{t-1}^{-1}},\nonumber\\
    \th_{t-1}^\intercal\rep(x_{t,\opt};\ww_{t-1}) &+ \th_{t-1}^\intercal\rep(x_{t,\ndt};\ww_{t-1}) + \cwct\|\Delta\rep_{t,\opt,\ndt}\|_{V_{t-1}^{-1}}\label{eq:b1}\\
    &\leq \th_{t-1}^\intercal\rep(x_{t,\stt};\ww_{t-1}) + \th_{t-1}^\intercal\rep(x_{t,\ndt};\ww_{t-1})  + \cwct\|\Delta\rep_{t,\stt,\ndt}\|_{V_{t-1}^{-1}}.\label{eq:b2}.
\end{align} Adding \cref{eq:b1} and \cref{eq:b2}, we can derive \cref{eq:sym1}. This completes the proof.
\end{proof}
\begin{lemma}[Candidate-based Symmetric Arm Selection (UCB-CSYM)]
    Let the confidence set be defined as
    \begin{align}
    \cs_t = \left\{i\in[K]:
    \th_{t-1}^\intercal\rep(x_{t,i};\ww_{t-1})- \th_{t-1}^\intercal\rep(x_{t,j};\ww_{t-1}) +\cwct\|
    \Delta\rep_{t,i,j}\|_{V_{t-1}^{-1}} \geq 0, \forall j\in [K]\right\}.\nonumber
\end{align} for some $\cwct$ and let two arms be chosen from $\cs_t$ as
\begin{align}
    (\stt, \ndt)=\argmax_{(i,j)\in \cs_t\times \cs_t}\|\Delta\rep_{t,i,j}\|_{V_{t-1}^{-1}}.\label{eq:strategy3-3}
\end{align} Then, 
\begin{align}
    &2\th_{t-1}^\intercal\rep(x_{t,\opt};\ww_{t-1}) - \th_{t-1}^\intercal\rep(x_{t,\stt};\ww_{t-1}) - \th_{t-1}^\intercal\rep(x_{t,\ndt};\ww_{t-1})\label{eq:sym2}\\
    &\leq -\cwct\|\Delta\rep_{t,\opt,\stt}\|_{V_{t-1}^{-1}}-\cwct\|\Delta\rep_{t,\opt,\ndt}\|_{V_{t-1}^{-1}} + 4\cwct\|\Delta\rep_{t,\stt,\ndt}\|_{V_{t-1}^{-1}}.
\end{align} 
\end{lemma} 
\begin{proof}
    Thanks to \cref{eq:strategy3-3}, for $\stt, \ndt\in C_t,$ we can get
    \begin{align}
    &\th_{t-1}^\intercal\left(2\rep\left(x_{t,\opt};\ww_{t-1}\right)-\rep\left(x_{t,\stt};\ww_{t-1}\right) - \rep\left(x_{t,\ndt};\ww_{t-1}\right)\right)\label{eq:c0}\\&\leq \cwct \|\Delta\rep_{t,\stt,\opt}\|_{V_{t-1}^{-1}} + \cwct \|\Delta\rep_{t,\ndt,\opt}\|_{V_{t-1}^{-1}}.
\end{align} Due to \cref{eq:strategy3-3},
\begin{align}
     \|\Delta\rep_{t,\stt,\opt}\|_{V_{t-1}^{-1}} &\leq  2\|\Delta\rep_{t,\stt,\ndt}\|_{V_{t-1}^{-1}} - \|\Delta\rep_{t,\stt,\opt}\|_{V_{t-1}^{-1}},\label{eq:c1}\\
     \|\Delta\rep_{t,\ndt,\opt}\|_{V_{t-1}^{-1}} &\leq  2\|\Delta\rep_{t,\stt,\ndt}\|_{V_{t-1}^{-1}} - \|\Delta\rep_{t,\ndt,\opt}\|_{V_{t-1}^{-1}}.\label{eq:c2}
\end{align}
Substituting \cref{eq:c1} and \cref{eq:c2} to \cref{eq:c0} leads to \cref{eq:sym2}, thus we can complete the proof.
\end{proof}

\subsection{Theoretical Cumulative Average Regret Analysis under UCB Regime}
We restate the theoretical results under Upper Confidence Bound (UCB) regime, already stated in the main manuscript.

\begin{theorem}[Variance-aware UCB methods]\label{theorem:app-main_ucb}
    Let us assume \cref{assumption:app-contextbdd}-\cref{assumption:app-lambda_min}. Let the confidence width coefficient  $\cwct$ for the arm selection methods and the step size $\lr$ be chosen to satisfy the following:
    \begin{align}
        \cwct &= \frac{8}{\linkb}\left(\sqrt{\dim\log (\cwct')\log(4t^2/\delta)} + \frac{1}{\epsilon}\log(4t^2/\delta)\right) \no\\\s+ 2\sqrt{\lambda}\left(J
    +2\left(2+\sqrt{d^{-1}\log\left(\frac{1}{\delta}\right)}\right)\right)\no
,\\
        \lr &\leq C\left(\dim^2\wid \gs T^{6}\len^6\log(TK/\delta)\right)^{-1}\no
    \end{align} for some $C>0$ with
    $$
        \cwct'=\left(1+t\left(2\frac{\sqrt{\dim\log \gs\log(TK/\delta)}}{\sqrt{\dim\lambda}\epsilon}\right)^2\right).\nonumber
    $$ If the width $\wid$ of $\nn$ is sufficiently large as in \cref{lem:concentration_inequality}
    and $\epsilon=\bigo(1/\sqrt{d})$, then the cumulative average regret $R_T^a$ under any of our UCB-based arm selection methods is upper-bounded as
     \begin{align}
        R_T^a = \bigol\left(\sqrt{dT}+d\sqrt{\sum_{i=1}^{T}\sigma_i^2}+\frac{Td^{3}}{ m^{1/6}}\|\mathbf{\util} - \widetilde{\mathbf{\util}}\|_{\ntk^{-1}}^{4/3}\right).\label{eq:app-variance-aware}
    \end{align} ignoring all other hyperparameters with probability at least $1-\delta$ over the randomness of the initialization of the neural network $\nn$.
\end{theorem} The following result is about variance-agnostic UCB-based algorithms.
\begin{corollary}[Variance-agnostic UCB methods]\label{theorem:app-main_ucb2}
Assume the conditions in \cref{theorem:main_ucb} hold, but we fix $\wt=1$ (i.e., $\epsilon=1$) for all round $0\leq t \leq T$.  Then the cumulative regret $R^{a}_t$ under any of our UCB-based arm selection methods is bounded as
    \begin{align}
        R_T^a = \bigol\left(d\sqrt{T} + \frac{Td^{3}}{ m^{1/6}}\|\mathbf{\util} - \widetilde{\mathbf{\util}}\|_{\ntk^{-1}}^{4/3}\right). \label{eq:app-variance-agnostic}
    \end{align} ignoring all other hyperparameters with probability at least $1-\delta$ over the randomness of learnable parameters in the neural network $\nn$
\end{corollary}

{\color{black}\textbf{Regret decomposition.} The instantaneous regret $2r_t^a$ is decomposed into three pairs of terms: $\mathcal{A}_{t,1,\cdot}$ captures the \emph{representation learning bias} (difference between the true utility and the linearized approximation around $\ww_0$); $\mathcal{A}_{t,2,\cdot}$ captures the \emph{estimated utility gap} (controlled by the arm selection strategy bounds above); and $\mathcal{A}_{t,3,\cdot}$ captures the \emph{parameter estimation error} $(\th_*-\th_{t-1})$, which is bounded by our concentration inequality. The bias terms $\mathcal{A}_{t,1,\cdot}$ and the inner product terms $\mathcal{K}_3$ vanish as $m\to\infty$, while $\mathcal{A}_{t,2,\cdot}$ and $\mathcal{A}_{t,3,\cdot}$ yield the exploration bonus terms that are summed via the elliptic potential lemma.}

To prove the results above, we decompose the average regret as follows by \cref{lemma:app-linear} with the notations in \cref{eq:not}:
\begin{align}
    2r_t^{a} &= \mathcal{A}_{t,1,1}+\mathcal{A}_{t,1,2} + \mathcal{A}_{t,2,1}+\mathcal{A}_{t,2,2}+\mathcal{A}_{t,3,1} + \mathcal{A}_{t,3,2}\nonumber
\end{align} where
\begin{align}
    \mathcal{A}_{t,1,1}&=\th_{0}^\intercal  \left(\nabla \Delta\rep_{t,\opt,\stt,1} \right)\left(\ww_{*}-\ww_{0}\right),\nonumber
    \\
   \mathcal{A}_{t,1,2}&= \th_{0}^\intercal  \left(\nabla \Delta\rep_{t,\opt,\ndt,1}\right)\left(\ww_{*}-\ww_{0}\right),\nonumber
   \\A_{t,2,1}&=\th_{t-1}^\intercal\Delta\rep_{t,\opt,\stt} ,\nonumber\\
    \mathcal{A}_{t,2,2}&= \th_{t-1}^\intercal\Delta\rep_{t,\opt,\ndt}, \nonumber\\
    \mathcal{A}_{t,3,1}&= (\th_*-\th_{t-1})^\intercal\Delta\rep_{t,\opt,\stt}, \nonumber\\
    \mathcal{A}_{t,3,2}&= (\th_*-\th_{t-1})^\intercal\Delta\rep_{t,\opt,\ndt}.\nonumber
\end{align}

\begin{lemma} With probability at least $1-\delta$, the following holds:
    \begin{align}
    \mathcal{A}_{t,1,1}+\mathcal{A}_{t,1,2}&\leq \repl\|\th_{0}\|_2\|x_{t,\stt}-x_{t,\ndt}\|_2\|\ww_{*}-\ww_{0}\|_2\nonumber\\&\leq  \mathcal{K}_2(\wid)\label{eq:A11-A12}
    \end{align} where
    \begin{align}
        \mathcal{K}_2 &=  4\xb\repl\left(2+\sqrt{\dim^{-1}\log\left(\frac{1}{\delta}\right)}\right) \frac{\|\mathbf{\util} - \widetilde{\mathbf{\util}}\|_{\ntk^{-1}}}{\sqrt{\wid}} =\bigol\lt(\frac{\|\mathbf{\util} - \widetilde{\mathbf{\util}}\|_{\ntk^{-1}}}{\sqrt{m}}\rt).
        \label{eq:M_2}
        \end{align}
\end{lemma}
\begin{proof}
It is obvious
    \begin{align}
    |\mathcal{A}_{t,1,1}+\mathcal{A}_{t,1,2}|&\leq \left\|\th_{0}\right\|_2\left\|\left(\nabla\rep\left(x_{t,\opt};\ww_{0}\right)-\nabla\rep\left(x_{t,\stt};\ww_{0}\right)\right)\left(\ww_{*}-\ww_{0}\right)\right\|_2\nonumber\\
    \s + \left\|\th_{0}\right\|_2\left\|\left(\nabla\rep\left(x_{t,a^*_t};\ww_{0}\right)-\nabla\rep\left(x_{t,\ndt};\ww_{0}\right)\right)\left(\ww_{*}-\ww_{0}\right)\right\|_2
    \nonumber
\end{align}
Notice that 
\begin{align}
    \|Mx\|_2\leq \|M\|_2\|x\|_2\nonumber
\end{align} for a matrix $M$ and vector $x$.  By \cref{eq:theta0}, \cref{assumption:app-contextbdd},~\cref{assumption:app-phi_lipschitz} and~\cref{lemma:app-linear}, we can show~\cref{eq:A11-A12}. This completes the proof.
\end{proof}

$\mathcal{A}_{t,2,1}+\mathcal{A}_{t,2,2}$ can be expressed as
\begin{align}
&\mathcal{A}_{t,2,1}+\mathcal{A}_{t,2,2}\no\\&=\th_{t-1}^\intercal\left(2\rep\left(x_{t,\opt};\ww_{t-1}\right)-\rep\left(x_{t,\stt};\ww_{t-1}\right) - \rep\left(x_{t,\ndt};\ww_{t-1}\right)\right)\no\\
&=\th_{t-1}^\intercal\left(2\left(\rep\left(x_{t,\opt};\ww_{t-1}\right)-\rep\left(x_{t,\stt};\ww_{t-1}\right)\right)\right)\no
 \\\s + \th_{t-1}^\intercal \left(\rep\left(x_{t,\stt};\ww_{t-1}\right)-\rep\left(x_{t,\ndt};\ww_{t-1}\right)\right),\no
\end{align} 
where the first equality applies to symmetric arm selections, while the second equality pertains to the asymmetric arm selection. Using the results in \cref{sec:UCB-ARM}, the following holds:
\begin{align}
    \text{(UCB-ASYM):}\quad\mathcal{A}_{t,2,1}+\mathcal{A}_{t,2,2}&\leq 2\cwct\|\Delta\rep_{t,\stt,\ndt}\|_{V_{t-1}^{-1}}-2\cwct\|\Delta\rep_{t,\opt,\stt}\|_{V_{t-1}^{-1}}.
    \no\\\text{(UCB-OSYM):}\quad
    \mathcal{A}_{t,2,1}+\mathcal{A}_{t,2,2}&\leq 2\cwct\|\Delta\rep_{t,\stt,\ndt}\|_{V_{t-1}^{-1}}\no\\\s-\cwct\|\Delta\rep_{t,\opt,\stt}\|_{V_{t-1}^{-1}}-\cwct\|\Delta\rep_{t,\opt,\ndt}\|_{V_{t-1}^{-1}}.
\no\\ \text{(UCB-CSYM):}\quad
    \mathcal{A}_{t,2,1}+\mathcal{A}_{t,2,2}&\leq 4\cwct\|\Delta\rep_{t,\stt,\ndt}\|_{V_{t-1}^{-1}}\no\\\s-\cwct\|\Delta\rep_{t,\opt,\stt}\|_{V_{t-1}^{-1}}-\cwct\|\Delta\rep_{t,\opt,\ndt}\|_{V_{t-1}^{-1}}.\no
\end{align}

$\mathcal{A}_{t,3,1}+\mathcal{A}_{t,3,2}$ can be decomposed into
\begin{align}
    \mathcal{A}_{t,3,1}+\mathcal{A}_{t,3,2}&=\left(\th_* - \th_{t-1}+U_{t-1}^{-1}M_t\right)^\intercal\Delta\rep_{t,\opt,\stt}  + \left(\th_* - \th_{t-1}+U_{t-1}^{-1}M_t\right)^\intercal\Delta\rep_{t,\opt,\ndt}\nonumber\\ \s
    - (U_{t-1}^{-1}M_t)^\intercal\Delta\rep_{t,\opt,\stt} - (U_{t-1}^{-1}M_t)^\intercal\Delta\rep_{t,\opt,\ndt} \nonumber
    \\ &= 2\left(\th_* - \th_{t-1}+U_{t-1}^{-1}M_t\right)^\intercal\Delta\rep_{t,\opt,\stt}  + \left(\th_* - \th_{t-1}+U_{t-1}^{-1}M_t\right)^\intercal\Delta\rep_{t,\stt,\ndt}\nonumber\\ \s
    - 2(U_{t-1}^{-1}M_t)^\intercal\Delta\rep_{t,\opt,\stt} - (U_{t-1}^{-1}M_t)^\intercal\Delta\rep_{t,\stt,\ndt}.  \nonumber
\end{align} The first inequality is used for symmetric arm selections, while the second equality pertains to the asymmetric arm selection. The last two terms are estimated using the following lemma. 

\begin{lemma} \label{lemma:umphi} With probability at least $1-\delta$, the following holds:
\begin{align}
    |(U_{t-1}^{-1}M_t)^\intercal\Delta\rep_{t,i,j}|\leq \mathcal{K}_3 \nonumber,
\end{align} where
\begin{align}
\begin{aligned}
    \mathcal{K}_3 &=\left(\lt\|\mathbf{\util} - \widetilde{\mathbf{\util}}\rt\|_{\ntk^{-1}}^{4/3}L^3d^{1/2}\sqrt{\log m}/(m^{1/6}) + 2G \repl \frac{\delta^{3/2}}{\sqrt{m}n^{9/2}\log^3m}\right)\\\s\times 2\dim \repl \xb \frac{\linku}{\epsilon}\left(2+\sqrt{\dim^{-1}\log\left(\frac{1}{\delta}\right)}\right)\sqrt{\dim\len}\|\mathbf{\util} - \widetilde{\mathbf{\util}}\|_{\ntk^{-1}}\\
&=\bigol\lt(m^{-1/6}\diff^{4/3}d^{3/2}\epsilon^{-1}\rt).
\end{aligned}\label{eq:M_3}
\end{align} 
\end{lemma}
\begin{proof}
Note that
$$
|(U_{t-1}^{-1}M_t)^\intercal\Delta\rep_{t,i,j}|\leq \|U_{t-1}^{-1}M_t\|_2\|\Delta\rep_{t,i,j}\|_2.
$$ by the Cauchy–Schwarz inequality. By applying \cref{lemma:um} to $\|U_{t-1}^{-1}M_t\|_2$ and \cref{lemma:diffphix} to $\|\Delta\rep_{t,i,j}\|_2,$ we can complete the proof.
\end{proof}
\begin{proof}[Proof of~\cref{theorem:app-main_ucb2}]
Under the asymmetric arm selection,
\begin{align}
2r_t^{a} &\leq \mathcal{K}_2 + 2\mathcal{K}_3 + \cwct\|\Delta\rep_{t,\stt,\ndt}\|_{V_{t-1}^{-1}}-\cwct\|\Delta\rep_{t,\opt,\stt}\|_{V_{t-1}^{-1}}\\
\s+\left(\mathcal{K}'_1(t) + \mathcal{J}_1'\right)\|\Delta\rep_{t,\stt,\ndt}\|_{V_{t-1}^{-1}}+\left(\mathcal{K}'_1(t) + \mathcal{J}_1'\right)\|\Delta\rep_{t,\opt,\stt}\|_{V_{t-1}^{-1}},
\end{align} where the last two terms result from \cref{lem:app-concentration_inequality}. Therefore, choosing $a_t= \mathcal{K}'_1(t) + \mathcal{J}_1'$, we can obtain 
\begin{align}
    2r_t^{a} & =\bigol\lt(\mathcal{K}_2 + \mathcal{K}_3 + \left(\mathcal{K}'_1(t) + \mathcal{J}_1'\right)\|\Delta\rep_{t,\stt,\ndt}\|_{V_{t-1}^{-1}}\rt).
\end{align} where $\mathcal{J}'_1, \mathcal{K}'_1, \mathcal{K}_2, \mathcal{K}_3$ are defined in \eqref{eq:n1'}, \eqref{eq:m1'}, \eqref{eq:M_2}, \eqref{eq:M_3}.

In a similar manner to the asymmetric arm selection, we can obtain equivalent inequalities under the two symmetric arm selections. In other words, there exists $C>0$ such that
\begin{align}
2r_t^{a} & =\bigol\lt(\left(\mathcal{K}_2 + \mathcal{K}_3 + \left(\mathcal{K}'_1(t) + \mathcal{J}_1'\right)\|\Delta\rep_{t,\stt,\ndt}\|_{V_{t-1}^{-1}}\right)\rt),\label{eq:ucb-final}
\end{align} independently of the arm selection strategies, if $\cwct =\mathcal{K}'_1(t) + \mathcal{J}_1'.$  Note  that $\mathcal{J}'_1 = \bigol(\sqrt{d}).$  Therefore, if we set $\wt=1$, which is the variance-agnostic case. By summing up to $T$, and applying \cref{lemma:epl},  we can show that
\begin{align}
     2\sum_{t=1}^{T}r_t^{a} & = 2\sqrt{T}\lt(\sum_{t=1}^{T}|r_t^a|^2\rt)^{1/2}\no\\ &=\bigol\lt(\sqrt{T}\lt(d+d^3\lt(T^{1/2}M^{-1/6}\rt)\diff^{4/3}\rt)\rt) \no\\
    &= \bigol\lt(d\sqrt{T}+d^3\lt(TM^{-1/6}\rt)\diff^{4/3}\rt).
\end{align} Thus, we can prove \cref{theorem:main_ucb2}. 
\end{proof}

\begin{proof}[Proof of \cref{theorem:app-main_ucb}]
We now set $\epsilon = \bigol\left(\frac{1}{\sqrt{d}}\right)$. 
Since 
\[
\Delta\rep_{t,\stt,\ndt} = \wt\left(\frac{\Delta\rep_{t,\stt,\ndt}}{\wt}\right),
\]
summing both sides of~\cref{eq:ucb-final} from $t = 1$ to $T$, we apply the Cauchy–Schwarz inequality and subsequently use the elliptic potential lemma in~\cref{lemma:epl} to obtain the following bound:

\begin{align}
    2\sum_{t=1}^{T}r_t^{a}&\leq
     C\lt(\mathcal{J}'_1+\mathcal{K}'_1\rt) \left(\sum_{t=1}^{T}\lt\|\frac{\Delta\rep_{t,\stt,\ndt}}{\wt}\rt\|_{V_{t-1}^{-1}}^2\right)^{1/2}
    \left(\sum_{t=1}^{T}\wt^2\right)^{1/2}+\mathcal{R}_1(\wid,T)\label{eq:app-ucb-regret}\nonumber
    \\&\leq  C\sqrt{d}\sqrt{2\dim\log\left(1+T\left(\dim \log\gs \log(TK/\delta)\right)^2/\lambda\right)}
    \left(\sum_{t=1}^{T}\wt^2\right)^{1/2}+\mathcal{R}_1(\wid, T)\nonumber
    \\
    &=\bigol\lt(d\left(\sum_{t=1}^{T}\wt^2\right)^{1/2}+\mathcal{R}_1(\wid,T)\rt),\no
\end{align}
where
\begin{align}
\mathcal{R}_1(\wid, T) &= \bigol\lt(d^3\lt(TM^{-1/6}+
\lt(TM^{-1/6}\rt)^{2}/(\sqrt{T})\rt)\diff^{4/3}\rt)\no\\&=
\bigol\lt(d^3\lt(TM^{-1/6}\rt)\diff^{4/3}\rt) \no
\end{align} by hiding the dependence on all other hyperparameters within the $\bigol$ notation. The last term results from \cref{eq:app-m-eq}. Note  that $\mathcal{J}'_1 = \bigol(\sqrt{d}).$  
It is obvious that 
\begin{align}
    \wi^2 = \max\left\{\epsilon^2, \widehat{\sigma}_i^2\right\}\leq \epsilon^2+\widehat{\sigma}_i^2\nonumber.
\end{align} $\widehat{\sigma}_i^2$ is a estimation of variance by $(\th_{i}, \ww_{i})$. Since $\epsilon = \bigol\left(\frac{1}{\sqrt{d}}\right)$ and
\begin{align}
    \sqrt{a+b}\leq \sqrt{a}+\sqrt{b}\no
\end{align} for $a,b>0,$ we can estimate
\begin{align}
    \left(\sum_{i=1}^{T}\wi^2\right)^{\frac{1}{2}}&\leq
    \epsilon\sqrt{T}+\sqrt{\sum_{i=1}^{T}\widehat{\sigma}_i^2}\nonumber\\
    &=\bigol\lt(\sqrt{\frac{T}{\dim}}+\sqrt{\sum_{i=1}^{T}\widehat{\sigma}_i^2}\rt)
\end{align}
As a result, utilizing \cref{eq:app-ucb-regret}, we can estimate cumulative average regret as 
\begin{align}
    2\sum_{i=1}^{T}r_t^{a}&=\bigol\lt(\sqrt{dT}+d\sqrt{\sum_{i=1}^{T}\widehat{\sigma}_i^2}\rt)+\mathcal{R}_1(\wid, T)\no
    \\&=\bigol\lt(\sqrt{dT}+d\sqrt{\sum_{i=1}^{T}\widehat{\sigma}_i^2}+T^2m^{-1/6}d^3\diff^{4/3}\rt).\no
\end{align}

Now we estimate the difference between $\sigma_t^2$ a real variance with respect to $x_{t,\stt}$ and $x_{t,\ndt},$ and the estimated variance $\widehat{\sigma}_t^2$. Let
\begin{align}
    \widehat{\link}_t &= \link(\th_{t-1}^\intercal\Delta\rep_{t})\no\\
    \link_t &= \link(\th_*^\intercal\Delta\rep_t+\th_{0}^\intercal\nabla\Delta\rep_{t,1}(\ww_{*}-\ww_{0})).\no
\end{align}We can compute that
\begin{align}
    |\sigma^2_t - \widehat{\sigma}^2_t|&=|\link_t(1-\link_t)-\widehat{\link}_t(1-\widehat{\link}_t)|=|\link_t-\widehat{\link}_t-\link_t^2+\widehat{\link}_t^2|\no\\&=|(\link_t-\widehat{\link}_t)(1-(\link_t+\widehat{\link}_t))|.\no
\end{align} 
Note that the mean-value theorem leads to
\begin{align}
    |\link_t-\widehat{\link}_t|\leq \utilu|(\th_{t-1}-\th_*)^\intercal\Delta\rep_t + \th_{0}^\intercal\nabla\Delta\rep_{t,1}(\ww_{*}-\ww_{0})|. \label{eq:ucb-eq0}
\end{align} 
Due to \cref{assumption:app-thetastarbdd} and \cref{eq:app-theta_t_bounded1}, 
$$\|\th_{t-1}-\th_*\|_2=\bigol\lt(d^{3/2}\lt(1+\diff^{4/3}\rt)\rt)$$ since \cref{eq:app-m-eq} is satisfied. Moreover, due to \eqref{eq:diffphi_xx'},
\begin{align}
    \|\Delta\rep_t\|_2=\bigol\left(\lt\|\mathbf{\util} - \widetilde{\mathbf{\util}}\rt\|_{\ntk^{-1}}^{4/3}d^{1/2}m^{-1/6}\right). \label{eq:ucb-eq1}
\end{align} Besides, since $\th_{0}$ is bounded with probability at least $1-\delta$ by \cref{lemma:inittheta0} and the Lipschitz condition of $\nabla_W\rep(x,W_0)$ by \cref{assumption:phi_lipschitz}, we can  have
\begin{align}
    |\th_{0}^\intercal\nabla\Delta\rep_{t,1}(\ww_{*}-\ww_{0})|=\bigol(\diff m^{-1/2}) \label{eq:ucb-eq2}
\end{align} as $$|\ww_{*}-\ww_{0}|\leq \frac{\|\widetilde{\mathbf{\util}}-\mathbf{\util}\|_{\ntk^{-1}}}{\sqrt{m}}.$$
Thus, using \cref{eq:ucb-eq0,eq:ucb-eq1,eq:ucb-eq2}, we can show that
\begin{align}
     |\sigma^2_t - \widehat{\sigma}^2_t|\leq \mathcal{R}_2(T,\wid) 
\end{align}
where 
\begin{align}
    \mathcal{R}_2(\wid, T) &=\bigol(d^2m^{-1/6}\diff^{4/3})
\end{align} by hiding the dependence on all other hyperparameters within the $\bigol$ notation for sufficiently large $\wid$ such that \cref{eq:app-m-eq} is satisfied.
As a result, since $\epsilon = \bigol\left(\frac{1}{\sqrt{d}}\right)$, we can obtain
\begin{align}
    2\sum_{i=1}^{T}r_t^{a}&=\bigol\lt(d\epsilon\sqrt{T}+d\sqrt{\sum_{i=1}^{T}\widehat{\sigma}_i^2}\rt)+\mathcal{R}_1(\wid,T)\nonumber\\
    &=\bigol\lt(\sqrt{dT}+d\sqrt{\sum_{i=1}^{T}\sigma_i^2}\rt)+\mathcal{R}_1(\wid,T)+T\mathcal{R}_2(\wid,T)\no
\\&=\bigol\lt(\sqrt{dT}+d\sqrt{\sum_{i=1}^{T}\sigma_i^2}+m^{-1/6}Td^3\diff^{4/3}\rt).\no
\end{align} This completes the proof of \cref{theorem:main_ucb}.
\end{proof}
\section{THEORETICAL CUMULATIVE AVERAGE REGRET ANALYSIS UNDER TS REGIME}\label{sec:ts}

{\color{black}\textbf{Overview.} The TS proof strategy differs from UCB. Instead of directly bounding $r_t^a$, we construct a super-martingale $Y_t = \sum_{i=1}^t (r_i^a I_{E_i} - c_t)$ and apply the Azuma--Hoeffding inequality. The key steps for each arm selection strategy are: (1) define ``good events'' $E_t$ (concentration holds) and $F_t$ (Gaussian samples are well-behaved); (2) define a ``saturation set'' $S_t$ of arm pairs whose utility gap is large enough that they are unlikely to be selected; (3) show that the selected arms $(\stt,\ndt)$ fall outside $S_t$ with positive probability $p$; (4) use this to bound $\mathbb{E}[r_t^a|\mathcal{F}_{t-1}]$ by exploration bonus terms. The TS-OSYM and TS-CSYM strategies are more complex because both arms are drawn simultaneously from a joint posterior, requiring careful probability bounds via independence arguments.}

We restate theoretical results for TS algorithms, already explained in the main manuscript, as follows.

\begin{theorem}[Variance-aware TS methods]\label{theorem:app-main_ts}
   Assume the conditions in \cref{theorem:app-main_ucb} hold. Then  the cumulative regret $R^{a}_t$ under any of TS-based arm selection methods is bounded as \cref{eq:app-variance-aware} with probability at least $1-\delta$ over the randomness of learnable parameters in the neural network $\nn$.
\end{theorem} 
In the case of variance-agnostic TS algorithms, we have the following result.
\begin{corollary}[Variance-agnostic TS methods]\label{theorem:app-main_ts2}
Assume the conditions in \cref{theorem:app-main_ucb2} hold. Then the cumulative regret $R^{a}_t$ under any of our TS-based arm selection methods  is bounded as \eqref{eq:app-variance-agnostic} with probability at least $1-\delta$ over the randomness of learnable parameters in the neural network $\nn$.
\end{corollary}




Let the difference of the utilities resulting from $x_{t,\arm}$ and $x_{t,\arm'}$ for $\arm, \arm'\in[K]$ as
$$\Delta \util_{t,\arm,\arm'}=\util(x_{t,\arm})-\util(x_{t,\arm'}).$$
Due to \cref{lem:app-concentration_inequality} and \cref{lemma:um}, we have the following inequality: 
    \begin{align}
    \begin{aligned}
    |\Delta \util_{t,\arm,\arm'}-\th_{t-1}^\intercal\Delta\rep_{t,\arm,\arm'}| &= |(\th_*-\th_{t-1})^\intercal\Delta\rep_{t,\arm,\arm'} + \th_{0}^\intercal \nabla\Delta\rep_{t,\arm,\arm',1}(\ww_{*}-\ww_{0})|\\
    &\leq \|\th_*-\th_{t-1}+U_{t-1}^{-1}M_t\|_{V_{t-1}}\|\Delta\rep_{t,\arm,\arm'}\|_{V_{t-1}^{-1}}\\
    \s + |(U_{t-1}^{-1}M_t)^\intercal\Delta\rep_{t,\arm,\arm'}| + |\th_{0}^\intercal\nabla\Delta\rep_{t,\arm,\arm',1}(\ww_{*}-\ww_{0})|\\ 
    &\leq (\mathcal{J}'_1+\mathcal{K}'_1)\|\Delta\rep_{t,\arm,\arm'}\|_{V_{t-1}^{-1}} + \mathcal{K}_3 + 2C\xb \frac{\|\widetilde{\mathbf{\util}}-\mathbf{\util}\|_{\ntk^{-1}}}{\sqrt{m}}\\
    &= (\mathcal{J}'_1+\mathcal{K}'_1)\|\Delta\rep_{t,\arm,\arm'}\|_{V_{t-1}^{-1}} + \mathcal{R}(m),
    \end{aligned}
    \label{eq:tsf}
\end{align} for some constant $C>0$ with probability at least $1-\delta$, where $\mathcal{J}'_1$ is defined in \cref{eq:n1'}, $\mathcal{K}'_1$ is defined in \cref{eq:m1'}, $\mathcal{K}_3$ defined in \cref{eq:M_3}. Note that $$\mathcal{R}(m)=\bigol\left(\dim^2 m^{-1/6}\|\widetilde{\mathbf{\util}}-\mathbf{\util}\|_{\ntk^{-1}}^{4/3}\right)$$ with probability at least $1-\delta.$ Using \cref{lemma:diffphix} for $\|\Delta\rep_{t,t}\|_{2}$, we can show that 
$$
|\Delta \util_{t,t}-\th_{t-1}^\intercal\Delta\rep_{t,t}|=\bigol(dm^{-1/6}\diff^{4/3}).
$$ In other words, for sufficiently large $\wid$, $|\Delta \util_{t,t}-\th_{t-1}^\intercal\Delta\rep_{t,t}|$ can be bounded.

We need the following lemma for the theoretical analysis of each arm selection strategy. 
\begin{lemma} \label{lemma:ts-ci}Suppose that $\cref{eq:app-m-eq}$ is satisfied. 
Let $\delta\in[0,1]$ and $X_t$ be a random variable of the form:
    \begin{align}
        X_t = r_t^a I_{E_{t}} - c_t\no
    \end{align} where $E_t$ is a set of events such that $\mathbb{P}(E_t)\geq 1-\delta$  and $c_t$ is of the form:
    \begin{align}
    c_t \leq C\left(c'_t\|\Delta\rep_{t,\stt,\ndt}\|_{V_{t-1}^{-1}} + \mathcal{R}(m) +\frac{1}{t^2}\right)\no
    \end{align} for some $c_t'>0$ such that $c_t'=\bigol(\sqrt{d})$. If the network width $\wid$ satisfies \cref{eq:app-m-eq}, and  
    $Y_t=\sum_{i=1}^{t}X_i$ is a super-martingale with respect to some filtration set $\{\mathcal{F}_t\}$, the following holds:
    \begin{align}
    R_T^a &= \bigol\left(d\epsilon\sqrt{T}+d^3Tm^{-1/6}\right) \quad \text{for $\epsilon=\bigol(1/\sqrt{d})$ (variance-aware) in \cref{eq:app-wi}},\no\\
    R_T^a &= \bigol\left(d\epsilon\sqrt{T} + d\sqrt{\sum_{i=1}^{T}{\sigma}_i^2}+ d^3Tm^{-1/6}\right)\quad\text{for $\epsilon=1$ (variance-agnostic) in \cref{eq:app-wi}}\no.
    \end{align}
\end{lemma}
\begin{proof}
We apply the Azuma-Hoeffding inequality~\citep{zhang2023mathematical} to $Y_t$. With probability at least $1-\delta,$ the following holds:
\begin{align}
    Y_t-Y_0&=\sum_{t=1}^{T} r_t^aI_{E_{t}}-c_t\leq \sqrt{\log(1/\delta)\sum_{i=1}^{t}|Y_i-Y_{i-1}|}.\no
\end{align} Notice that $|Y_i-Y_{i-1}| = |r_t^a|+|c_t|=\bigol(1)$ due to the definition of $\mathcal{R}(m)$ to $c_t$ with~\cref{eq:app-m-eq}. Therefore, there exists a constant $C$ such that 
$|Y_i-Y_{i-1}| \leq C$ and 
\begin{align}
    R_T^a=\sum_{t=1}^{T}r_t\leq \sum_{t=1}^{T}c_t+\sqrt{\log(1/\delta)TC}.    \no
\end{align} with probability at least $1-\delta.$ By the definition of $c_t$, we can compute the following:
\begin{align}
    R_T^a&=\sum_{t=1}^{T}r_t\leq C_1\left(\sum_{t=1}^{T}c'_t\|\Delta\rep_{t,\stt,\ndt}\|_{V_{t-1}^{-1}} + T\mathcal{R}(m) +\sum_{t=1}^{T}\frac{1}{t^2}\right)\no\\
    &\leq C_2\left(\sum_{t=1}^{T}c'_t\|\Delta\rep_{t,\stt,\ndt}\|_{V_{t-1}^{-1}} + T\mathcal{R}(m) +\pi^2/6\right)\no\\
    &\leq C_2\lt(\max_{t\in[0,T]}c'_t\lt(\sum_{t=1}^{T} \|\Delta\rep_{t,\stt,\ndt}\|_{V_{t-1}^{-1}}^2\rt)+ T\mathcal{R}(m) + \pi^2/6
    \rt)\no
\end{align} for some constants $C_1, C_2>0,$
where the second inequality results from $\sum_{t=1}^{\infty}\frac{1}{t^2}=\pi/6$ and the third inequality results from the Cauchy–Schwarz inequality. Notice that $c_t'=\bigol(\sqrt{d})$. In a similar manner to the case of the UCB framework (\cref{sec:UCB}), applying \cref{lemma:epl}, we can obtain the following inequality in both the variance-aware 
\begin{align}
    R_T^a &= \bigol\lt(d\epsilon\sqrt{T}+d\sqrt{\sum_{i=1}^{T}{\sigma}_i^2}+m^{-1/6}d^3T\rt),\label{eq:ts1}
\end{align} and variance-agnostic cases:
\begin{align}
    R_T^a &= \bigol\lt(d\epsilon\sqrt{T}+m^{-1/6}d^3T\rt)\label{eq:ts2},
\end{align} with probability at least $1-\delta$. This completes the proof.
\end{proof}

Accordingly, for each TS-based arm selection, we will show that \cref{lemma:ts-ci} is satisfied. Then  \cref{theorem:app-main_ts} and \cref{theorem:app-main_ts2} are proved.

{\color{black}\textbf{Proof structure for each TS strategy.} For each arm selection method below, the proof follows the same template: (1) define the good events $E_t$ and $F_t$; (2) define the saturation set $S_t$; (3) show the selected arms avoid $S_t$ with positive probability; (4) bound the conditional regret $\mathbb{E}[r_t^a|\mathcal{F}_{t-1}]$; (5) construct the super-martingale $Y_t$ and apply Lemma~\ref{lemma:ts-ci}.}

\subsection{Asymmetry Arm Selection Strategy:}\label{sec:ts-usym}
 We consider the following strategy:
\begin{align}
    \stt &= \argmax_{\arm\in[K]} \th_{t-1}^\intercal\rep(x_{t,\arm};\ww_{t-1}),\no\\
    \ndt&=\argmax_{\arm\in [K]} \Delta\widetilde{\util}_{t,\arm, \stt}\no
\end{align}
where $\Delta\widetilde{\util}_{t,\arm,\stt}\sim \gauss(\th_{t-1}^\intercal\Delta\rep_{t,\arm,\stt}, \cwct^2\|\Delta\rep_{t,\arm,\stt}\|^2_{V_{t-1}^{-1}})$ with the Gaussian distribution $\gauss$.
We define $E_{t}$, and $F_{t}$ as
\begin{align}
    E_{t}&=\{|\Delta \util_{t,\arm,\stt}-\th_{t-1}^\intercal\Delta\rep_{t,\arm,\stt}|\leq \cwct\|\Delta\rep_{t,\arm,\stt}\|_{V_{t-1}^{-1}} + \mathcal{R}(m)\}\label{eq:E_t_1}\\
    F_{t}&=\{|\Delta\widetilde{\util}_{t,\arm,\stt}-\th_{t-1}^\intercal\Delta\rep_{t,\arm,\stt}|\leq \cwct\sqrt{2\log(Kt^2)}\|\Delta\rep_{t,\arm,\stt}\|_{V_{t-1}^{-1}}\}.\label{eq:F_t_1}
\end{align} Moreover, let $b_t=\cwct\left(1+\sqrt{2\log\left(Kt^2\right)}\right)$. We need to show that $E_t$ and $F_t$ in~\cref{eq:E_t_1,eq:F_t_1} are good events.
\begin{lemma} The following inequalities are satisfied:
    \begin{align}
    \mathbb{P}\left(E_t\right)&\geq 1-\delta,\label{eq:ci1}\\
    \mathbb{P}\left(F_t\right)&\geq 1-\frac{1}{t^2}.\label{eq:normal1}
\end{align}
\end{lemma}
\begin{proof}
\cref{eq:ci1} is obtained by \cref{eq:tsf}. Besides, by the property of the Gaussian distribution, \cref{eq:normal1} is obtained by Hoeffding's inequality~\citep{lattimore2020bandit} with choosing the confidence radius as $\cwct\sqrt{2\log(Kt^2)}\|\Delta\rep_{t,\arm,\stt}\|_{V_{t-1}^{-1}}$ with the sample size 1. 
\end{proof}

\begin{lemma}\label{lemma:normal1}
For any filtration $\mathcal{F}_{t-1}$ and the condition $E_{t}$, the following inequality is satisfied:
    \begin{align}
        \mathbb{P}(\Delta\widetilde{\util}_{t,\arm,\stt}+\mathcal{R}(m)>\Delta \util_{t,\arm,\stt}|\mathcal{F}_{t-1})\geq \frac{1}{4e\sqrt{\pi}}.\no
    \end{align} 
\end{lemma}
\begin{proof} It is straightforward due to the following inequalities:
    \begin{align}
        &\mathbb{P}(\Delta\widetilde{\util}_{t,\arm,\stt}+\mathcal{R}(m)>\Delta \util_{t,\arm,\stt}|\mathcal{F}_{t-1})\no\\&=
        \mathbb{P}\left(\frac{\Delta\widetilde{\util}_{t,\arm,\stt}+\mathcal{R}(m)-\th_{t-1}^\intercal\Delta\rep_{t,\arm,\stt}}{\cwct\|\Delta\rep_{t,\arm,\stt}\|_{V_{t-1}^{-1}}}>\frac{\Delta \util_{t,\arm,\stt}-\th_{t-1}^\intercal\Delta\rep_{t,\arm,\stt}}{\cwct\|\Delta\rep_{t,\arm,\stt}\|_{V_{t-1}^{-1}}}\Bigg|\mathcal{F}_{t-1}\right)\no
        \\&\geq
        \mathbb{P}\left(\frac{\Delta\widetilde{\util}_{t,\arm,\stt}-\th_{t-1}^\intercal\Delta\rep_{t,\arm,\stt}}{\cwct\|\Delta\rep_{t,\arm,\stt}\|_{V_{t-1}^{-1}}}>\frac{|\Delta \util_{t,\arm,\stt}-\th_{t-1}^\intercal\Delta\rep_{t,\arm,\stt}|-\mathcal{R}(m)}{\cwct\|\Delta\rep_{t,\arm,\stt}\|_{V_{t-1}^{-1}}}\Bigg|\mathcal{F}_{t-1}\right)\no
        \\&\geq
        \mathbb{P}\left(\frac{\Delta\widetilde{\util}_{t,\arm,\stt}-\th_{t-1}^\intercal\Delta\rep_{t,\arm,\stt}}{\cwct\|\Delta\rep_{t,\arm,\stt}\|_{V_{t-1}^{-1}}}>1\Bigg|\mathcal{F}_{t-1}\right)\geq\frac{1}{4e\sqrt{\pi}}.\no
    \end{align}
    The last inequality is obtained by the definition of $E_{t}.$ This completes the proof.
\end{proof}

{\color{black}\textbf{Saturation set.} The saturation set $S_t$ contains arms that are ``clearly suboptimal''---their utility gap to the optimal arm exceeds the exploration bonus. Arms in $S_t$ are unlikely to be selected by TS because the Gaussian perturbation is insufficient to make them appear optimal. The key argument is that the selected arm $\ndt$ falls outside $S_t$ with positive probability, so the instantaneous regret can be bounded by the exploration bonus of the non-saturated arm.}

 We define a set of saturated points.
\begin{align}
    S_t = \lt\{\arm: \Delta \util_{t,\opt,\arm}> b_t \|\Delta\rep_{t,\arm,\stt}\|_{V_{t-1}^{-1}}+2\mathcal{R}(m)\rt\}.\label{eq:satur_1}
\end{align}

Then we prove the following lemma.

\begin{lemma}
    Under any filtration $\mathcal{F}_{t-1}$ and $E_{t},$
    \begin{align}
        \mathbb{P}\lt(\ndt\in[K]\setminus S_t|\mathcal{F}_{t-1}\rt)\geq p-1/t^2.\no
    \end{align} where $[K]\setminus S_t$ is a set of elements belonging to $[K]$ but not in $S_t$.
\end{lemma} 
\begin{proof}
We first show that
\begin{align}
    \mathbb{P}(\ndt\in[K]\setminus S_t|\mathcal{F}_{t-1})\geq \mathbb{P}(\Delta\widetilde{\util}_{t,\opt,\stt} >\Delta\widetilde{\util}_{t,\arm,\stt}, \:\forall \arm\in S_t|\mathcal{F}_{t-1}).\no
\end{align}
    Let us suppose  that $\widetilde{\util}_{t,\opt}>\widetilde{\util}_{t,\arm}$ for all $\arm\in S_t$. Note that the definition of $\ndt=\arg\max_{\arm\in[K]}\Delta \widetilde{\util}_{t,\arm,\stt}$. Then  $\ndt\notin S_t$ due to the property of $S_t$ in \cref{eq:satur_1}. Accordingly, under $F_{t}$, for all $\arm\in S_t,$ we have
    \begin{align}
        \Delta\widetilde{\util}_{t,\arm,\stt}&\leq \Delta \util_{t,\arm,\stt} + b_t\|\Delta\rep_{t,\arm,\stt}\|_{V_{t-1}^{-1}}+2\mathcal{R}(m)-\mathcal{R}(m)\no\\
        &\leq \Delta \util_{t,\arm,\stt}  +  \Delta \util_{t,\opt,k} - \mathcal{R}(m) \textrm{ by \cref{eq:satur_1}}\no\\
        &=\Delta \util_{t,\opt,\stt} -\mathcal{R}(m).\no
    \end{align} 
    Therefore,
    \begin{align}
        \mathbb{P}(\ndt\in[K]\setminus S_t|\mathcal{F}_{t-1})&\geq \mathbb{P}(\Delta\widetilde{\util}_{t,\opt,\stt}>\Delta\widetilde{\util}_{t,\arm,\stt} \forall \arm\in S_t|\mathcal{F}_{t-1})\no\\
        &\geq \mathbb{P}(\Delta\widetilde{\util}_{t,\opt,\stt}>\Delta \util_{t,\opt,\stt} -\mathcal{R}(m) |\mathcal{F}_{t-1})-\mathbb{P}(F_{t}^c|\mathcal{F}_{t-1})\no
        \\
        &\geq \frac{1}{4e\sqrt{\pi}} - \frac{1}{t^2},        \no
    \end{align} where $F_{t}^c$ is the complement of $F_{t}$. This completes the proof.
\end{proof}
\begin{lemma}\label{lemma:d4-usym}
  Under any filtration $\mathcal{F}_{t-1}$ and $E_{t}$,
  \begin{align}
      \mathbb{E}[2r^a_t | \mathcal{F}_{t-1}]=\mathbb{E}\left[\|\Delta\rep_{t,\ndt,\stt}\|_{V_{t-1}^{-1}}|\mathcal{F}_{t-1}\right](4b_t/(p-1/t^2)+3b_t) + 9\mathcal{R}(m) + \frac{4}{t^2}.\no
  \end{align} 
\end{lemma}
\begin{proof}
    Let us define $\overline{\arm}_t=\arg\min_{\arm\in[K]\setminus S_t}\|\Delta\rep_{t,\arm,\stt}\|_{V_{t-1}^{-1}}.$
    Notice that a regret can be decomposed as follows.
    \begin{align}
        r_t^a = 2\Delta \util_{t,\opt,\ndt} + \Delta \util_{t,\ndt,\stt}.\no
    \end{align}As a result, under $E_{t}\bigcap F_{t}$,
    \begin{align}
        \Delta \util_{t,\opt,\ndt} &= \Delta \util_{t,\opt,\overline{\arm}_t}
        +\Delta \util_{t,\overline{\arm}_t,\stt}+\Delta \util_{t,\stt,\ndt}\label{eq:DTKK}\\&
        =\Delta \util_{t,\opt,\overline{\arm}_t}
        +\Delta \util_{t,\overline{\arm}_t,\stt}-\Delta \util_{t,\ndt,\stt}\no
        \\&\leq\Delta \util_{t,\opt,\overline{\arm}_t} + \Delta\widetilde{\util}_{t,\overline{\arm}_t,\stt}+b_t\|\Delta\rep_{t,\overline{\arm}_t,\stt}\|_{V_{t-1}^{-1}} + \mathcal{R}(m)\no\\
        \s -\Delta\widetilde{\util}_{t,\ndt,\stt}+b_t\|\Delta\rep_{t,\ndt,\stt}\|_{V_{t-1}^{-1}} + \mathcal{R}(m)
        \no\\&\leq \Delta \util_{t,\opt,\overline{\arm}_t} + b_t\|\Delta\rep_{t,\overline{\arm}_t,\stt}\|_{V_{t-1}^{-1}} + b_t\|\Delta\rep_{t,\ndt,\stt}\|_{V_{t-1}^{-1}}+2\mathcal{R}(m).\no
    \end{align} The first inequality is obtained by the definition of $E_{t}$ and $F_{t}$ in \cref{eq:E_t_1} and \cref{eq:F_t_1}, respectively. The asymmetric arm selection to choose $(\stt,\ndt)$ is used to get the second inequality. 
    Since $\overline{\arm}_t$ is not saturated, 
    \begin{align}
 \Delta \util_{t,\opt,\overline{\arm}_t}\leq b_t\|\Delta\rep_{t,\overline{\arm}_t,\stt}\|_{V_{t-1}^{-1}} + \mathcal{R}(m).\label{eq:Delta_tkk}
    \end{align}
    Thus, combining \cref{eq:DTKK} and \cref{eq:Delta_tkk}, we can estimate the following inequality:
    \begin{align}
       \Delta \util_{t,\opt,\ndt} &\leq 2b_t\|\Delta\rep_{t,\overline{\arm}_t,\stt}\|_{V_{t-1}^{-1}} + b_t\|\Delta\rep_{t,\ndt,\stt}\|_{V_{t-1}^{-1}} + 3\mathcal{R}(m). \no
    \end{align}Furthermore, by the definition of $E_{t},$
    \begin{align}
        \Delta \util_{t,\ndt,\stt}&\leq \th_{t-1}^\intercal\Delta\rep_{t,\ndt,\stt}+\cwct\|\Delta\rep_{t,\stt,\ndt}\|_{V_{t-1}^{-1}}+\mathcal{R}(m)\no
        \\&\leq \cwct\|\Delta\rep_{t,\stt,\ndt}\|_{V_{t-1}^{-1}}+\mathcal{R}(m).\no
    \end{align} The second inequality results from $$\th_{t-1}^\intercal\Delta\rep_{t,\ndt,\stt}\leq 0,$$ by the asymmetry arm strategy, i.e., $\stt=\argmax_{\arm\in[K]} \th_{t-1}^\intercal\rep(x_{t,\arm};\ww_{t-1}).$   
    Thus,
    \begin{align}
        r_t^a&\leq 2\left(2b_t\|\Delta\rep_{t,\overline{\arm}_t,\stt}\|_{V_{t-1}^{-1}} + b_t\|\Delta\rep_{t,\ndt,\stt}\|_{V_{t-1}^{-1}} + 4\mathcal{R}(m)\right) \no\\\s + \cwct\|\Delta\rep_{t,\stt,\ndt}\|_{V_{t-1}^{-1}}+\mathcal{R}(m).\no
    \end{align}
    Notice that 
    \begin{align}
        &\mathbb{E}[\|\Delta\rep_{t,\stt,\ndt}\|_{V_{t-1}^{-1}} |\mathcal{F}_{t-1}]\no\\&\geq \mathbb{E}[\|\Delta\rep_{t,\stt,\ndt}\|_{V_{t-1}^{-1}}|\mathcal{F}_{t-1}, \ndt\in [K]\setminus S_t]\times\mathbb{P}(\ndt\in[K]\setminus S_t|\mathcal{F}_{t-1})\no\\
        &\geq \|\Delta\rep_{t,\stt,\overline{\arm}_t}\|_{V_{t-1}^{-1}}(p-1/t^2).\no
    \end{align}
    Thus, we can estimate
    \begin{align}
        \mathbb{E}[2r_t^a|\mathcal{F}_{t-1}]&\leq \mathbb{E}\left[\|\Delta\rep_{t,\ndt,\stt}\|_{V_{t-1}^{-1}}|\mathcal{F}_{t-1}\right](4b_t/(p-1/t^2)+3b_t) + 9\mathcal{R}(m) + \frac{4}{t^2}.\label{eq:ts-asym-1}
    \end{align} due to  $b_t\geq \cwct.$ This completes the proof.
\end{proof}

Finally,  we can show the following:
\begin{lemma}\label{lemma:sm1}
    Let $Y_t=\sum_{i=1}^{t}X_i$ where $X_t=r_t^aI_{E_t}-c_t$ with 
    \begin{align}
        c_t = \frac{7b_t}{2p}\|\Delta\rep_{t,\ndt,\stt}\|_{V_{t-1}^{-1}}  + \frac{9}{2}\mathcal{R}(m) + \frac{2}{t^2}.\label{eq:app-ct_1}
    \end{align}
    Then $Y_t$ is a super-martingale with respect to the filtration $\mathcal{F}_{t}.$
\end{lemma}
\begin{proof} It is obvious that 
    \begin{align}
        \mathbb{E}[Y_t-Y_{t-1}|\mathcal{F}_{t-1}] = \mathbb{E}[X_t|\mathcal{F}_{t-1}] \leq 0, \label{eq:mar1} 
    \end{align} since $\mathbb{E}[r(t)|\mathcal{F}_{t-1}]\leq c_t$ by \cref{eq:ts-asym-1}. Due to \cref{eq:mar1}, $Y_t$ is a super-martingale under the filtration $\mathcal{F}_t$, thus this completes the proof.
\end{proof}
Finally, by applying \cref{lemma:ts-ci} with the help of \cref{lemma:sm1}, we complete the proofs of \cref{theorem:app-main_ts} and \cref{theorem:app-main_ts2} under the asymmetric arm selection strategy.

{\color{black}\textbf{TS-OSYM: Simultaneous arm selection.} Unlike the asymmetric case where $\stt$ is chosen greedily and only $\ndt$ is sampled, the optimistic symmetric strategy draws \emph{both} arms simultaneously from a joint Gaussian posterior. This introduces a key technical challenge: we must show that the pair $(\stt,\ndt)$ avoids the saturation set $S_t\subset [K]^2$ with positive probability, which requires bounding a product of Gaussian tail probabilities. The proof exploits the independence of the two Gaussian draws conditioned on $\mathcal{F}_{t-1}$.}

\subsection{Optimistic Symmetric Arm Selection Strategy}
Now, we consider the following strategy:
\begin{align}
    \stt, \ndt=\argmax_{i,j\in[K]}\left(\widetilde{\util}_{t,i,j}\right).\label{eq:pt_qt}
\end{align}
Here,
\begin{align}
    \widetilde{\util}_{t,i,j}&\sim \gauss\lt(\th_{t-1}^\intercal\rep_{t,i,j}, \cwct^2\|\Delta\rep_{t,i,j}\|^2_{V_{t-1}^{-1}}\rt),\no
\end{align} where $\rep_{t,i,j}=\rep(x_{t,i};\ww_{t-1})+\rep(x_{t,j};\ww_{t-1})$ as in \cref{eq:not}. Let us define 
    \begin{align}\widetilde{\util}_{t,i,j,1}&=\widetilde{\util}_{t,i,j} - \th_{t-1}^\intercal\rep(x_{t,j},\ww_{t-1}),\no\\
    \widetilde{\util}_{t,i,j,2}&=\widetilde{\util}_{t,i,j} - \th_{t-1}^\intercal\rep(x_{t,i},\ww_{t-1}).\no
\end{align}
We define the following sets:
\begin{align}
    S_t &= \lt\{(i,j): \Delta \util_{t,\opt,i}+\Delta \util_{t,\opt,j}> b_t \|\Delta\rep_{t,i,j}\|_{V_{t-1}^{-1}}+\mathcal{R}(m)\rt\} ,\label{eq:S_t_2}\\
    E_{t}&=\lt\{(k_1,k_2)\in[K]^2: |\Delta \util_{t,k_1,k_2}-\th_{t-1}^\intercal\Delta\rep_{t,k_1,k_2}|\leq \cwct\|\Delta\rep_{t,k_1,k_2}\|_{V_{t-1}^{-1}} + \mathcal{R}(m)\rt\},\label{eq:E_t_2}
    \\
    F_{t}&=\lt\{x: x\sim \gauss\lt(\th_{t-1}^\intercal\rep_{t,i,j}, \cwct^2\|\Delta\rep_{t,i,j}\|_{V_{t-1}^{-1}}\rt), \rt.\label{eq:F_t_2} \\ &\qquad\lt.\text{ such that } |x-\mathbb{E}[x]|\leq \cwct\sqrt{2\log(Kt^2)}\|\Delta\rep_{t,i,j}\|_{V_{t-1}^{-1}}) \text{ for all $i,j\in[K]$ }\rt\}, \no
\end{align} with $b_t=\cwct\sqrt{2\log(Kt^2)}$ and $\mathcal{R}(m)$ defined in \cref{eq:tsf}.

    Notice that in the saturation set $S_t$ in \cref{eq:S_t_2}, 
    $(\opt,\opt)\notin S_t$ and $(\arm,\arm)\in S_t$ such that  $\arm\neq \opt$ at round $t$ for sufficiently large $m$ since $\mathcal{R}(m)$ approaches to 0 as $m$ increases. We need to show that $E_t$ and $F_t$ in ~\cref{eq:E_t_2} and~\cref{eq:F_t_2} are good events. 
\begin{lemma} The following inequalities are satisfied:
    \begin{align}
    \mathbb{P}(E_{t})&\geq 1-\delta,\label{eq:ci2}\\
    \mathbb{P}(F_{t})&\geq 1-\frac{1}{t^2}.\label{eq:normal2}
\end{align}
\end{lemma}
\begin{proof}
\cref{eq:ci2} is obtained by \cref{eq:tsf}. In addition,
    since
\begin{align}
    \mathbb{P}(\widehat{\mu}-\mu\geq \epsilon)\leq \frac{1}{2}\exp{\left(-\frac{n\epsilon^2}{2\sigma^2}\right)},\no
\end{align} by setting $\epsilon = \cwct\sqrt{2\log(Kt^2)}\|\Delta\rep_{t,\arm,\stt}\|_{V_{t-1}^{-1}},$ we can show that \cref{eq:normal2}.
\end{proof}

Then we prove the following lemma.

\begin{lemma}
    For any filtration $\mathcal{F}_{t-1},$ under the condition $F_t$, there exists a strictly positive $p$ such that
    \begin{align}
        \mathbb{P}((\stt,\ndt)\in[K]^2\setminus S_t |\mathcal{F}_{t-1})\geq p-1/t^2.\no
    \end{align}
\end{lemma}
\begin{proof}
Due to the definition of $\widetilde{\util}_{i,\stt,\ndt}$, the following is satisfied:
\begin{align}
    \mathbb{P}((\stt,\ndt)\in \left([K]\setminus S_t\right)^2|\mathcal{F}_{t-1})\geq \mathbb{P}\left(\widetilde{\util}_{t,\opt,i}+\widetilde{\util}_{t,\opt,j}\geq 2\widetilde{\util}_{t,i,j}, \forall (i,j)\in S_t|\mathcal{F}_{t-1}\right).\label{eq:pp2}
\end{align}
We can rewrite~\cref{eq:pp2} as
\begin{align}
&\mathbb{P}\left(\widetilde{\util}_{t,\opt,i}+\widetilde{\util}_{t,\opt,j}\geq 2\widetilde{\util}_{t,i,j}, \forall (i,j)\in S_t|\mathcal{F}_{t-1}\right)\no\\
&=
\mathbb{P}\left(\widetilde{\util}_{t,\opt,i}+\widetilde{\util}_{t,\opt,j}-2\th_{t-1}^\intercal\Delta\rep_{t,i,j}\geq 2\widetilde{\util}_{t,i,j}-2\th_{t-1}^\intercal\Delta\rep_{t,i,j}, \forall (i,j)\in S_t|\mathcal{F}_{t-1}\right).  \no  
\end{align} and
\begin{align}
    \widetilde{\util}_{t,i,j}-\th_{t-1}^\intercal\rep_{t,i,j}&\leq  b_t\|\Delta\rep_{t,i,j}\|_{V_{t-1}^{-1}}\leq \frac{\Delta \util_{t,\opt,i}+\Delta \util_{t,\opt,j}}{2} \text{ by~\cref{eq:F_t_2} and~\cref{eq:S_t_2}},\no\\
    \widetilde{\util}_{t,\opt,i}+\widetilde{\util}_{t,\opt,j}-2\th_{t-1}^\intercal\rep_{t,i,j} &\sim \gauss(\th_{t-1}^\intercal\left(\Delta\rep_{t,\opt,i}+\Delta\rep_{t,\opt,j}\right),\|\Delta\rep_{t,\opt,i}\|^2_{V_{t-1}^{-1}} +\|\Delta\rep_{t,\opt,j}\|^2_{V_{t-1}^{-1}}), \no
\end{align}  under $F_{t}$. Thus, in a similar manner to \cref{lemma:normal1}, we can show that there is a strictly positive probability $p>0$, e.g. $\lt(\frac{1}{4e\sqrt{\pi}}\rt)^2$,  such that 
\begin{align}
    &\mathbb{P}\left(\widetilde{\util}_{t,\opt,i}+\widetilde{\util}_{t,\opt,j}\geq 2\widetilde{\util}_{t,i,j}, \forall (i,j)\in S_t|\mathcal{F}_{t-1}\right)\no\\&\geq
    \mathbb{P}^2\left(x\geq \mu:x\sim \gauss(\mu, \cwct^2(\|\Delta\rep_{t,\opt,i}\|_{V_{t-1}^{-1}}^2+\|\Delta\rep_{t,\opt,j}\|_{V_{t-1}^{-1}}^2))\right)\geq p >0,\no
\end{align} where $
\mu:=\th_{t-1}^\intercal(\Delta\rep_{t,\opt,i}+\Delta\rep_{t,\opt,j}).
$
As a result,
\begin{align}
    \mathbb{P}((\stt,\ndt)\in [K]\times[K]\setminus S_t|\mathcal{F}_{t-1})&\geq  \mathbb{P}\left(\widetilde{\util}_{t,\opt,i}+\widetilde{\util}_{t,\opt,j}\geq 2\widetilde{\util}_{t,i,j}, \forall (i,j)\in S_t|\mathcal{F}_{t-1}\right) \no\\\s - \mathbb{P}(\overline{F_{t}}|\mathcal{F}_{t-1})\no\\&\geq p-\frac{1}{t^2}.\no
\end{align} This completes the proof.
\end{proof}
\begin{lemma}
  For any filtration $\mathcal{F}_{t-1}$,
  \begin{align}
      \mathbb{E}[2r_t | \mathcal{F}_{t-1}]=b_t(6p+3)\mathbb{E}\lt[\|\Delta\rep_{\stt,\ndt}\|_{V_{t-1}^{-1}}\mid \mathcal{F}_{t-1}\rt] + 4\mathcal{R}(m)+\frac{4}{t^2}.\no
  \end{align} 
\end{lemma}
\begin{proof}
    
    Notice that
    \begin{align}
        2r_t = 2h_{t,\opt}-h_{t,\stt}-h_{t,\ndt}. \label{eq:sreg}
    \end{align} 
    Let $(\overline{\arm}_1,\overline{\arm}_2)\in [K]^2\setminus (S_t)$ such that 
\begin{align}
k_1 = \argmax_{\arm\in[K]}\|\Delta\rep_{t,\arm,\stt}\|_{V_{t-1}^{-1}}&\leq \|\Delta\rep_{t,\stt,\ndt}\|_{V_{t-1}^{-1}},\no\\
k_2 = \argmax_{\arm\in[K]}\|\Delta\rep_{t,\arm,\ndt}\|_{V_{t-1}^{-1}}&\leq \|\Delta\rep_{t,\stt,\ndt}\|_{V_{t-1}^{-1}}.\no
\end{align}
Thus,
\begin{align}
    &\mathbb{E}[\|\Delta\rep_{t,\stt,\ndt}\|_{V_{t-1}^{-1}}|\mathcal{F}_{t-1}]\nonumber\\&\geq \mathbb{E}\left[\frac{\|\Delta\rep_{\stt,\ndt}\|_{V_{t-1}^{-1}}+\|\Delta\rep_{\stt,\ndt}\|_{V_{t-1}^{-1}}}{2}|\mathcal{F}_{t-1},(\stt,\ndt
    )\in[K]\setminus S_t\right]\no\\\s\times\mathbb{P}\left[(\stt,\ndt
    )\in[K]\setminus S_t|\mathcal{F}_{t-1}\right]\nonumber\\
    &\geq (\|\Delta\rep_{t,\overline{\arm}_1,\stt}\|_{V_{t-1}^{-1}}+\|\Delta\rep_{t,\overline{\arm}_2,\ndt}\|_{V_{t-1}^{-1}})\left(p-\frac{1}{t^2}\right)\label{eq:prob2}
\end{align}

Note that a tuple $(k_1,k_2)$ exists in $[K]^2\setminus S_t$ since $(\opt,\opt)\in [K]^2\setminus S_t\neq \emptyset.$
    We can decompose \cref{eq:sreg} as
    \begin{align}
        2r_t &= h_{t,\opt}-h_{t,\overline{\arm}_2}+h_{t,\overline{\arm}_2}-h_{t,\stt}\no\\
        \s+h_{t,\opt}-h_{t,\overline{\arm}_1}+h_{t,\overline{\arm}_1}-h_{t,\ndt}.\no
    \end{align}
    The fact $(\overline{\arm}_1,\overline{\arm}_2)\in [K]^2\setminus S_t$ results in
    \begin{align}
        |h_{t,\opt}-h_{t,\overline{\arm}_1}+h_{t,\opt}-h_{t,\overline{\arm}_2}| \leq 2b_t\|\Delta\rep_{\overline{\arm}_1,\overline{\arm}_2}\|_{V_{t-1}^{-1}} + 2\mathcal{R}(m).\no
    \end{align}
    Furthermore,
    \begin{align}
        &\Delta \util_{t,\overline{\arm}_{1},\stt} + \Delta \util_{t,\overline{\arm}_{2},\ndt} \no\\&\leq \th_{t-1}^\intercal\Delta\rep_{t,\overline{\arm}_1,\stt}+\th_{t-1}^\intercal\Delta\rep_{t,\overline{\arm}_2,\ndt}+2b_t \|\Delta\rep_{\overline{\arm}_1,\stt}\|_{V_{t-1}^{-1}}+\|\Delta\rep_{\overline{\arm}_2,\ndt}\|_{V_{t-1}^{-1}}\no\\
        &\leq \widetilde{\util}_{t,\overline{\arm}_1, \overline{\arm}_2} - \widetilde{\util}_{t,\stt,\ndt} \no\\\s + 2b_t\left(\|\Delta\rep_{\overline{\arm}_1,\stt}\|_{V_{t-1}^{-1}}+\|\Delta\rep_{\overline{\arm}_2,\ndt}\|_{V_{t-1}^{-1}}+
        \|\Delta\rep_{\stt,\ndt}\|_{V_{t-1}^{-1}}+\|\Delta\rep_{\overline{\arm}_1,\overline{\arm}_2}\|_{V_{t-1}^{-1}}\right) + 2\mathcal{R}(m).\no
    \end{align} Then we can derive the following by the triangular inequality:
    $$
    \|\Delta\rep_{\overline{\arm}_1,\overline{\arm}_2}\|_{V_{t-1}^{-1}}\leq \|\Delta\rep_{\overline{\arm}_1,\stt}\|_{V_{t-1}^{-1}} + \|\Delta\rep_{\stt,\ndt}\|_{V_{t-1}^{-1}}+\|\Delta\rep_{\ndt,\overline{\arm}_2}\|_{V_{t-1}^{-1}}.
    $$
    By the definition of $\stt, \ndt$, 
    \begin{align}
        \widetilde{\util}_{t,\overline{\arm}_1, \overline{\arm}_2} - \widetilde{\util}_{t,\stt,\ndt} < 0.\no
    \end{align} 
    Thus, we can obtain
    \begin{align}
        2r_t \leq  b_t\left(6\left(\|\Delta\rep_{\overline{\arm}_1,\stt}\|_{V_{t-1}^{-1}}+\|\Delta\rep_{\overline{\arm}_2,\ndt}\|_{V_{t-1}^{-1}}\right)+3\|\Delta\rep_{\stt,\ndt}\|_{V_{t-1}^{-1}}\right)+4\mathcal{R}(m).\no
    \end{align}
    Using \cref{eq:prob2}, we can conclude that under \cref{assumption:utilbound}, 
    \begin{align}
        2\mathbb{E}[r_t|\mathcal{F}_{t-1}]&\leq b_t \left(6\left(p-\frac{1}{t^2}\right)^{-1}+3\right)\mathbb{E}\lt[\|\Delta\rep_{\stt,\ndt}\|_{V_{t-1}^{-1}}\mid \mathcal{F}_{t-1}\rt] + 4\mathcal{R}(m)+\frac{4}{t^4}\label{eq:app-final-2}\\
        &\leq b_t(6p+3)\mathbb{E}\lt[\|\Delta\rep_{\stt,\ndt}\|_{V_{t-1}^{-1}}\mid \mathcal{F}_{t-1}\rt] + 4\mathcal{R}(m)+\frac{4}{t^2}.\no
    \end{align}
    This completes the proof.
\end{proof}

After the above, as the same manner to the case of \cref{lemma:sm1} using \cref{eq:app-final-2}, we can obtain the following lemma.

\begin{lemma}\label{lemma:app-ts-final2}
    Let $Y_t=\sum_{i=1}^{t}X_i$ where $X_t=r_tI_{E}-c_t$ with
$$
c_t = \frac{b_t(6p+3)}{2}\|\Delta\rep_{\stt,\ndt}\|_{V_{t-1}^{-1}}+ 2\mathcal{R}(m)+\frac{2}{t^2}
$$
    Then $(Y_t)$ is a super-martingale with respect to the filtration $\mathcal{F}_{t}.$
\end{lemma}

Combining~\cref{lemma:ts-ci} and~\cref{lemma:app-ts-final2}, we complete the proofs of \cref{theorem:app-main_ts} and \cref{theorem:app-main_ts2} under the optimistic symmetric arm selection strategy.

{\color{black}\textbf{TS-CSYM: Confidence set filtering + stochastic selection.} This strategy first constructs a confidence set $\mathcal{C}_t$ of ``plausibly optimal'' arms, then selects the pair $(\stt,\ndt)$ from $\mathcal{C}_t\times\mathcal{C}_t$ by sampling exploration bonuses from Gaussian distributions. The proof follows the same template as TS-OSYM but restricts attention to arms within $\mathcal{C}_t$, leveraging the property that the optimal arm $\opt\in\mathcal{C}_t$ with high probability (when $\cwct$ is chosen appropriately).}

\subsection{Candidate-based Symmetric Arm Selection Strategy}
Let us define
    \begin{align}
    \mathcal{C}_t = \left\{i:
    \th_{t-1}^\intercal\rep(x_{t,i};\ww_{t-1})- \th_{t-1}^\intercal\rep(x_{t,j};\ww_{t-1}) +\cwct\|
    \Delta\rep_{t,i,j}\|_{V_{t-1}^{-1}} \geq 0, \forall j\in [K]\right\}.\nonumber
\end{align} We choose $(\stt,\ndt)$ as follows:
\begin{align}
    (\stt, \ndt)&=\argmax_{(i,j)\in \mathcal{C}_t\times \mathcal{C}_t}\widehat{\sigma}_{i,j},\no\\
    \widehat{\sigma}_{i,j}&\sim\gauss\left(\|\Delta\rep_{t,i,j}\|^2_{V_{t-1}^{-1}}, \frac{\|\Delta\rep_{t,i,j}\|^4_{V_{t-1}^{-1}}}{4\log(Kt^2)}\right).\no
\end{align}

We define
\begin{align}
    \widetilde{\util}_{t,i,j} &= \th_{t-1}\rep(x_{t,i};\ww_{t-1})+\th_{t-1}\rep(x_{t,j};\ww_{t-1})+\cwct\widehat{\sigma}_{i,j},\no\\
    \widetilde{\util}_{t,i,j,1} &= \widetilde{\util}_{t,i,j} - \th_{t-1}\rep(x_{t,j};\ww_{t-1}),\no\\
    \widetilde{\util}_{t,i,j,2} &= \widetilde{\util}_{t,i,j} - \th_{t-1}\rep(x_{t,i};\ww_{t-1}),\no
\end{align}
and define
\begin{align}
    S_t &= \{i\in \mathcal{C}_t: \Delta \util_{t,\opt,i}> b_t \|\Delta\rep_{t,\opt,i}\|_{V_{t-1}^{-1}}+\mathcal{R}(m)\} ,\no\\
    E_{t}&=\{(k_1,k_2)\in[K]^2: |\Delta \util_{t,k_1,k_2}-\th_{t-1}^\intercal\Delta\rep_{t,k_1,k_2}|\leq \cwct\|\Delta\rep_{t,k_1,k_2}\|_{V_{t-1}^{-1}} + \mathcal{R}(m)\},\no
    \\
    F_{t}&=\{x: x\sim \gauss(\mu, a^2\|\Delta\rep_{i,j}\|_{V_{t-1}^{-1}}), i,j\in[K]:|x-\mathbb{E}[x]|\leq \cwct\sqrt{2\log(Kt^2)}\|\Delta\rep_{t,i,j}\|_{V_{t-1}^{-1}})\}, \no
\end{align} 
where 
$b_t=\cwct(1+\sqrt{2\log(Kt^2)})$. 
Applying the Hoeffding inequality~\citep{lattimore2020bandit}, in $F_{t}$, we can show that the following holds:
\begin{align}
    \frac{1}{2}\|\Delta\rep_{t,i,j}\|_{V_{t-1}^{-1}} \leq \frac{3}{2}\|\Delta\rep_{t,\stt,\ndt}\|_{V_{t-1}^{-1}}\label{eq:cj}.
\end{align} 
We show that $E_t$ and $F_t$ are good events.
\begin{lemma} The following inequalities are satisfied:
    \begin{align}
    \mathbb{P}(E_{t})&\geq 1-\delta,\label{eq:ci3}\\
    \mathbb{P}(F_{t})&\geq 1-\frac{1}{t^2}.\label{eq:normal3}
\end{align}
\end{lemma}
\begin{proof}
\cref{eq:ci3} is obtained by \cref{eq:tsf}.
Since
\begin{align}
    \mathbb{P}(\widehat{\mu}-\mu\geq \epsilon)\leq \frac{1}{2}\exp{\left(-\frac{n\epsilon^2}{2\sigma^2}\right)}, 
\end{align} we get \cref{eq:normal3} in a similar manner to \cref{lemma:normal1} by setting the confidence radius as $ \cwct\sqrt{2\log(Kt^2)}\|\Delta\rep_{t,\arm,\stt}\|_{V_{t-1}^{-1}}\}$ with the sample size $1.$
\end{proof}

\begin{lemma}
  For any filtration $\mathcal{F}_{t-1}$, under $F$,
  \begin{align}
      \mathbb{E}[2r_t | \mathcal{F}_{t-1}]= C(b_t\mathbb{E}[\|\Delta\rep_{t,t}\|_{V_{t-1}^{-1}}|\mathcal{F}_{t-1}] + \mathcal{R}(m) +\frac{1}{t^2})
  \end{align} for some constant $C>0$.
\end{lemma}
\begin{proof}
Notice that
    \begin{align}
        2r_t &=\Delta \util_{t,\opt,\stt}+\Delta \util_{t,\opt,\ndt}.\no
    \end{align} We can estimate that 
    \begin{align}
        \Delta\rep_{t,\opt,\stt}=\th_{t-1}\Delta\rep_{t,\opt,\stt}+b_t\|\Delta\rep_{t,\opt,\stt}\|_{V_{t-1}^{-1}}+\mathcal{R}(m),
        \\\Delta\rep_{t,\opt,\ndt}=\th_{t-1}\Delta\rep_{t,\opt,\ndt}+b_t\|\Delta\rep_{t,\opt,\ndt}\|_{V_{t-1}^{-1}}+\mathcal{R}(m),
    \end{align} in $E_{t}$ and $F_{t}$. Since $\stt, \ndt\in \mathcal{C}_t,$
    \begin{align}
        \th_{t-1}\Delta\rep_{t,\opt,\stt}&\leq  b_t\|\Delta\rep_{t,\opt,\stt}\|_{V_{t-1}^{-1}}+\mathcal{R}(m),\no\\
        \th_{t-1}\Delta\rep_{t,\opt,\ndt}&\leq  b_t\|\Delta\rep_{t,\opt,\ndt}\|_{V_{t-1}^{-1}}+\mathcal{R}(m).\no
    \end{align} in $E_{t}$ and $F_{t}$.
    Thus, under \cref{assumption:utilbound}, we have
    \begin{align}
        \mathbb{E}[2r_t|\mathcal{F}_{t-1}]&\leq \frac{4}{t^2}+4\mathcal{R}(m)+2b_t\mathbb{E}[\|\Delta\rep_{t,\opt,\stt}\|_{V_{t-1}^{-1}}\mid\mathcal{F}_t,E_{t}\bigcap F_{t}]\label{eq:ts-final3} \\ \s+2b_t\mathbb{E}[\|\Delta\rep_{t,\opt,\ndt}\|_{V_{t-1}^{-1}}|\mathcal{F}_t,E_{t}\bigcap F_{t}]\nonumber\\
        &\leq \frac{4}{t^2}+4\mathcal{R}(m)+12b_t\|\Delta\rep_{t,\stt,\ndt}\|_{V_{t-1}^{-1}}. \no
    \end{align} The last inequality is obtained by \cref{eq:cj}. This completes the proof.
\end{proof}

Utilizing \cref{eq:ts-final3} above, we can show the following lemma as \cref{lemma:sm1}:
\begin{lemma}\label{lem:ts-final3}
    Let $Y_t=\sum_{i=1}^{t}X_i$ where $X_t=r_tI_{E}-c_t$ with
\begin{align}
    2c_t = \frac{4}{t^2}+4\mathcal{R}(m)+12b_t\|\Delta\rep_{t,\stt,\ndt}\|_{V_{t-1}^{-1}}.\no
\end{align}
    Then $(Y_t)$ is a super-martingale with respect to the filtration $\mathcal{F}_{t}.$
\end{lemma}
Therefore, combining~\cref{lem:ts-final3} and \cref{lemma:ts-ci}, we complete the proofs of \cref{theorem:app-main_ts} and \cref{theorem:app-main_ts2} under the candidate-based symmetric arm selection strategy.

\section{ADDITIONAL EXPERIMENTS AND IMPLEMENTATION DETAILS} \label{sec:addexp}
\subsection{Datasets}
\subsubsection{Synthetic Dataset}
We consider the following three synthetic tasks having nonlinear utilities $u$:
\begin{itemize}
    \item \textbf{Cosine:} $u(\mathbf{x})=\cos(3\boldsymbol{\Theta}^\intercal \mathbf{x})$
    \item \textbf{Square:} $u(\mathbf{x})=10(\boldsymbol{\Theta}^\intercal \mathbf{x})^2$
    \item \textbf{Quadratic:} $u(\mathbf{x})=\mathbf{x}^\intercal (\boldsymbol{\Theta}\boldsymbol{\Theta}^\intercal) \mathbf{x}$
\end{itemize}
with unknown parameters $\boldsymbol{\Theta}\in\mathbb{R}^{d}$.

These tasks are commonly used as synthetic datasets in both dueling bandits and multi-armed bandits~\citep{ndb, zhang2021neuralthompsonsampling, zhou2020neural, dai2022federated}.
For each task, the environment first draws a parameter vector $\boldsymbol{\Theta}$ from the uniform distribution $\mathcal{U}([-1,1]^d)$.
Then, in each round $t$, it samples a set of context vectors
$$ \mathcal{X}_t = \bigl\{ \mathbf{x}_{t,1}, \ldots, \mathbf{x}_{t,K} \bigr\}, $$
where each $\mathbf{x}_{t,k}$ is drawn independently from $\mathcal{U}([-1,1]^{d})$ for $k \in [K]$.
Finally, each $\mathbf{x}_{t,k}$ is normalized by its Euclidean norm, $\|\mathbf{x}_{t,k}\|_2$. The following code provides how to generate $\bigl\{\{\mathbf{x}_{t,k}\}_{k=1}^{K}\bigr\}_{t=1}^{T}.$

\begin{lstlisting}[language=Python, caption=Pseudocode for constructing context sets $X_t$.]
context_arms_list = []
for _ in range(T):
    # Generate all context_arm pairs
    context_arms = np.random.uniform(
        low=-1, high=1, size=(self.arms, self.dim)
    )
    context_arms = context_arms / np.linalg.norm(
        context_arms, 2, axis=1
    ).reshape(-1, 1)
    context_arms_list.append(context_arms)
\end{lstlisting}

\subsubsection{Real-World Dataset}\label{sec:real-world}
To verify the effectiveness of our method, we consider real-world datasets provided by the \textbf{UCI Machine Learning Repository}\footnote{The UCI Machine Learning Repository, developed by the University of California, Irvine, is a widely used collection of datasets for empirical evaluation and benchmarking of machine learning algorithms. Available at \url{https://archive.ics.uci.edu/}.}---Statlog, Magic, and Covertype---each consisting of feature vectors associated with one of multiple possible actions (arms) (see \cref{tab:uci}).

To the best of our knowledge, there is no standard experimental setup for dueling bandits using UCI datasets (unlike for multi-armed bandits~\citep{shallow}); thus, we designed the tasks as follows. For the \textbf{Magic} and \textbf{Covertype} datasets, we defined the pairwise preference probability using a sigmoid function:
$$ P(i \succ j) = \sigma(u_i -u_j) = \frac{1}{1 + e^{-(u_i - u_j)}} $$
where $u_k$ represents the underlying true utility value corresponding to category index $k \in [K]$ in each round, and we set a baseline of $P(i \succ j) = 0.7$ for clarity, though the preference is derived from the utility difference $u_i - u_j$.

Note that variance-aware algorithms are not expected to consistently outperform variance-agnostic algorithms in normal tasks, as their advantages are primarily realized in scenarios with very low variances. Thus, to specifically evaluate performance under \textbf{low-variance conditions}, we set a deterministic preference model for the \textbf{Statlog} dataset:
$$ P(i \succ j) = 1 \quad \text{for } i > j $$
This setup ensures the advantage of a higher-indexed arm is realized with certainty.

\begin{table}[h]
\centering
\caption{UCI benchmark datasets used for contextual dueling bandit experiments.}
\resizebox{0.8\textwidth}{!}{
\begin{tabular}{lccc}
\toprule
Dataset & Number of Attributes & Number of Arms & Number of Instances \\
\midrule
Statlog & 9 & 7 & 58,000 \\
Magic   & 11 & 2 & 19,020 \\
Covertype & 54 & 7 & 581,012 \\
\bottomrule
\end{tabular}}\label{tab:uci}
\end{table}
\subsection{Implementation Details} 
\subsubsection{Hyperparameters.}

We report the hyperparameters used in our main manuscript. To evaluate \term, we employ a fully connected deep neural network (DNN) with ReLU activations. The architecture consists of $L = 2$ hidden layers, each with $W = 32$ units, and an output feature dimension of $D$.

We set the following training parameters:
\begin{itemize}
    \item Regularization parameter: $\lambda = 1.0$
    \item Number of gradient steps per round: $\gs = 20$
    \item Number of episodes: $\epi = 1$
\end{itemize}
We set the \textbf{confidence width coefficient} $\cwct = 1.0$. Note that this constant value is adopted from previous works \citep{ndb, shallow, zhou2020neural, bui2024variance} because computing the exact theoretical value of $\cwct$ is challenging in the context of neural networks. Specifically, this difficulty arises from the non-trivial computation of the differential term $\diff$. Implementing an exact algorithm to find the parameters $\theta_t$ that satisfy the first-order optimality condition $\frac{\partial \loss_t}{\partial \theta}(\theta_t, \ww_t) = 0$ is infeasible. Therefore, we approximate the parameter set $(\theta_t, \ww_t)$ via gradient descent on the loss function $\loss_t$ using a learning rate of $\lr = 0.01$ at each round $t$. Optimization is performed using the Adam optimizer in PyTorch, utilizing its default parameters except for the specified learning rate. Additionally, we reinitialize the parameter set $V_{t-1}$ to its initial state at the beginning of each round, as this was found not to incur significant computational overhead. Each experiment is run across \textbf{20 random seeds}. Unless otherwise stated, the remaining sections in this supplementary material adhere to the hyperparameters described above.



\subsubsection{Pseudocodes of Arm Selection Strategies}
The following are the pseudocodes for the asymmetric, optimistic symmetric, and candidate-based symmetric arm selection strategies, respectively.

\begin{lstlisting}[language=Python, caption=Pseudocode for the asymmetric arm selection in both UCB and TS frameworks as in~\citep{ndb}]
utilities = neural_network(context_actions)

# Selecting the first arm
at_1 = torch.argmax(utilities).item()
at_2 = 1
max_score = -np.inf
for j in range(len(utilities)):
    # Difference of features of the selected arm and the current arm
    zt_j1 = grad_list[j] - grad_list[at_1]
    zt_j1 = zt_j1.to("cpu")
    zt_dot_U = torch.matmul(zt_j1, torch.inverse(self.V))
    zt_dot_U_zt = torch.matmul(zt_dot_U, zt_j1.t())
    conf_term = torch.clamp(zt_dot_U_zt, min=0)
    sigma = self.nu * torch.sqrt(conf_term)

    # Selecting the arm based on the strategy
    utilities_j1 = utilities[j]
    if self.strategy == "ts":
        utilities_j1 = utilities_j1.to("cpu")
        action_score = torch.normal(utilities_j1.view(-1), sigma.view(-1))

        # Alternative: np.random.normal(loc=utilities_j1.item(), scale=sigma.item())

    elif self.strategy == "ucb":
        action_score = utilities_j1.item() + sigma.item()

    else:
        raise RuntimeError("Exploration strategy not set")

    # Selecting the second best arm
    if action_score > max_score:
        max_score = action_score
        at_2 = j
# Keeping the selected arms for model update
return at_1, at_2
\end{lstlisting}
\begin{lstlisting}[language=Python, caption=Pseudocode for the optimistic symmetric arm selection in both UCB and TS frameworks]
utilities = neural_network(context_actions)
at_1, at_2 = 0, 0
max_val = -torch.inf
if self.strategy == "ucb":
    for i in range(len(utilities)):
        for j in range(len(utilities)):
            zt = grad_list[i] - grad_list[j]
            zt_dot_U = torch.matmul(zt, torch.inverse(self.V))
            zt_dot_U_zt = torch.matmul(zt_dot_U, zt.t())
            conf_term = torch.clamp(zt_dot_U_zt, min=0)
            sigma = self.nu * torch.sqrt(conf_term)
            estimated = utilities[i] + utilities[j] + sigma
            if max_val < estimated: # Find best optimistic arms
                at_1, at_2 = i, j
                max_val = estimated
elif self.strategy == "ts":
    mus = torch.zeros((len(utilities), len(utilities)))
    sigmas = torch.zeros((len(utilities), len(utilities)))
    for i in range(len(utilities)):
        for j in range(len(utilities)):
            mus[i, j] = utilities[i] + utilities[j]
            zt = grad_list[i] - grad_list[j]
            zt_dot_U = torch.matmul(zt, torch.inverse(self.V))
            zt_dot_U_zt = torch.matmul(zt_dot_U, zt.t())
            conf_term = torch.clamp(zt_dot_U_zt, min=0)
            sigma = self.nu * torch.sqrt(conf_term)
            sigmas[i, j] = sigma
    sigmas = sigmas + 1e-12
    mus = mus.detach().cpu().numpy()
    sigmas = sigmas.detach().cpu().numpy()

    # Sample best optimistic arms
    sample = np.random.normal(loc=mus, scale=sigmas) 
    row_index, col_index = np.where(sample == np.max(sample))
    at_1, at_2 = row_index.item(), col_index.item()
\end{lstlisting}

\begin{lstlisting}[language=Python, caption=Pseudocode for the candidate-based symmetric arm selection in both UCB and TS frameworks]
utilities = neural_network(context_actions)
at_1, at_2 = 0, 0
if self.strategy == "ucb":
    sigmas = np.zeros((len(context_actions), len(context_actions)))
    mask = np.zeros((len(context_actions), len(context_actions)))
    for i in range(len(utilities)):
        for j in range(len(utilities)):
            if i == j:
                mask[i, j] = 1
                sigmas[i, j] = 0
                continue
            zt = grad_list[i] - grad_list[j]
            zt_dot_U = torch.matmul(zt, torch.inverse(self.V))
            zt_dot_U_zt = torch.matmul(zt_dot_U, zt.t())
            conf_term = torch.clamp(zt_dot_U_zt, min=0)
            sigma = self.nu * torch.sqrt(conf_term)
            sigmas[i, j] = sigma
            if utilities[i] + sigma > utilities[j]:
                mask[i, j] = 1
    # Make a candidate set
    candidate_mask = mask.sum(axis=1) == len(utilities)
    candidate_mask = candidate_mask.reshape(-1, 1) * candidate_mask.reshape(1, -1)
    sigmas = sigmas + 1

    at_1, at_2 = np.unravel_index(
        np.argmax(candidate_mask * sigmas, axis=None),
        (len(context_actions), len(context_actions)),
    )
elif self.strategy == "ts":
    sigmas = torch.zeros((len(context_actions), len(context_actions)))
    mask = torch.zeros((len(context_actions), len(context_actions)))
    for i in range(len(utilities)):
        for j in range(len(utilities)):
            if i == j:
                mask[i, j] = 1
                sigmas[i, j] = 0
                continue
            zt = grad_list[i] - grad_list[j]
            zt_dot_U = torch.matmul(zt, torch.inverse(self.V))
            zt_dot_U_zt = torch.matmul(zt_dot_U, zt.t())
            conf_term = torch.clamp(zt_dot_U_zt, min=0)
            sigma = self.nu * torch.sqrt(conf_term)
            sigmas[i, j] = sigma
            if utilities[i] + sigma > utilities[j]:
                mask[i, j] = 1
                
    # Make a candidate set
    candidate_mask = mask.sum(axis=1) == len(utilities)
    candidate_mask = candidate_mask.reshape(-1, 1) * candidate_mask.reshape(1, -1)
    candidate_mask = candidate_mask.cpu().numpy()
    
    sigmas = sigmas.detach().cpu().numpy()
    mus = sigmas
    sigmas = sigmas * sigmas / 4 + 1e-12
    
    sample = np.random.normal(loc=mus, scale=sigmas)
    sample[candidate_mask == False] = -np.inf
    row_index, col_index = np.where(sample == np.max(sample))
    at_1, at_2 = row_index.item(), col_index.item()
\end{lstlisting}       

\newpage
\subsection{Additional Experiments} 
In this section, we conduct additional experiments supporting our proposed method \term.

\begin{figure*}[h!]
    \centering
    \resizebox{1\textwidth}{!}{
    \begin{tabular}{ccc}
    \includegraphics[width=0.33\columnwidth]{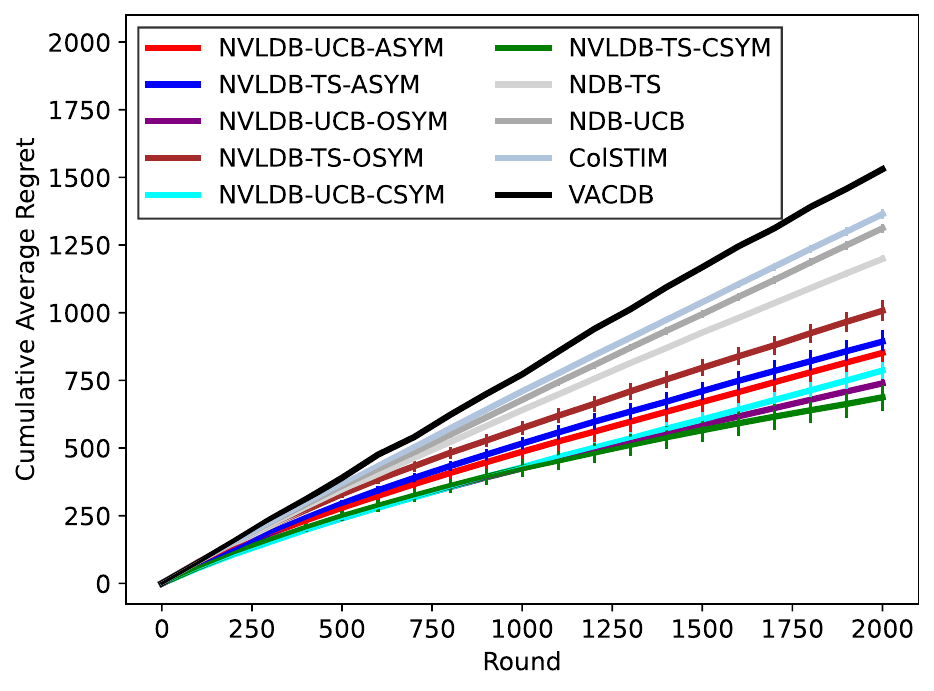}&
    \includegraphics[width=0.33\columnwidth]{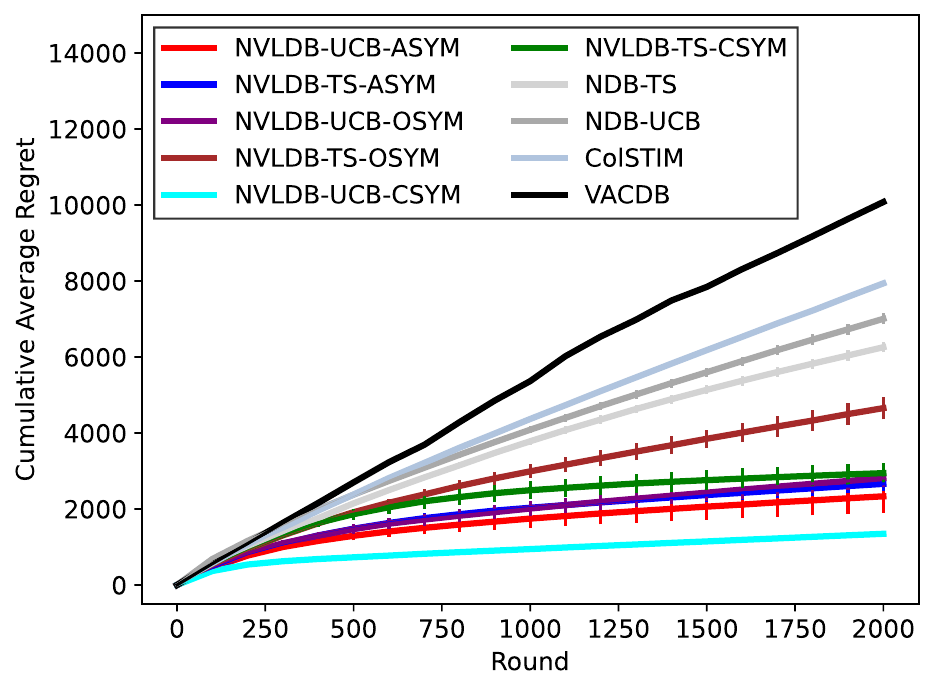}&
    \includegraphics[width=0.33\columnwidth]{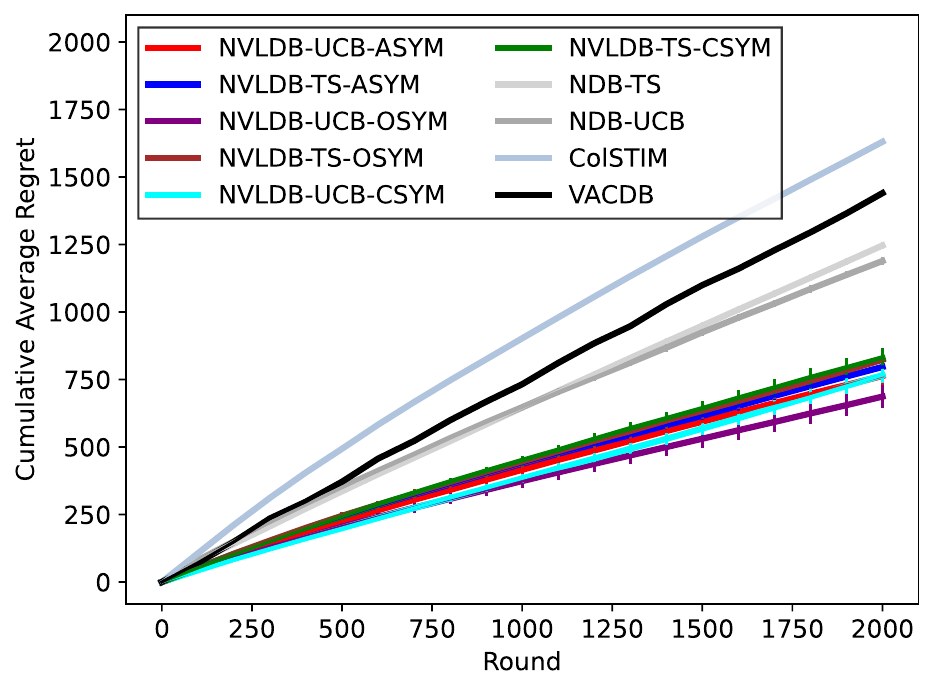}\\
    {\small (a): $\cos(3\Theta^\intercal x)$} &  {\small (b): $10(\Theta^\intercal x)^2$} & {\small (c): $x^\intercal\Theta\Theta^\intercal x$}
    \\
    \includegraphics[width=0.3\columnwidth]{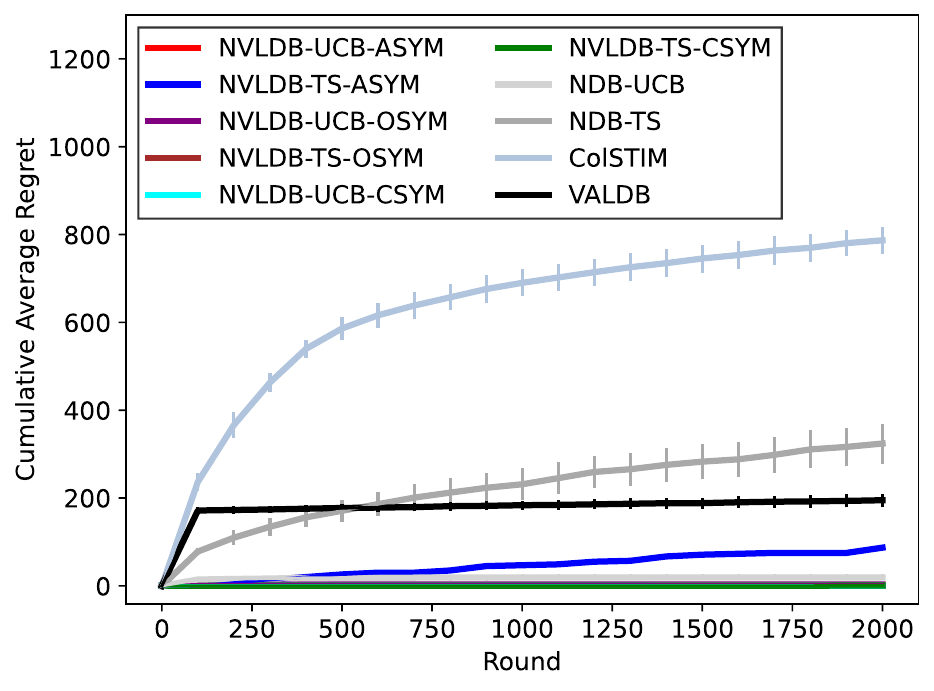}&
    \includegraphics[width=0.3\columnwidth]{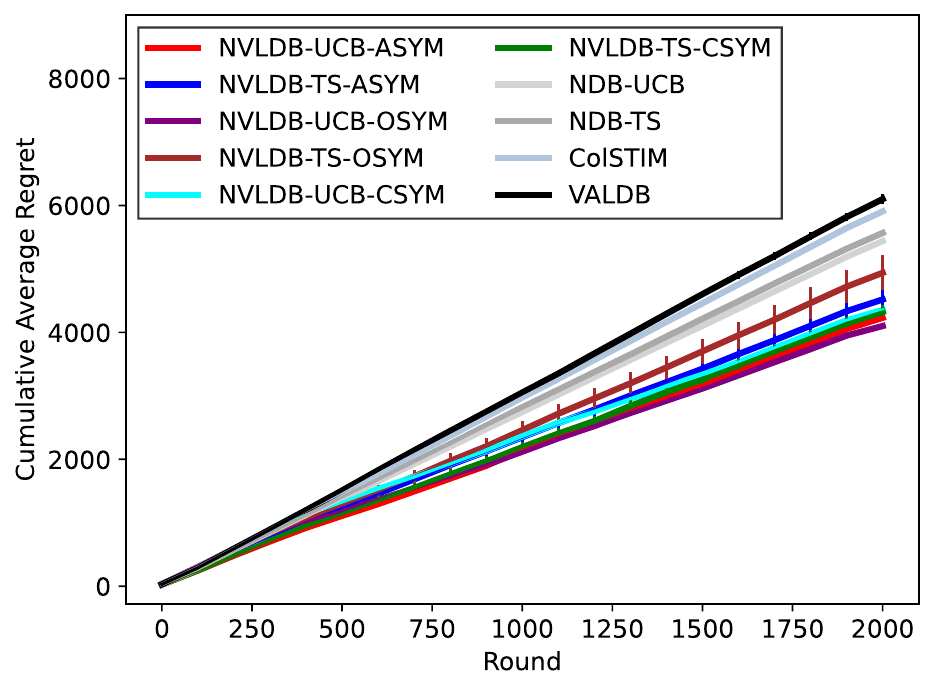}&
    \includegraphics[width=0.3\columnwidth]{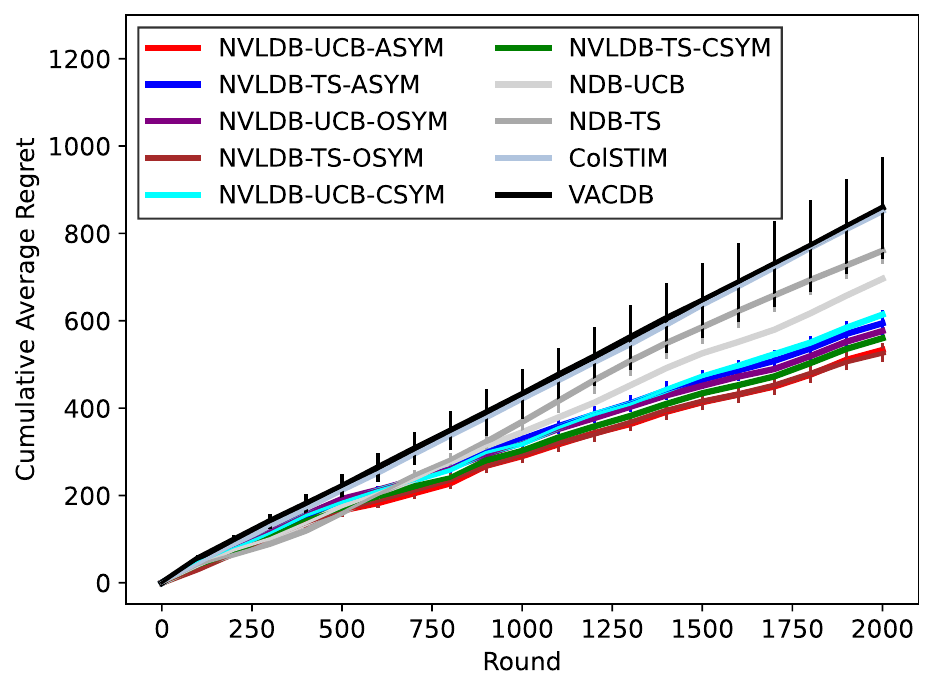}\\
    {\small (d): UCI-statlog} &  {\small (e): UCI-covertype} & {\small (f): UCI-magic}
    \end{tabular}}
    \caption{\small Comparison of cumulative average regret under (a)-(c): Synthetic tasks with context dimension $\dim=5$ and $K=5$ arms. (d)-(f): UC Irvine (UCI) machine learning repository data. All variants of NVLDB mostly outperform other baselines including NDB in both the UCB and TS frameworks. Each experiment is repeated across 20 random seeds.
}
    \label{fig:app-main_result_all}
\end{figure*}
\begin{figure*}[h!]
    \centering
    \begin{tabular}{ccc}
    \includegraphics[width=0.3\columnwidth]{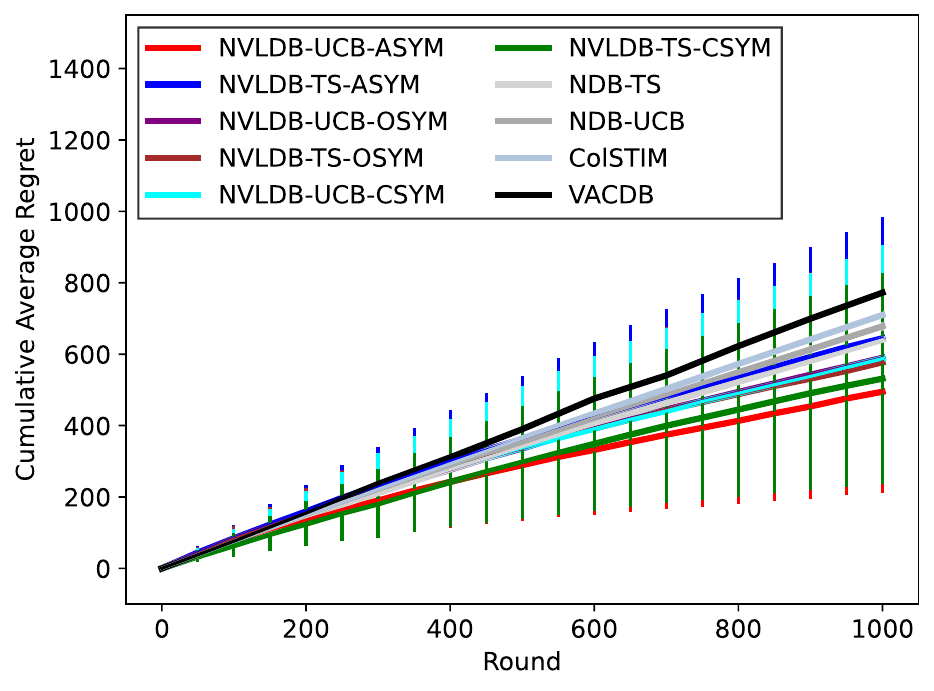}&
    \includegraphics[width=0.3\columnwidth]{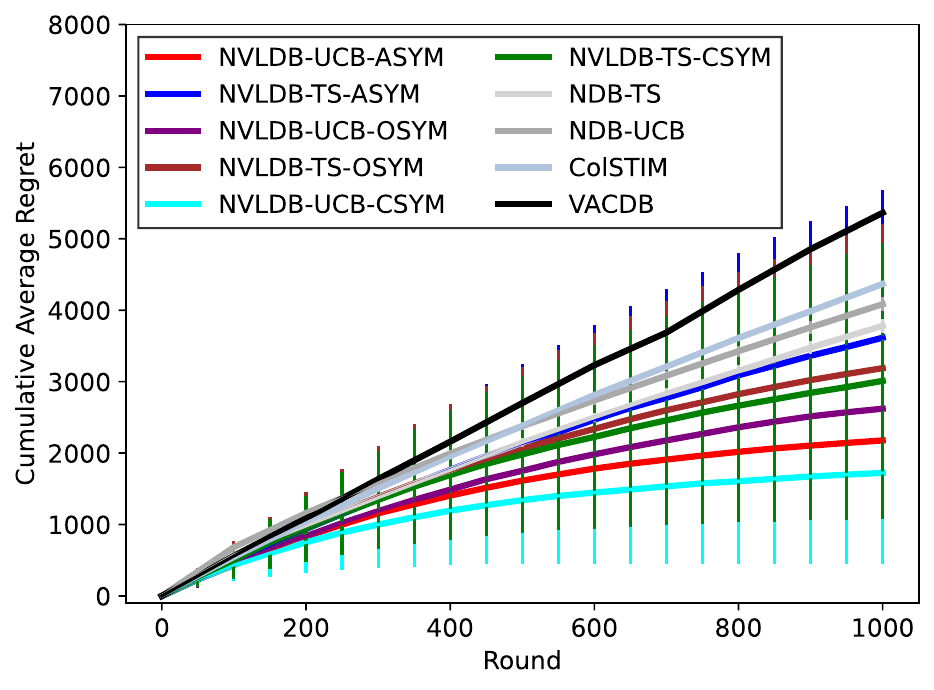}&
    \includegraphics[width=0.3\columnwidth]{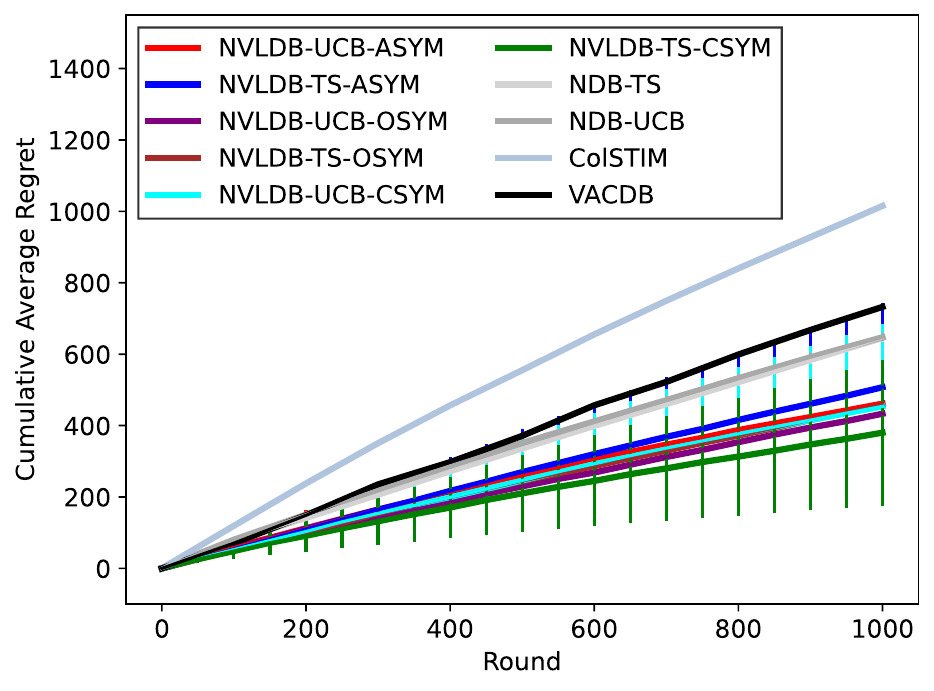}\\
    {\small (a): $\cos(3\Theta^\intercal x)$} &  {\small (b): $10(\Theta^\intercal x)^2$} & {\small (c): $x^\intercal \Theta^\intercal\Theta x$}
    \end{tabular}
    \caption{Comparison of cumulative average regret with increased network width. The experiments utilize an increased neural network width of $\wid=100$ (compared to the standard $\wid=32$ reported elsewhere), while the dimension is set to $\dim=5$ and the number of arms to $K=5$ for all tasks. The results demonstrate that our methods consistently outperform NDB and achieve sublinear cumulative regret. Each experiment is repeated across 20 random seeds.}
    \label{fig:width2}
\end{figure*}
\paragraph{Comparison of Cumulative Average Regret on Synthetic and Real-World Datasets.}
While the main manuscript presents comparisons of cumulative average regret on synthetic datasets with nonlinear utility functions, we additionally provide results on real-world datasets from the UCI Machine Learning Repository, as described in \cref{sec:real-world}.
\cref{fig:app-main_result_all} illustrates the cumulative average regret of \term's variance-aware variants compared to baseline algorithms.In particular, when comparing our methods against established baselines, including linear dueling bandit algorithms (VALDB~\citep{variance-aware} and ColSTIM~\citep{lst}) and the neural dueling bandit (NDB) algorithm~\citep{ndb}, it is noteworthy that all our proposed arm selection strategies consistently outperform these baselines. Furthermore, our methods exhibit sublinear cumulative regrets across both the synthetic and real-world datasets. These empirical results strongly support our theoretical findings, as established in \cref{theorem:app-main_ucb} and \cref{theorem:app-main_ts}. Each experiment was conducted using an average over 20 random seeds.

\paragraph{The Size of Neural Networks.}
To further demonstrate that \term guarantees sublinear cumulative average regret for each task, we increase the neural network width from $\wid=32$ to $\wid=100$. Based on our theoretical results (\cref{theorem:app-main_ucb} and \cref{theorem:app-main_ts}), we anticipate that the average cumulative regret will remain sublinear for the wider network ($\wid=100$), given that it already exhibited sublinear behavior at $\wid=32$. \cref{fig:width2} illustrates the cumulative average regrets of \term  under different arm selection strategies in comparison to the baselines. As theoretically and empirically expected, in all presented cases, the variants of \term consistently outperform the baselines and continue to show sublinear regret behavior. Each experiment was conducted using an average over 20 random seeds.

\paragraph{The Environment Condition.}
As we stated in the main manuscript, \cref{theorem:app-main_ucb} and \cref{theorem:app-main_ts} denote that the cumulative average regret should be sublinear even if the environment parameters $K$ and $\dim$ are changed. Inspired by this, we conduct experiments adjusting these parameters.
First, we increase the context dimension $\dim$ from $5$ to $10$ while keeping all other settings unchanged. \cref{fig:add2} illustrates the performance on tasks with $\dim=10$ and $K=5$, revealing substantial gaps between the \term variants and the baselines. Additionally, \cref{fig:red2} demonstrates the performance on tasks with $\dim=5$ and $K=10$. These results indicate that \term consistently outperforms the baselines across different task configurations, supporting our theoretical results. Each experiment was conducted using an average over 20 random seeds.
\begin{figure*}[h!]
    \centering
    \begin{tabular}{ccc}
    \includegraphics[width=0.3\columnwidth]{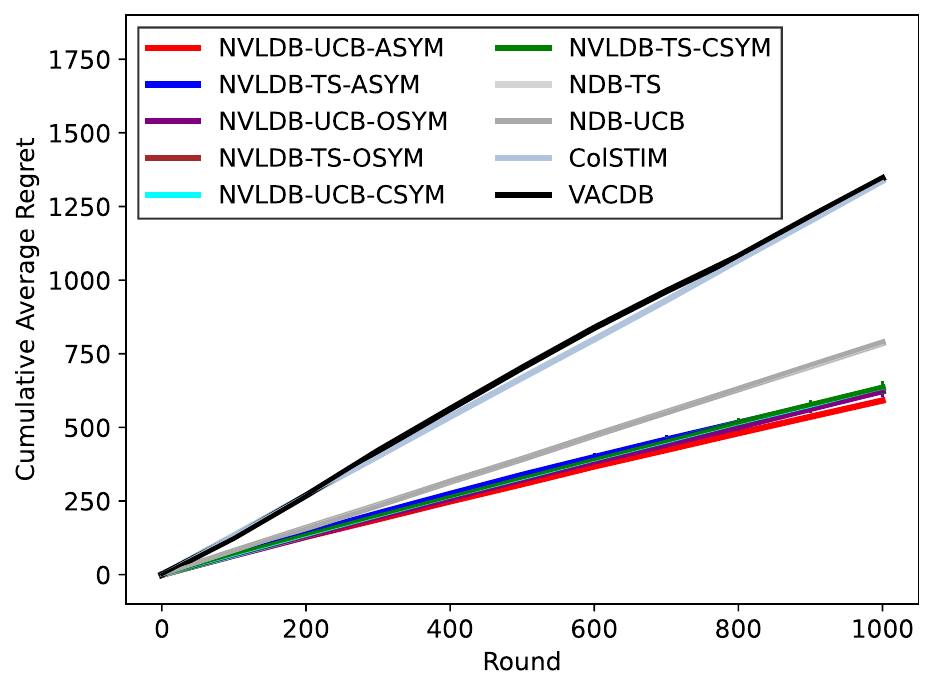}&
    \includegraphics[width=0.3\columnwidth]{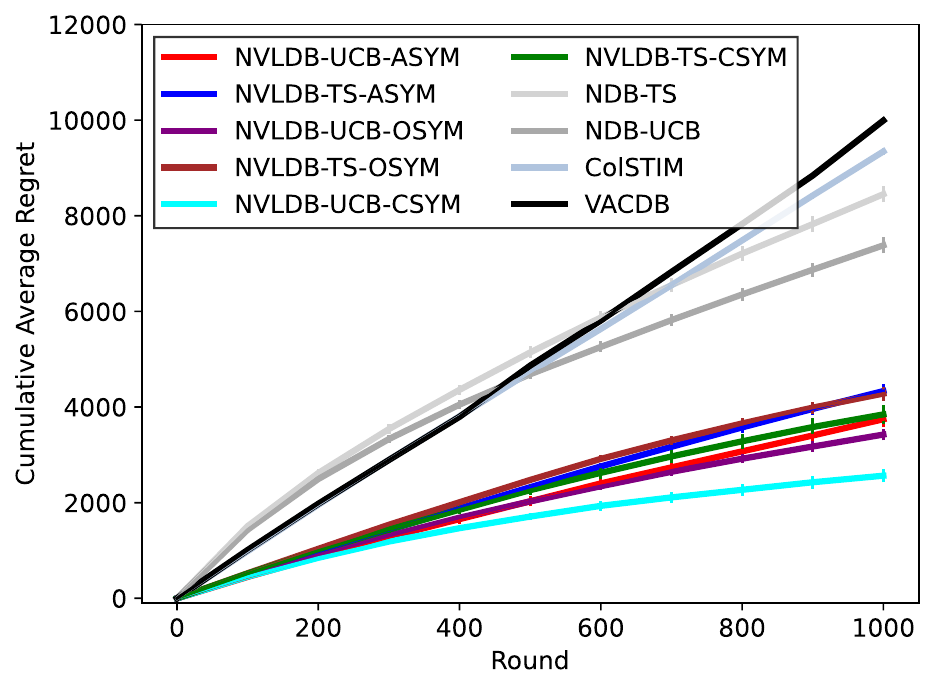}&
    \includegraphics[width=0.3\columnwidth]{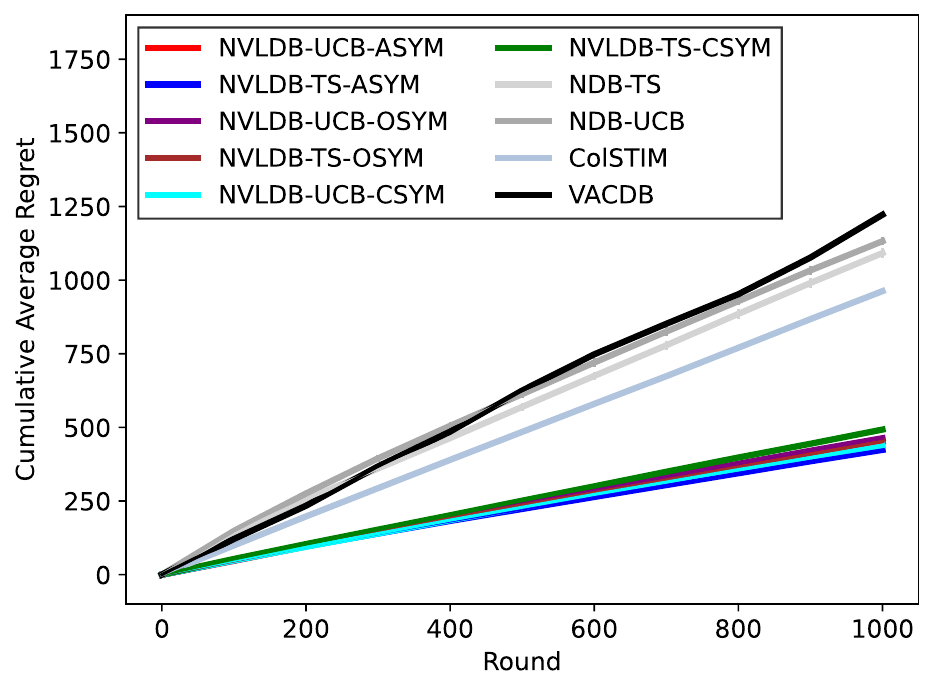}\\
    {\small (a): $\cos(3\Theta^\intercal x)$} &  {\small (b): $10(\Theta^\intercal x)^2$} & {\small (c): $x^\intercal \Theta\Theta^\intercal x$}
    \end{tabular}
    \caption{Comparison of cumulative average regret across our methods (\term variants) and the other baselines. Each task uses $\dim=10$ and $K=5$. It is clear that our variants outperform the other baselines. Each experiment was conducted across 20 random seeds.}
    \label{fig:add2}
\end{figure*}
\begin{figure*}[h!]
    \centering
    \begin{tabular}{ccc}
    \includegraphics[width=0.3\columnwidth]{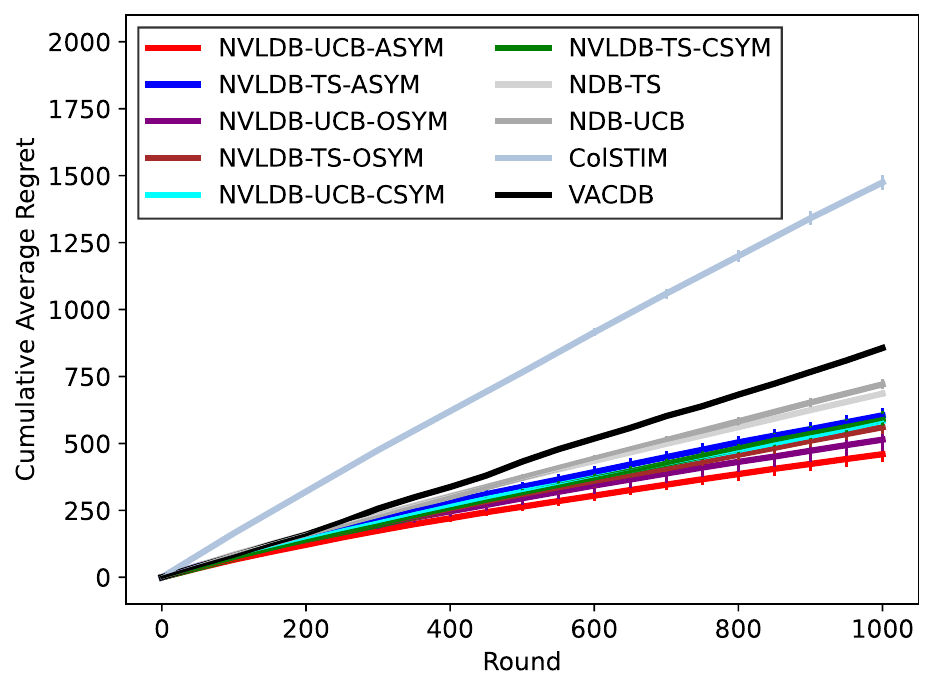}&
    \includegraphics[width=0.3\columnwidth]{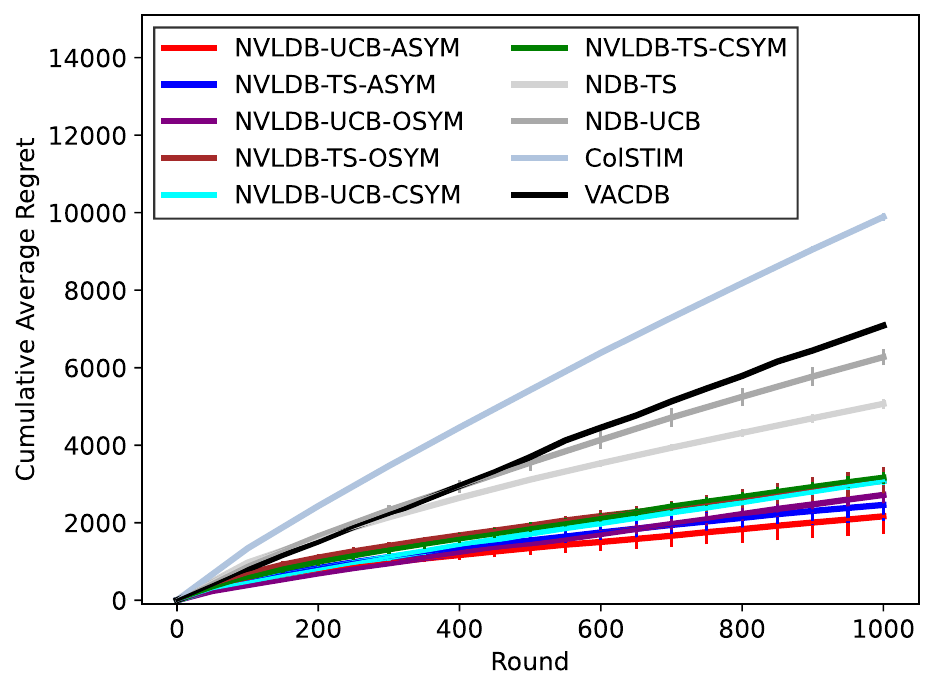}&
    \includegraphics[width=0.3\columnwidth]{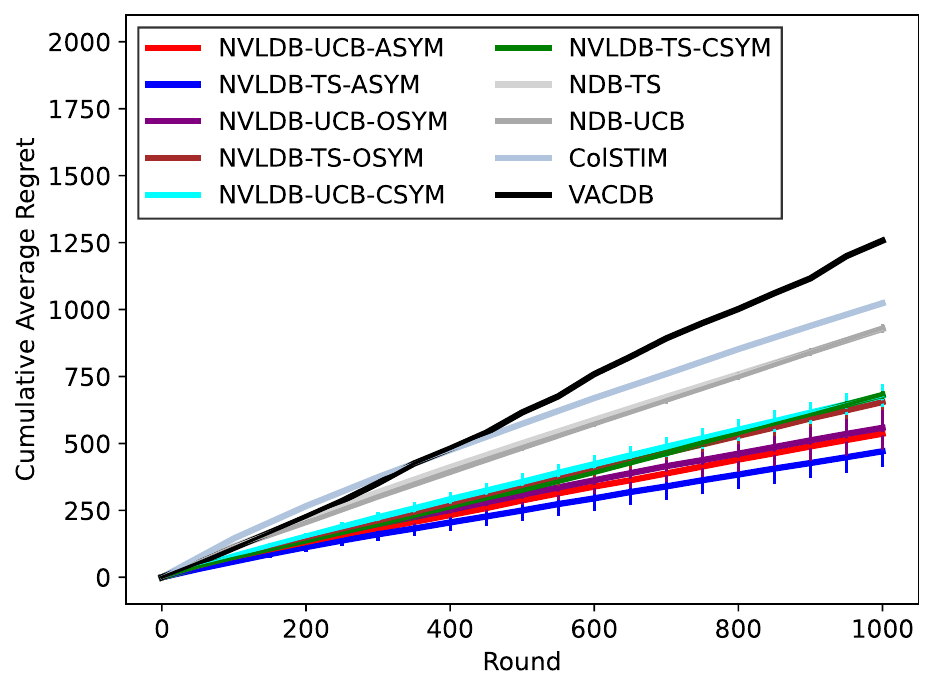}\\
    {\small (a): $\cos(3\Theta^\intercal x)$} &  {\small (b): $10(\Theta^\intercal x)^2$} & {\small (c): $x^\intercal \Theta\Theta^\intercal x$}
    \end{tabular}
    \caption{Comparison of cumulative average regret across our methods (\term variants) and the other baselines. Each task uses $\dim=5$ and $K=10$. It is clear that our variants outperform the other baselines. Each experiment was conducted across 20 random seeds.}
    \label{fig:red2}
\end{figure*}

\paragraph{Variance-agnostic Variants}
As stated in the main manuscript, \cref{theorem:app-main_ucb2} and \cref{theorem:app-main_ts2} indicate that, even without incorporating variance-awareness, the cumulative average regret of our arm selection strategies remains theoretically sublinear. To validate these theoretical results empirically, we conduct experiments comparing the performance of the variance-aware (\term) and variance-agnostic variants (NLDB, which stands for \emph{Neural Linear Dueling Bandit}) of our method. \cref{fig:NV1}--\cref{fig:NV3} compare the cumulative average regrets of these two sets of variants across different synthetic datasets featuring nonlinear utility functions.

As stated in the main manuscript, \cref{theorem:app-main_ucb2} and \cref{theorem:app-main_ts2} indicate that, without variance-awareness, the cumulative average regret of our arm selection strategies remains sublinear. To validate these theoretical results, we conduct experiments under variance-agnostic variants of \term. \cref{fig:NV1}--\cref{fig:NV3} compare the cumulative average regrets of the variance-aware (\term) and variance-agnostic variants (NLDB) of \term across different synthetic datasets with nonlinear utility functions.

We make two primary observations from these comparative experiments. First, the variance-agnostic NLDB (Neural Linear Dueling Bandit) methods exhibit sublinear cumulative average regret regardless of the arm selection strategy employed. This empirical finding is consistent with our theoretical results presented in \cref{theorem:app-main_ucb2} and \cref{theorem:app-main_ts2}, which suggest that a sublinear rate is attainable even without explicit variance-awareness. Second, the variance-aware variants generally outperform the variance-agnostic counterparts (NLDB) in terms of cumulative average regret, highlighting the practical benefit of incorporating variance information into the arm selection strategy. However, in some specific instances, the variance-agnostic variants (NLDB) are observed to perform better (see \cref{fig:NV3}(c)). This observation is still consistent with the theory, as both the variance-aware (\term) and variance-agnostic (NLDB) variants share the same theoretical upper bound on cumulative average regret, namely $\mathcal{O}(\dim\sqrt{T})$, provided that the variances are sufficiently large. Moreover, theoretical guarantees are typically derived under idealized assumptions that may not hold perfectly in practice. Thus, slight deviations in empirical performance between the two approaches are inevitable. \textbf{Crucially, these discrepancies do not undermine} the validity of the theoretical results but rather highlight the inherent gap between idealized theoretical modeling and real-world implementation.

\begin{figure*}[h!]
    \centering
    \begin{tabular}{ccc}
    \includegraphics[width=0.3\columnwidth]{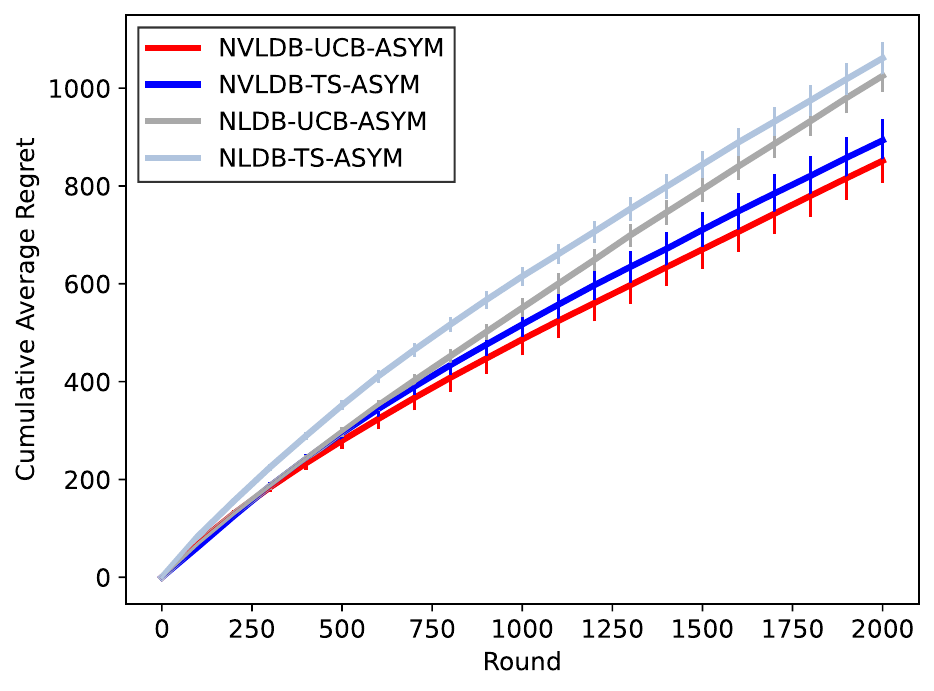}&
    \includegraphics[width=0.3\columnwidth]{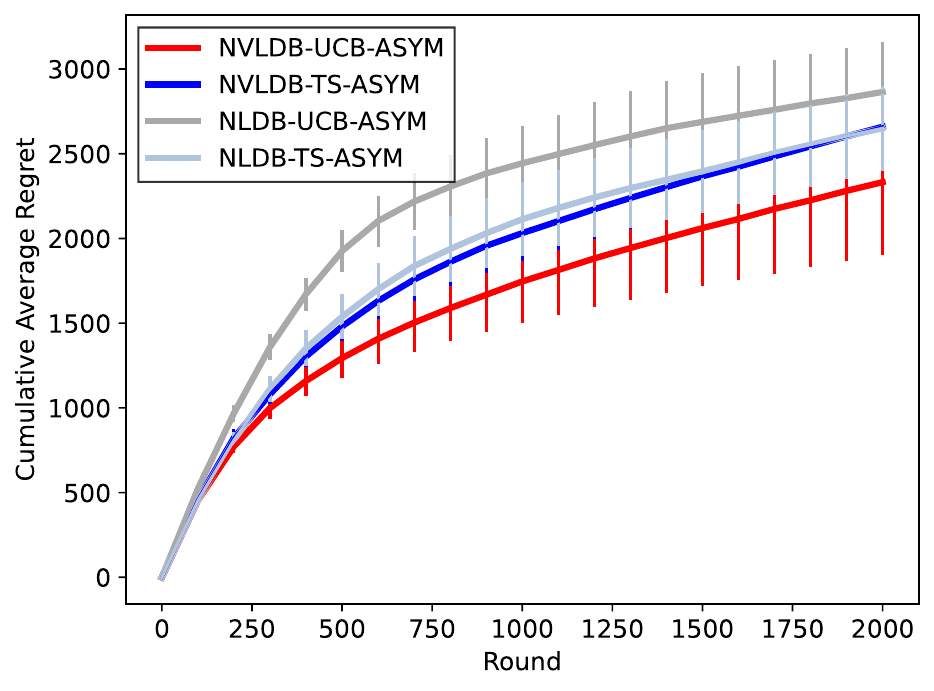}&
    \includegraphics[width=0.3\columnwidth]{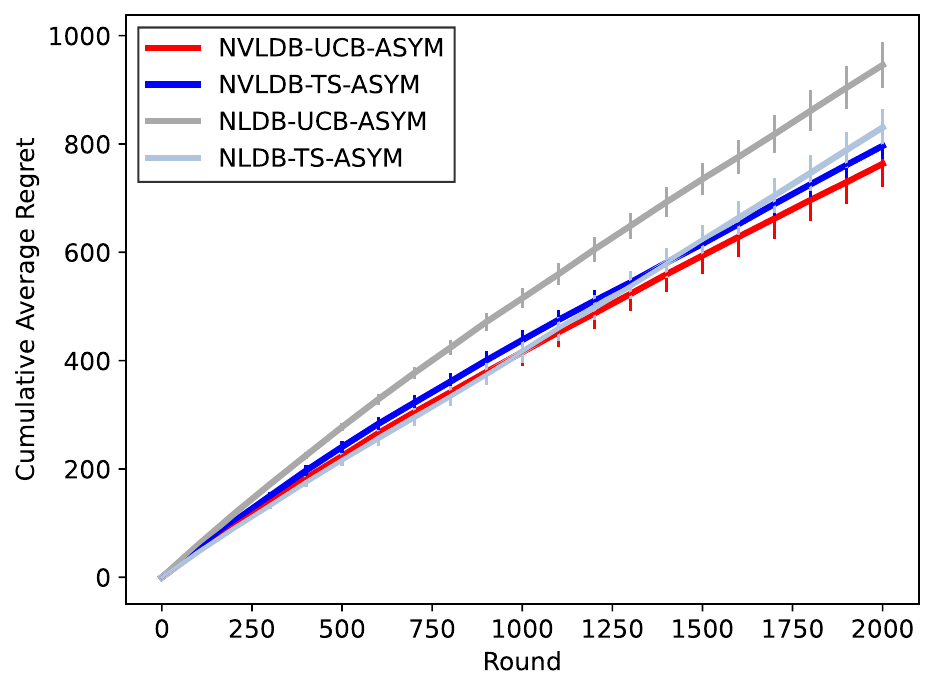}\\
    {\small (a): $\cos(3\Theta^\intercal x)$} &  {\small (b): $10(\Theta^\intercal x)^2$} & {\small (c): $x^\intercal \Theta\Theta^\intercal x$}
    \end{tabular}
    \caption{Comparison of cumulative average regret between variance-aware (NVLDB) and variance-agnostic (NLDB) variants of our proposed method under asymmetric (ASYM) arm selection on synthetic tasks with $\dim=5$ and $K=5$. Each experiment was conducted across 20 random seeds.}
    \label{fig:NV1}
\end{figure*}

\begin{figure*}[h!]
    \centering
    \begin{tabular}{ccc}
    \includegraphics[width=0.3\columnwidth]{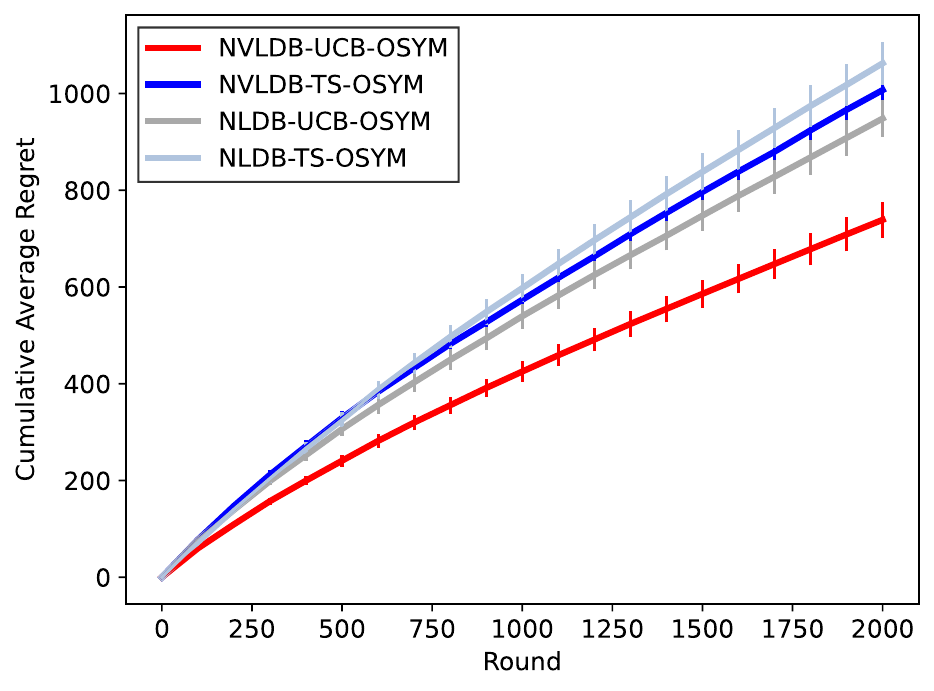}&
    \includegraphics[width=0.3\columnwidth]{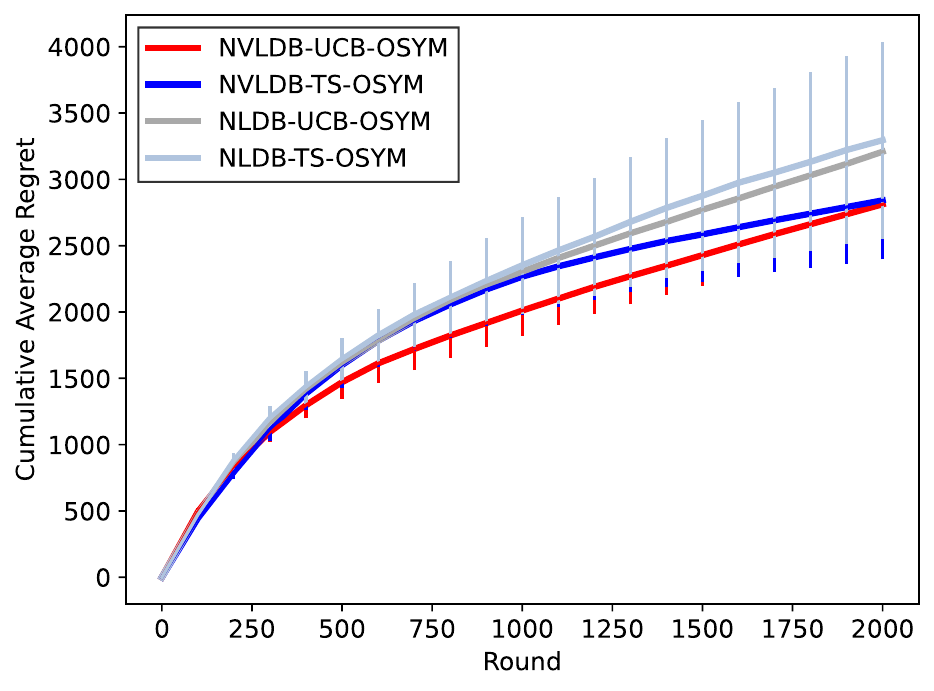}&
    \includegraphics[width=0.3\columnwidth]{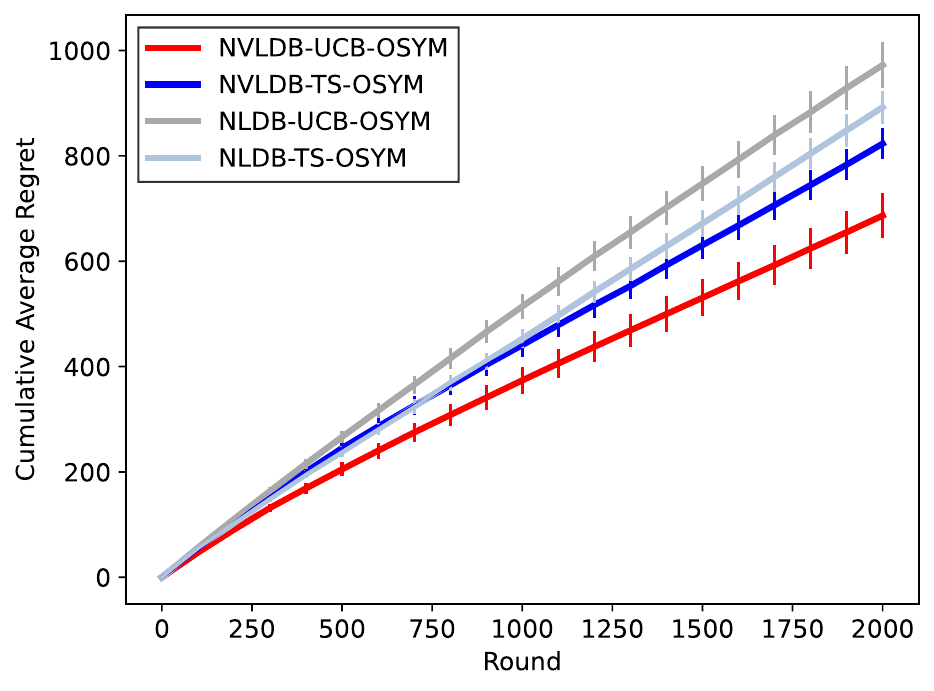}\\
    {\small (a): $\cos(3\Theta^\intercal x)$} &  {\small (b): $10(\Theta^\intercal x)^2$} & {\small (c): $x^\intercal \Theta\Theta^\intercal x$}
    \end{tabular}
    \caption{Comparison of cumulative average regret between variance-aware (NVLDB) and variance-agnostic (NLDB) variants of our proposed method under optimistic symmetric (OSYM) arm selection on synthetic tasks with $\dim=5$ and $K=5$. Each experiment was conducted across 20 random seeds.}
    \label{fig:NV2}
\end{figure*}

\begin{figure*}[h!]
    \centering
    \begin{tabular}{ccc}
    \includegraphics[width=0.3\columnwidth]{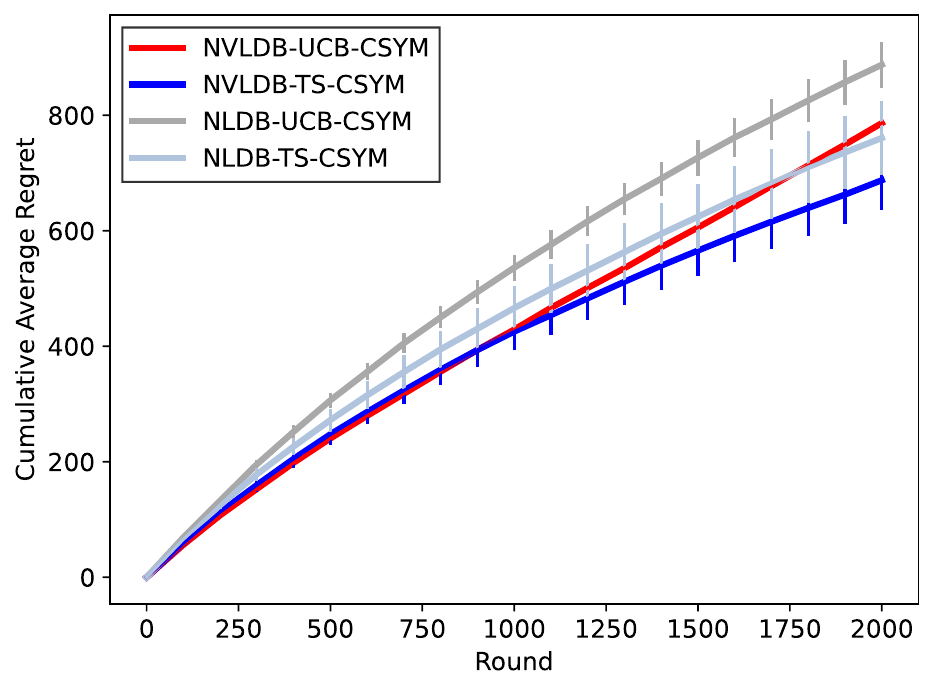}&
    \includegraphics[width=0.3\columnwidth]{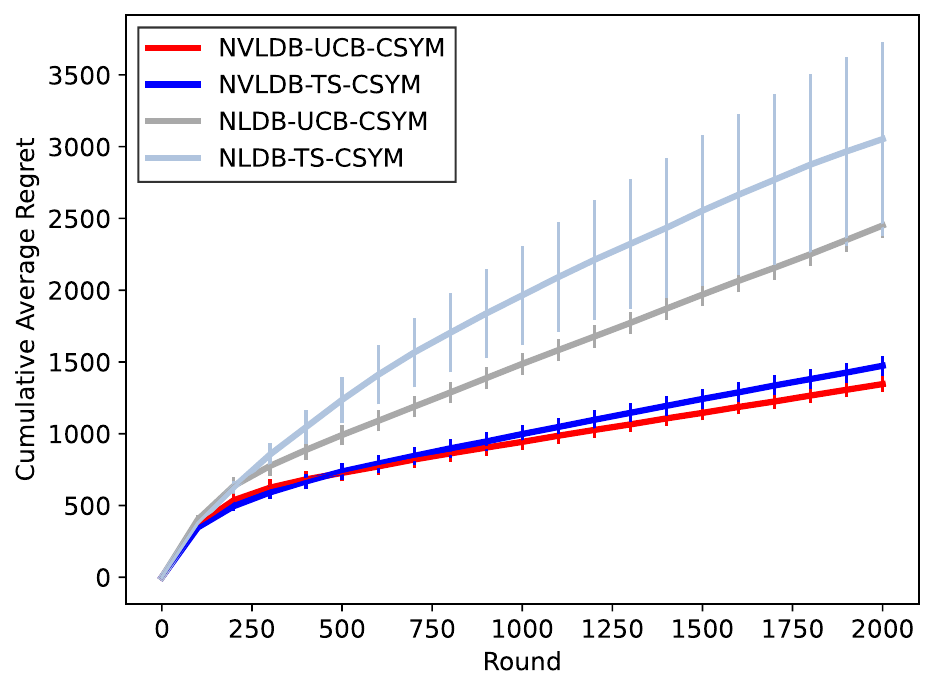}&
    \includegraphics[width=0.3\columnwidth]{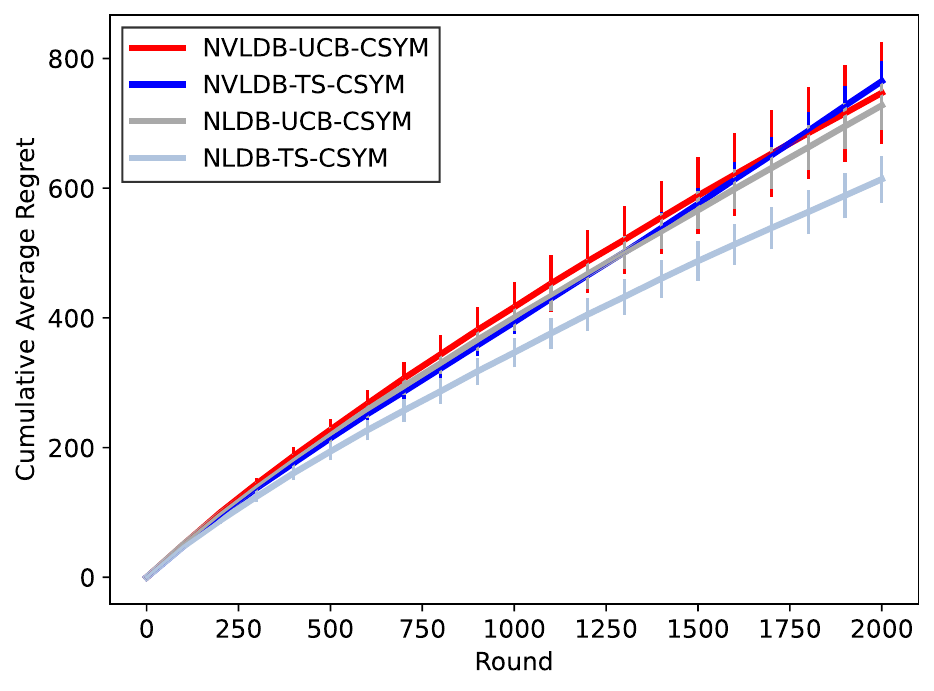}\\
    {\small (a): $\cos(3\Theta^\intercal x)$} &  {\small (b): $10(\Theta^\intercal x)^2$} & {\small (c): $x^\intercal \Theta\Theta^\intercal x$}
    \end{tabular}
    \caption{Comparison of cumulative average regret between variance-aware (NVLDB) and variance-agnostic (NLDB) variants of our proposed method under candidate-based symmetric (CSYM) arm selection on synthetic tasks with $\dim=5$ and $K=5$. Each experiment was conducted across 20 random seeds.}
    \label{fig:NV3}
\end{figure*}

\paragraph{Different Confidence Interval Coefficients}
For practical reasons, we chose and fixed the confidence interval coefficient $\cwct$. The selection of a fixed $\cwct$ has been widely adopted in similar neural bandit studies \citep{shallow, ndb, bui2024variance, zhou2020neural}, as previously stated. Intuitively, the parameter $\cwct$ controls the trade-off between exploration and exploitation. Both excessively small and excessively large values of $\cwct$ are detrimental to establishing sublinear cumulative average regret. A small $\cwct$ leads to weak exploration, potentially trapping the algorithm in suboptimal regions, while a large $\cwct$ conducts excessive and inefficient exploration. Indeed, excessively large or small values of this coefficient result in looser upper bounds during the proof of \cref{theorem:app-main_ucb} through \cref{theorem:app-main_ts2}.
We empirically evaluated the cumulative average regret for various values of $\beta$. \cref{fig:NV4} compares the cumulative average regrets of \term using asymmetric arm selection under the UCB regime. The results show that a moderate value of $\cwct$ provides the best performance, and performance degrades when $\cwct$ is too small or too large. This observation aligns well with our theoretical analysis regarding the exploration-exploitation balance.


\begin{figure*}[h!]
    \centering
    \begin{tabular}{ccc}
    \includegraphics[width=0.3\columnwidth]{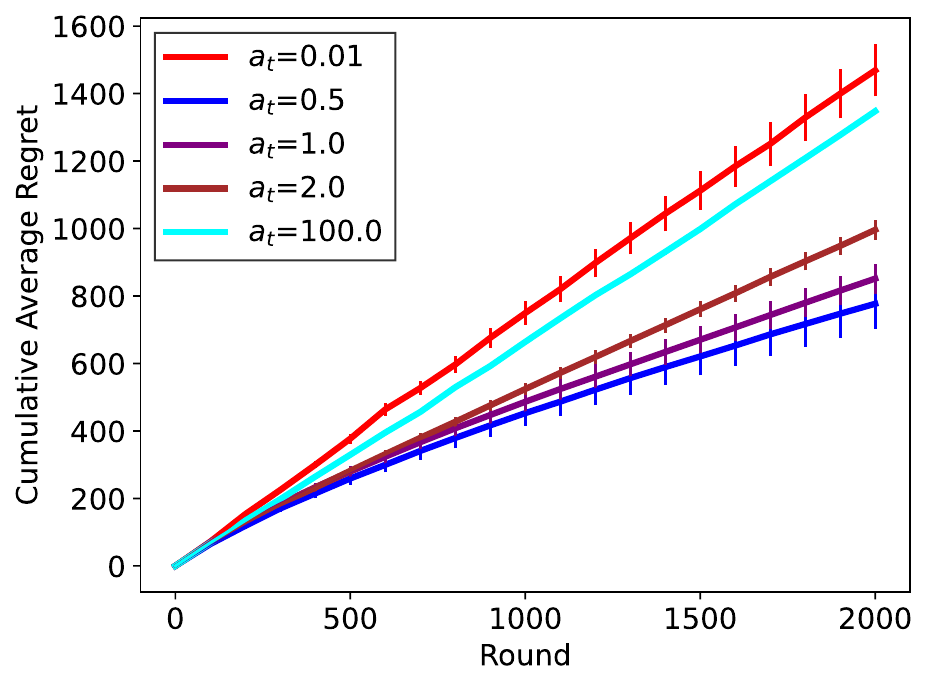}&
    \includegraphics[width=0.3\columnwidth]{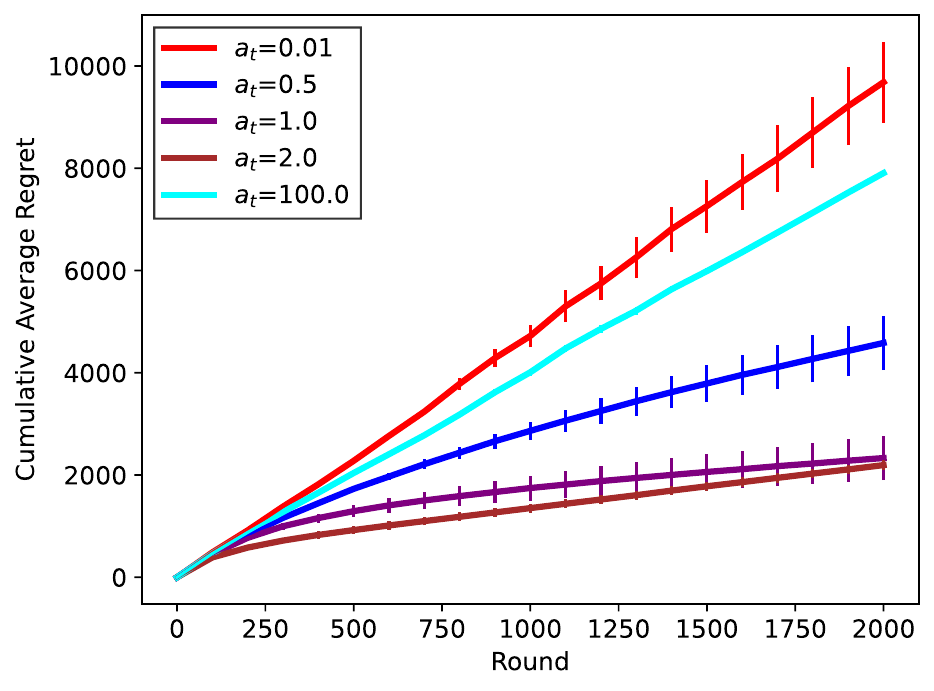}&
    \includegraphics[width=0.3\columnwidth]{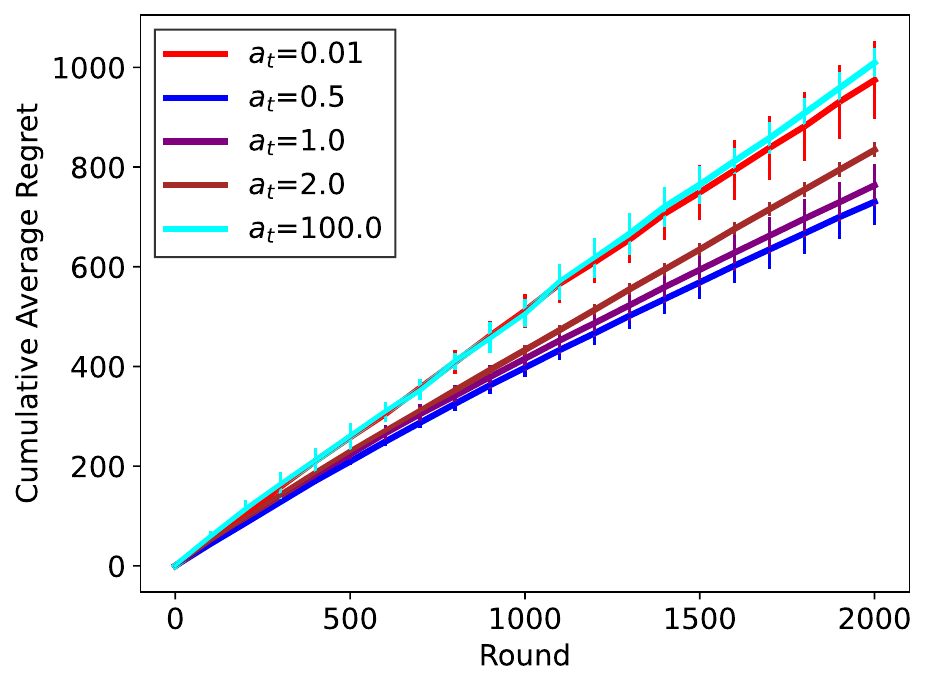}\\
    {\small (a): $\cos(3\Theta^\intercal x)$} &  {\small (b): $10(\Theta^\intercal x)^2$} & {\small (c): $x^\intercal \Theta\Theta^\intercal x$}
    \end{tabular}
    \caption{Comparison of cumulative average regret of \term with asymmetric arm selection strategy arm selection under UCB regime across different confidence interval coefficient $a_t$ on synthetic tasks with $\dim=5$ and $K=5$. Each experiment was conducted across 20 random seeds.}
    \label{fig:NV4}
\end{figure*}

\paragraph{Computational Efficiency.}
We compared the wall-clock times of \term-UCB-ASYM, which is one of our methods, and NDB~\citep{ndb} over 2,000 rounds across 20 random seeds. Measurements were taken on an NVIDIA GeForce RTX 3060 GPU paired with an AMD Ryzen 9 5950X 16-core CPU.

As a result, the proposed method significantly outperforms NDB in terms of computational efficiency. Specifically, \textbf{\term-UCB-ASYM completes 2,000 rounds in an average of $1.71 \pm 0.78$ minutes, while NDB requires $48.8 \pm 15.5$ minutes under the same conditions.} This substantial disparity is justified by the underlying computational processes: NDB utilizes gradients with respect to the full set of learnable parameters in the neural network, whereas our method utilizes gradients with respect to only the learnable parameters of the last layer to construct the Gram matrix.


\newpage

{\color{black}
\subsection{Network Width Sensitivity}\label{sec:width_sensitivity}
To verify scalability with respect to network width, we extended experiments to $\wid \in \{500, 1000\}$, confirming sublinear regrets over $T=10{,}000$ rounds (square utility, $\dim=5$, $K=5$) for all \term variants. \cref{tab:width_sensitivity} summarizes the $R(T)/T$ values, and \cref{fig:width_avg,fig:width_total} visualize the cumulative average and total regret curves, respectively.

\begin{table}[h]
\centering
\caption{$R(T)/T$ values across network width $\wid \in \{500, 1000\}$ on the square utility task ($\dim=5$, $K=5$). All tested algorithms exhibit steadily decreasing $R(T)/T$, following $\widetilde{\mathcal{O}}(T^{-1/2})$ convergence. Each experiment is repeated across 20 random seeds.}
\label{tab:width_sensitivity}
\resizebox{0.8\textwidth}{!}{
\begin{tabular}{lccc}
\toprule
Algorithm & $R(2000)/2000$ & $R(6000)/6000$ & $R(10000)/10000$ \\
\midrule
NVLDB-UCB-ASYM ($\wid=500$) & $1.66 \pm 0.78$ & $0.82 \pm 0.39$ & $0.54 \pm 0.19$ \\
NVLDB-UCB-OSYM ($\wid=500$) & $0.97 \pm 0.17$ & $0.51 \pm 0.07$ & $0.39 \pm 0.04$ \\
NVLDB-UCB-CSYM ($\wid=500$) & $1.13 \pm 0.35$ & $0.61 \pm 0.13$ & $0.46 \pm 0.07$ \\
NVLDB-UCB-ASYM ($\wid=1000$) & $1.37 \pm 0.22$ & $0.65 \pm 0.09$ & $0.42 \pm 0.04$ \\
NVLDB-UCB-OSYM ($\wid=1000$) & $1.73 \pm 0.65$ & $0.81 \pm 0.27$ & $0.52 \pm 0.13$ \\
NVLDB-UCB-CSYM ($\wid=1000$) & $1.54 \pm 0.65$ & $0.73 \pm 0.27$ & $0.51 \pm 0.14$ \\
\bottomrule
\end{tabular}
}
\end{table}

\begin{figure*}[h!]
    \centering
    \begin{tabular}{ccc}
    \includegraphics[width=0.3\columnwidth]{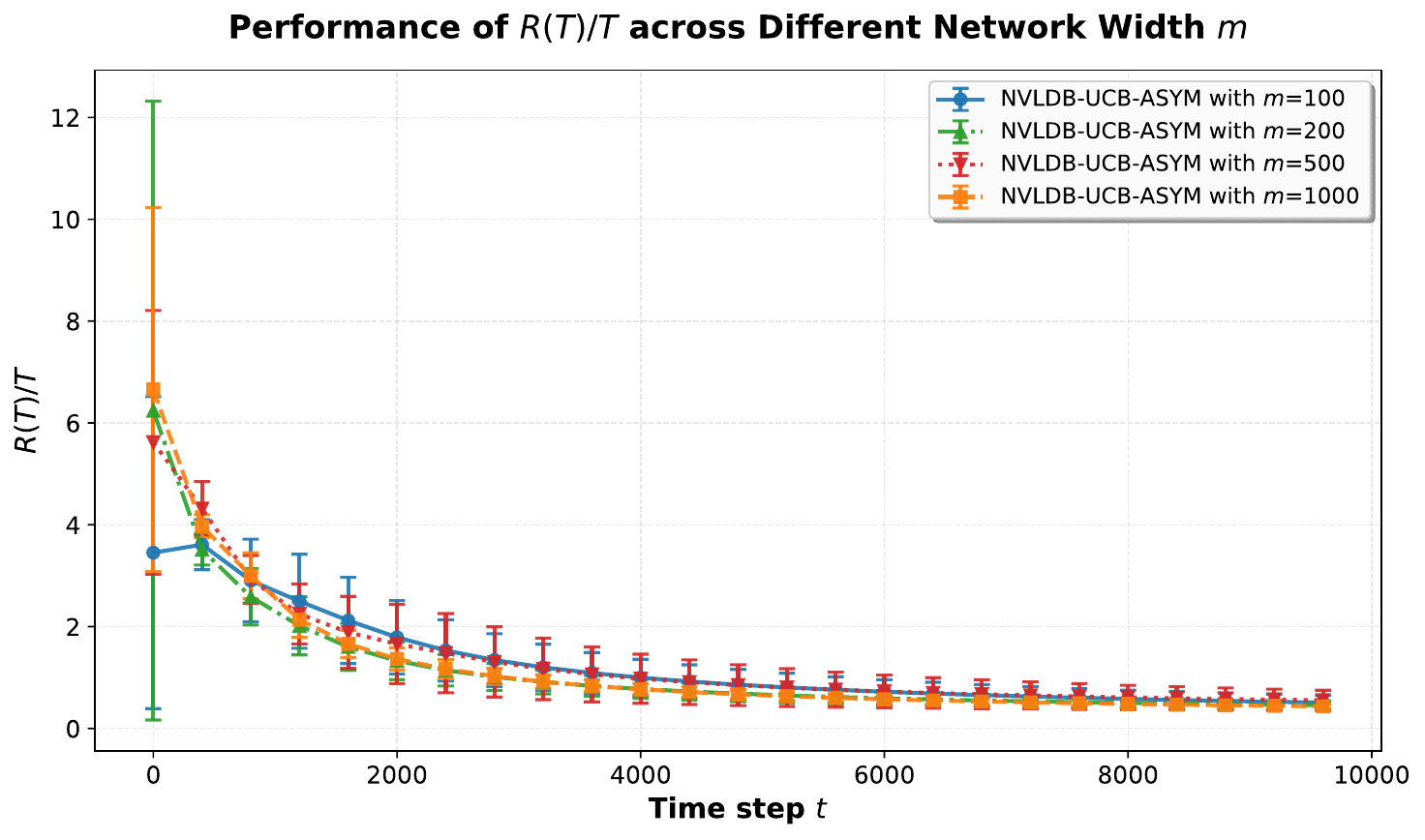}&
    \includegraphics[width=0.3\columnwidth]{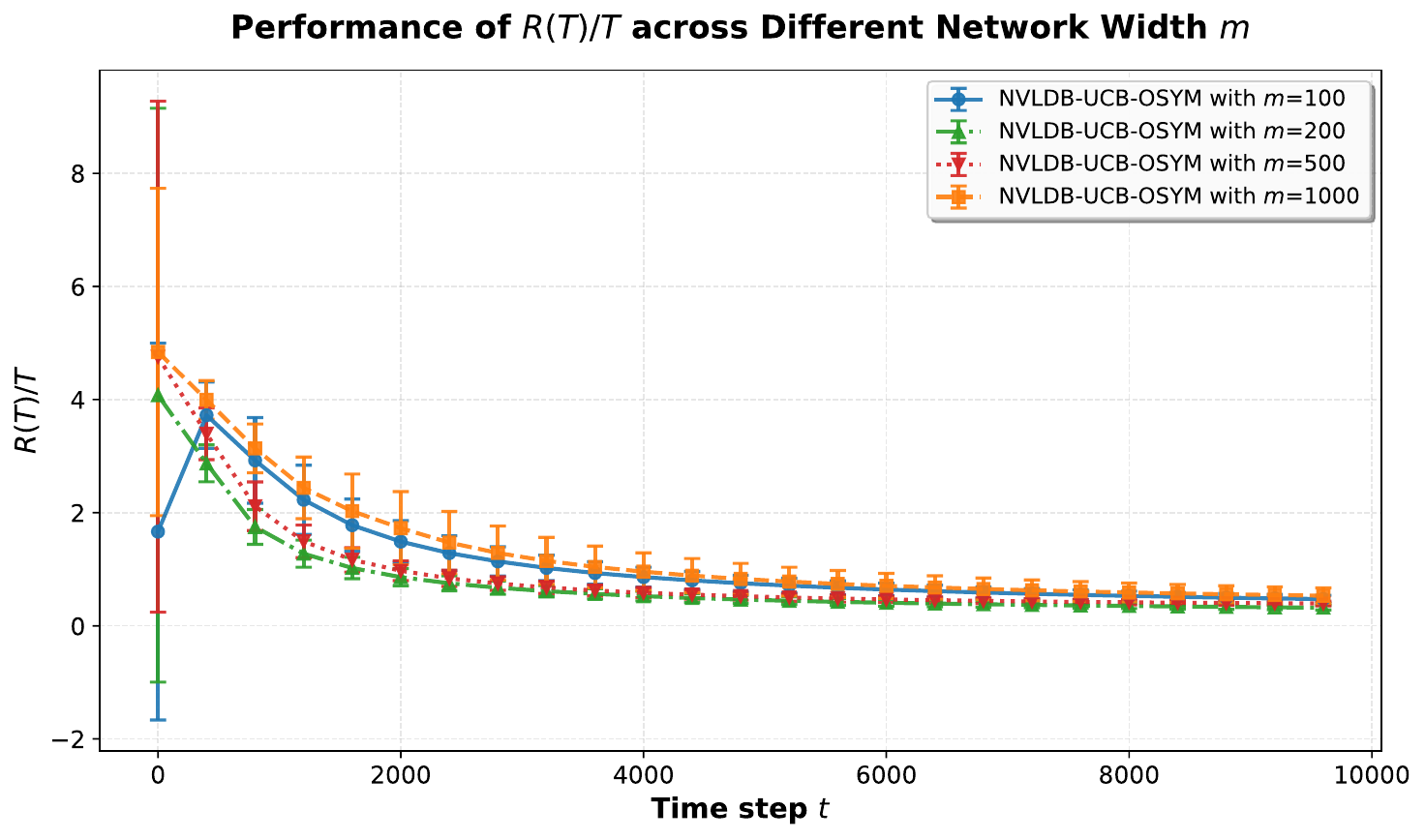}&
    \includegraphics[width=0.3\columnwidth]{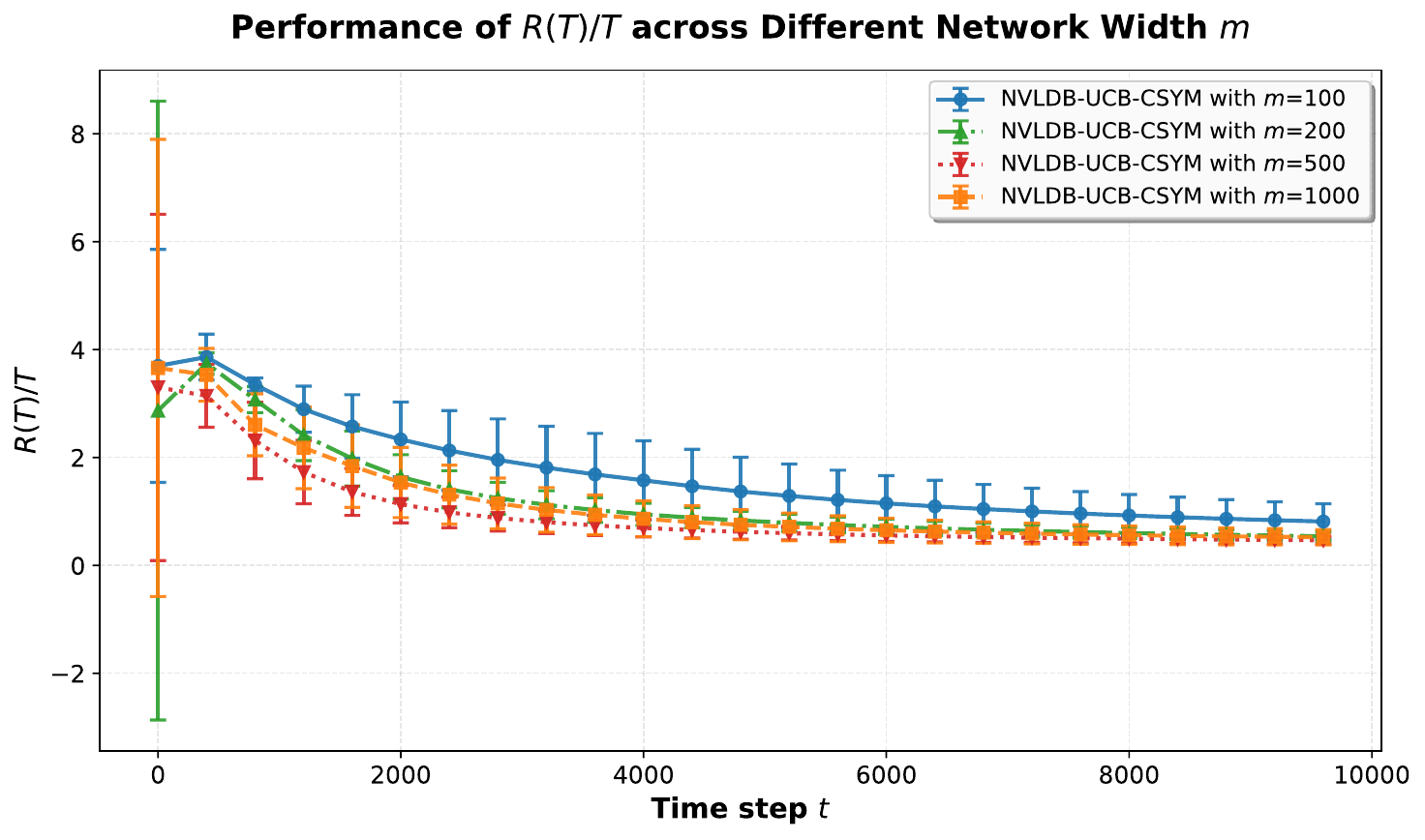}\\
    {\small (a): UCB-ASYM} & {\small (b): UCB-OSYM} & {\small (c): UCB-CSYM} \\[6pt]
    \includegraphics[width=0.3\columnwidth]{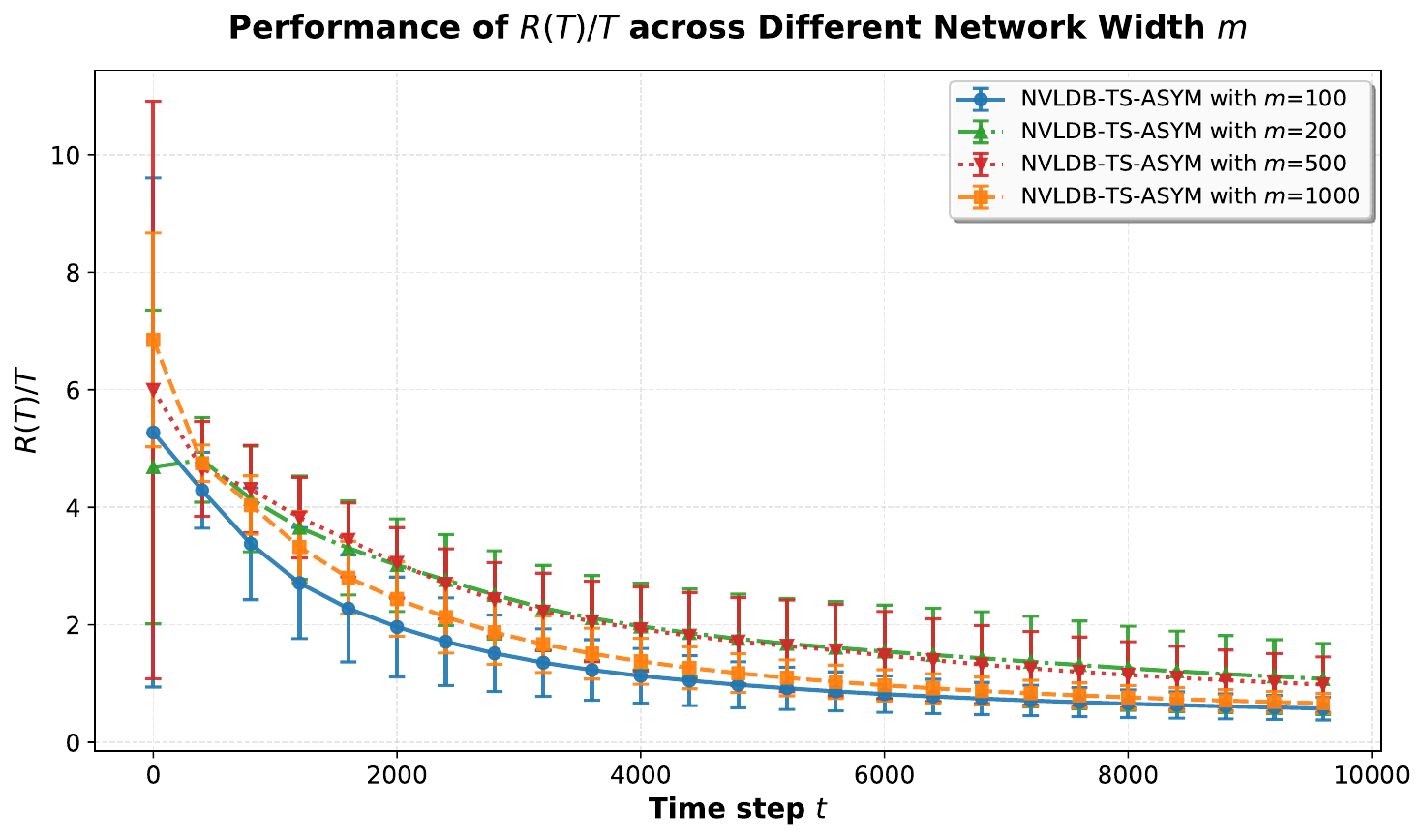}&
    \includegraphics[width=0.3\columnwidth]{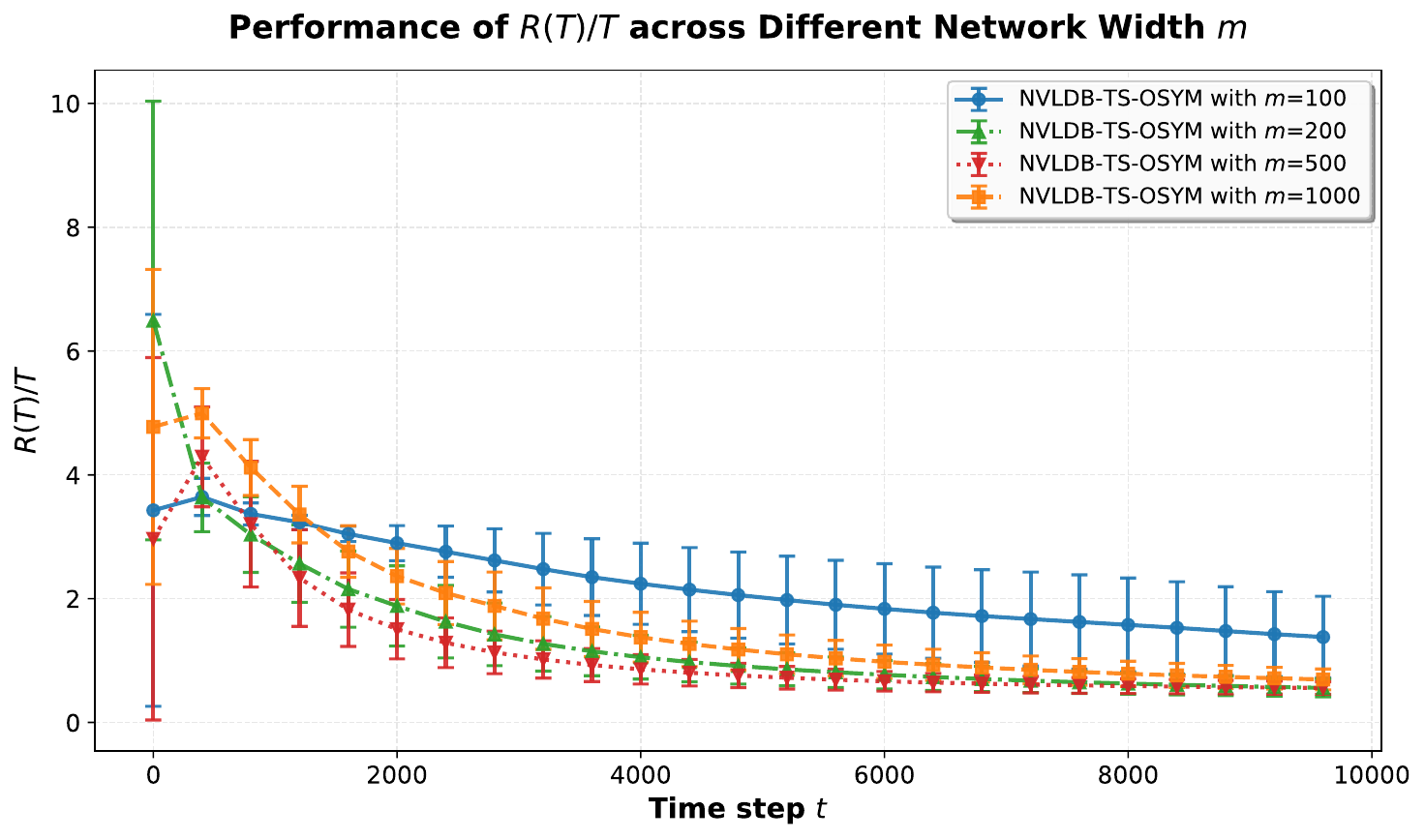}&
    \includegraphics[width=0.3\columnwidth]{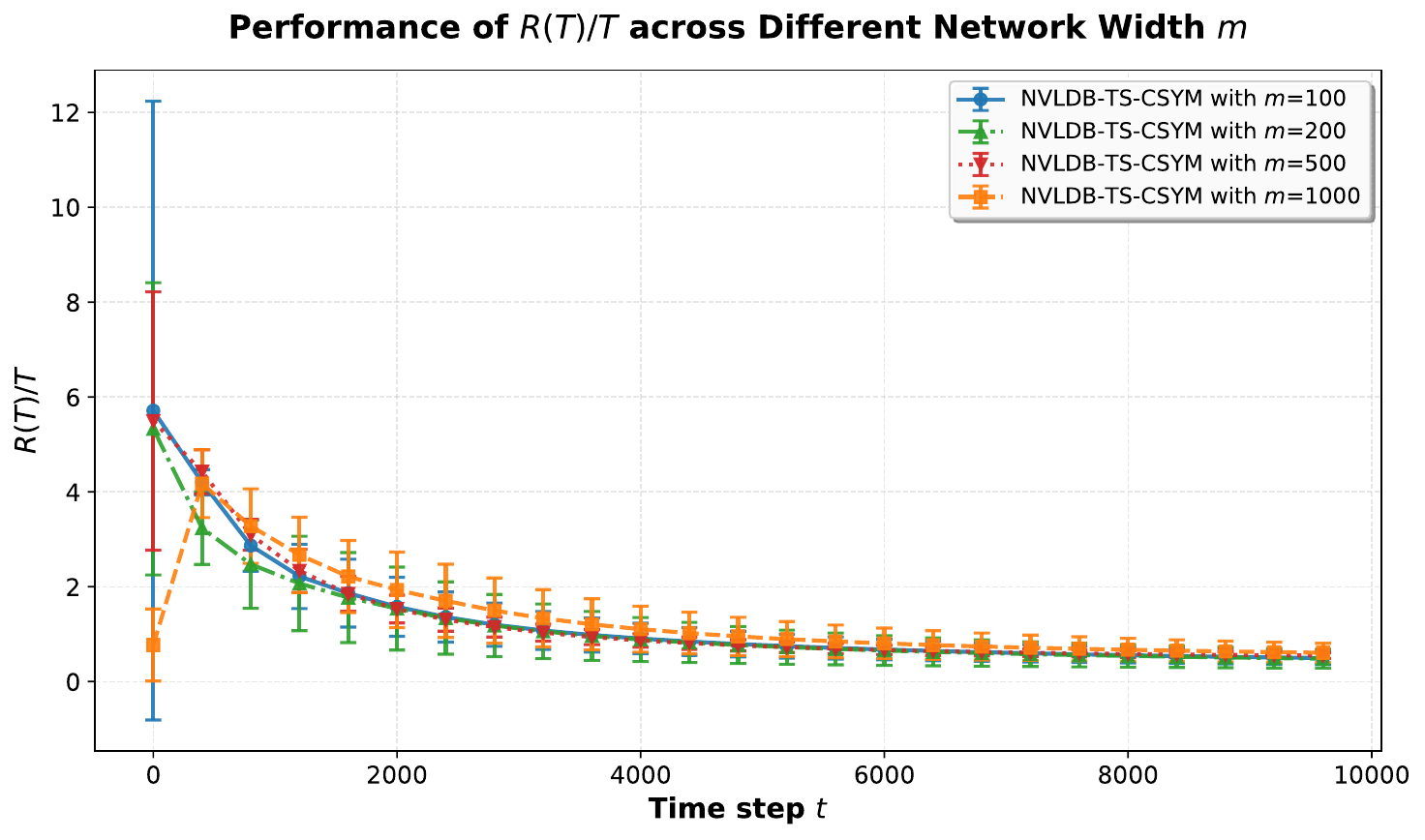}\\
    {\small (d): TS-ASYM} & {\small (e): TS-OSYM} & {\small (f): TS-CSYM}
    \end{tabular}
    \caption{\small Cumulative average regret across different network widths ($\wid \in \{500, 1000\}$) on the square utility task with $\dim=5$ and $K=5$. All \term variants maintain sublinear regret as the network width increases. Each experiment is repeated across 20 random seeds.}
    \label{fig:width_avg}
\end{figure*}

\begin{figure*}[h!]
    \centering
    \begin{tabular}{ccc}
    \includegraphics[width=0.3\columnwidth]{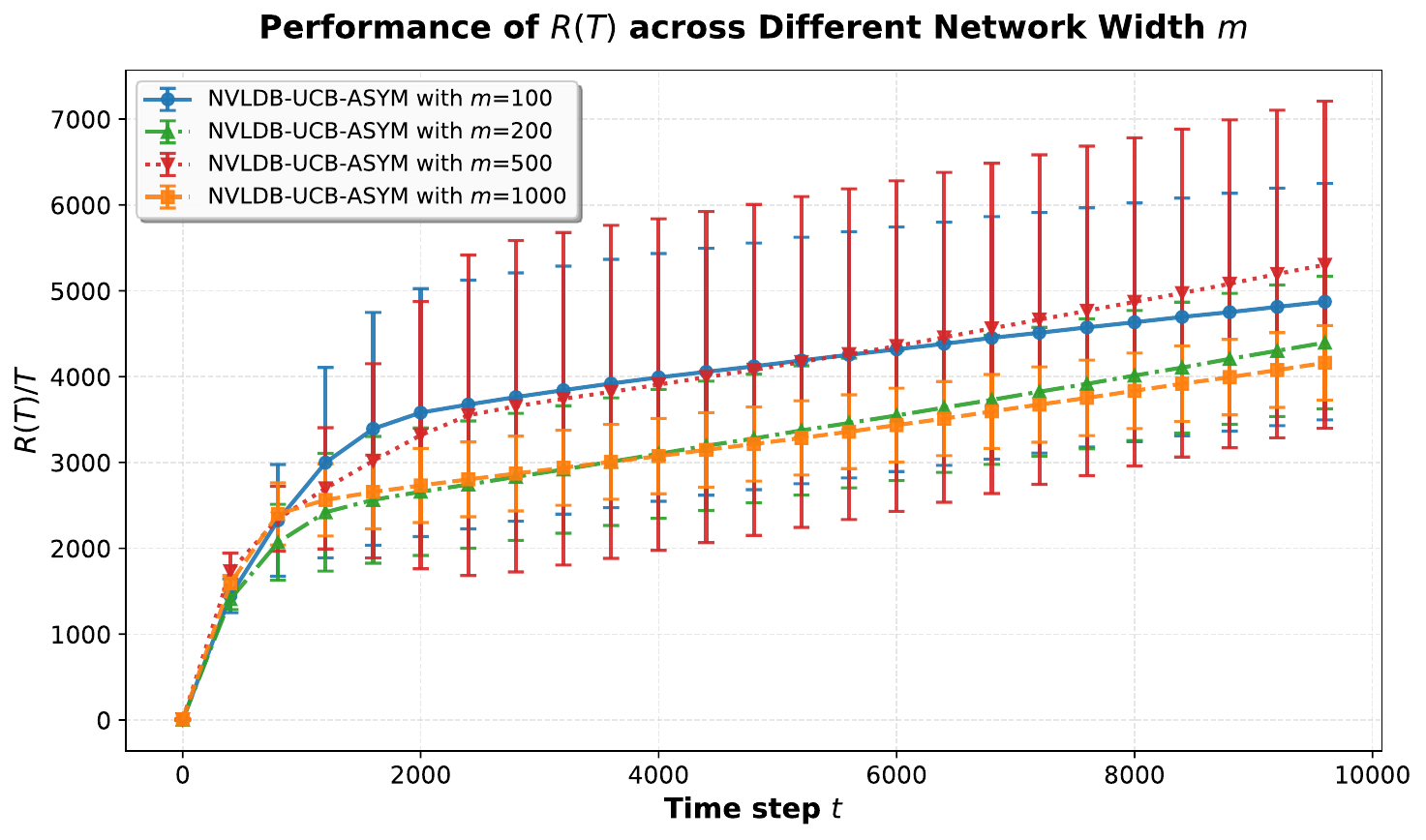}&
    \includegraphics[width=0.3\columnwidth]{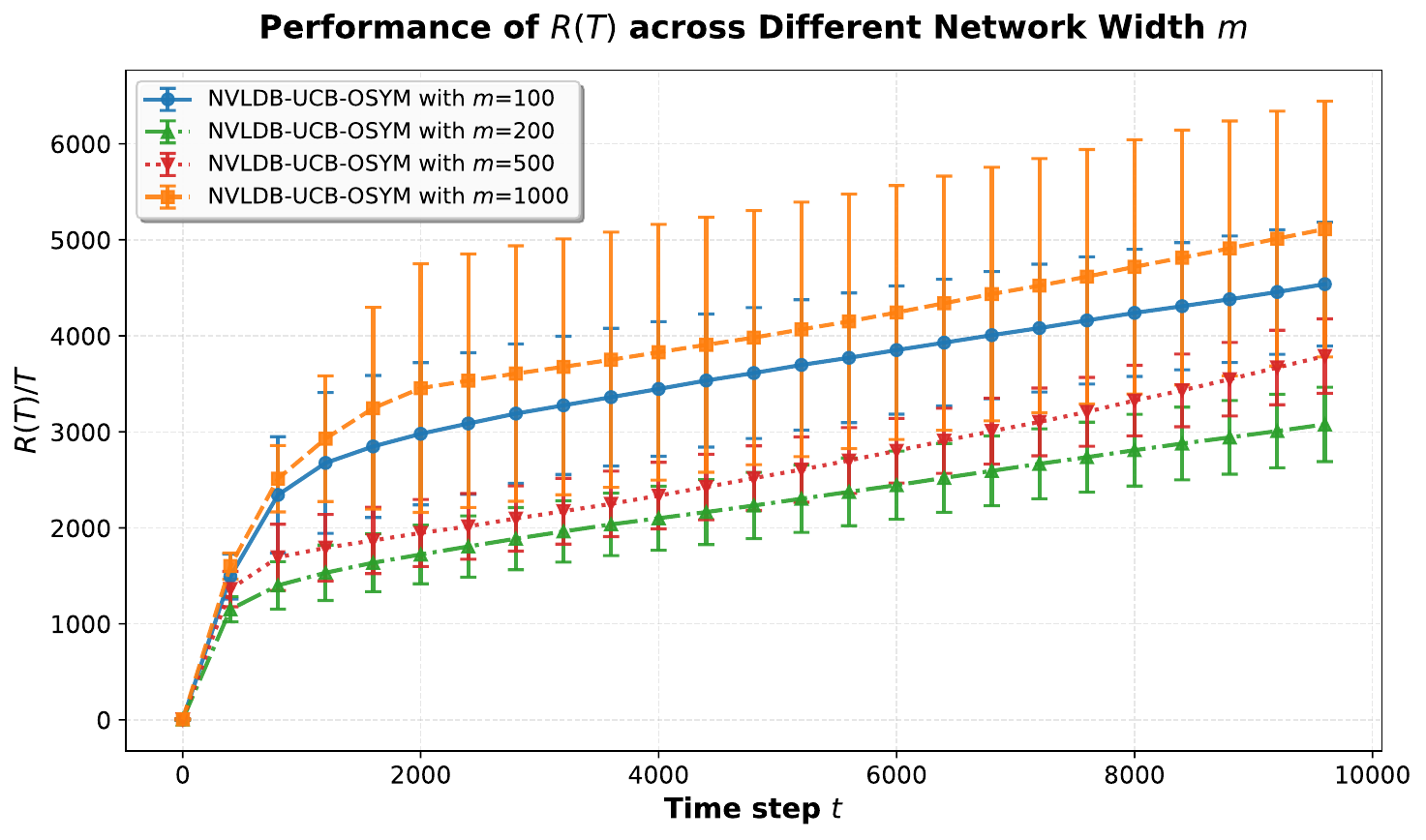}&
    \includegraphics[width=0.3\columnwidth]{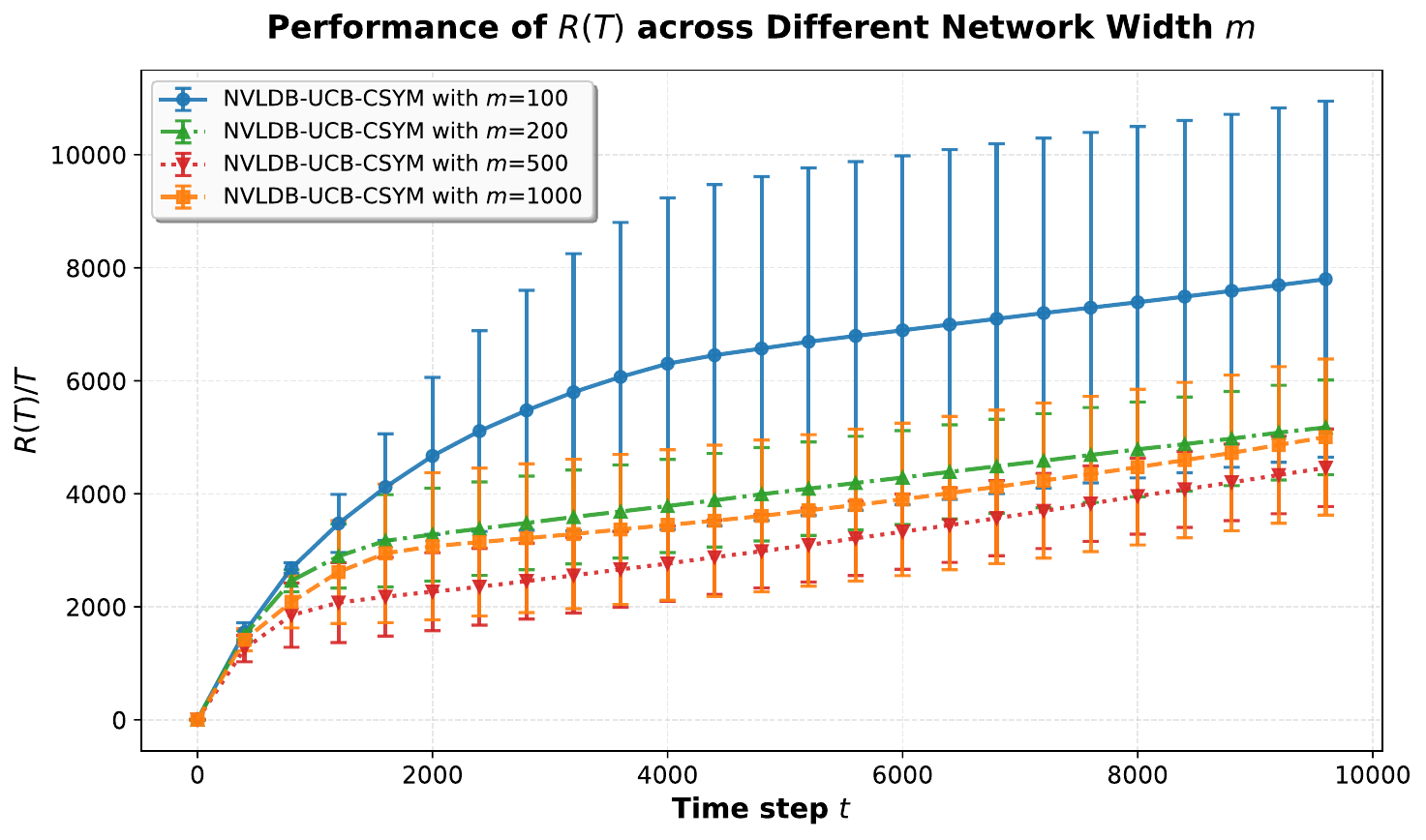}\\
    {\small (a): UCB-ASYM} & {\small (b): UCB-OSYM} & {\small (c): UCB-CSYM} \\[6pt]
    \includegraphics[width=0.3\columnwidth]{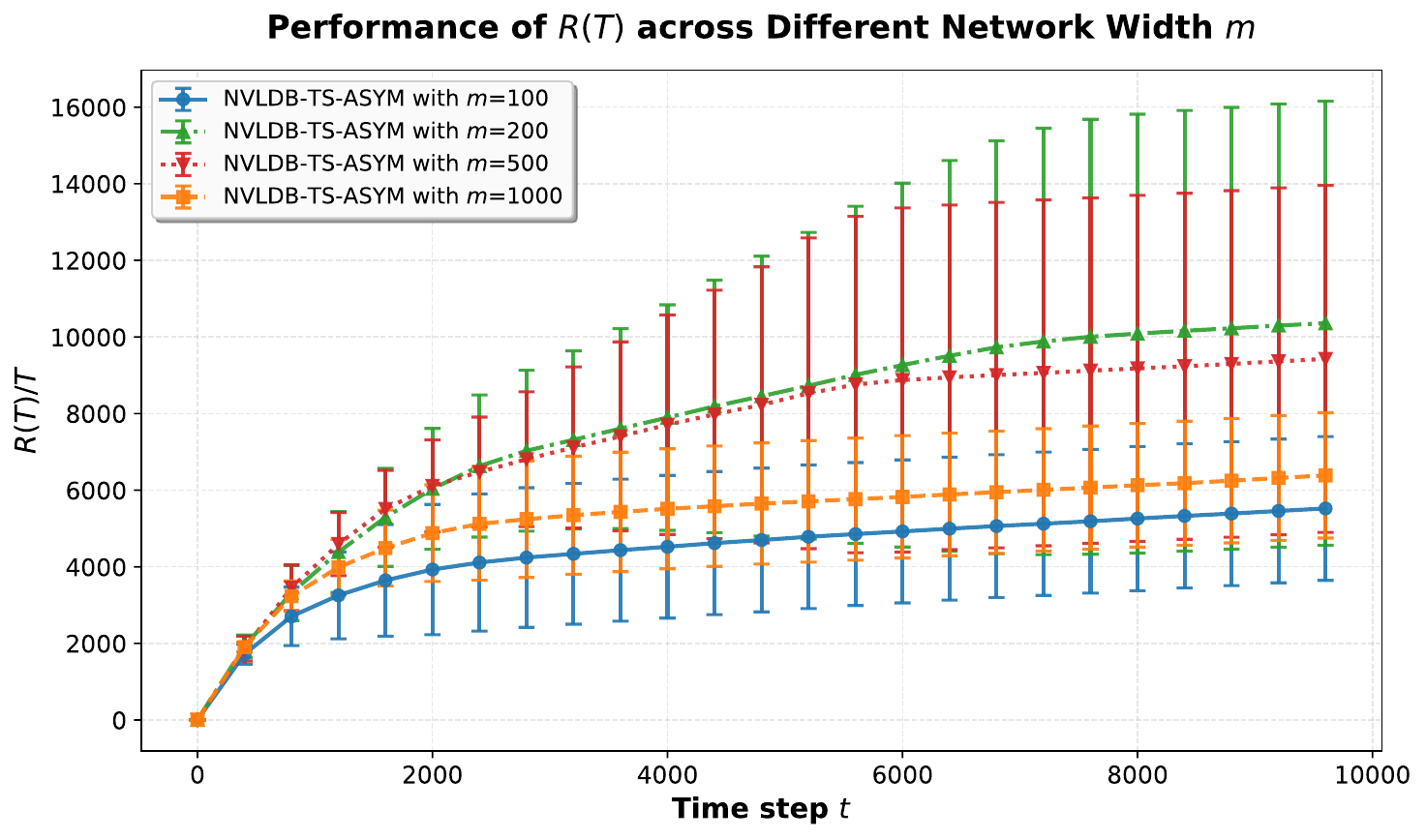}&
    \includegraphics[width=0.3\columnwidth]{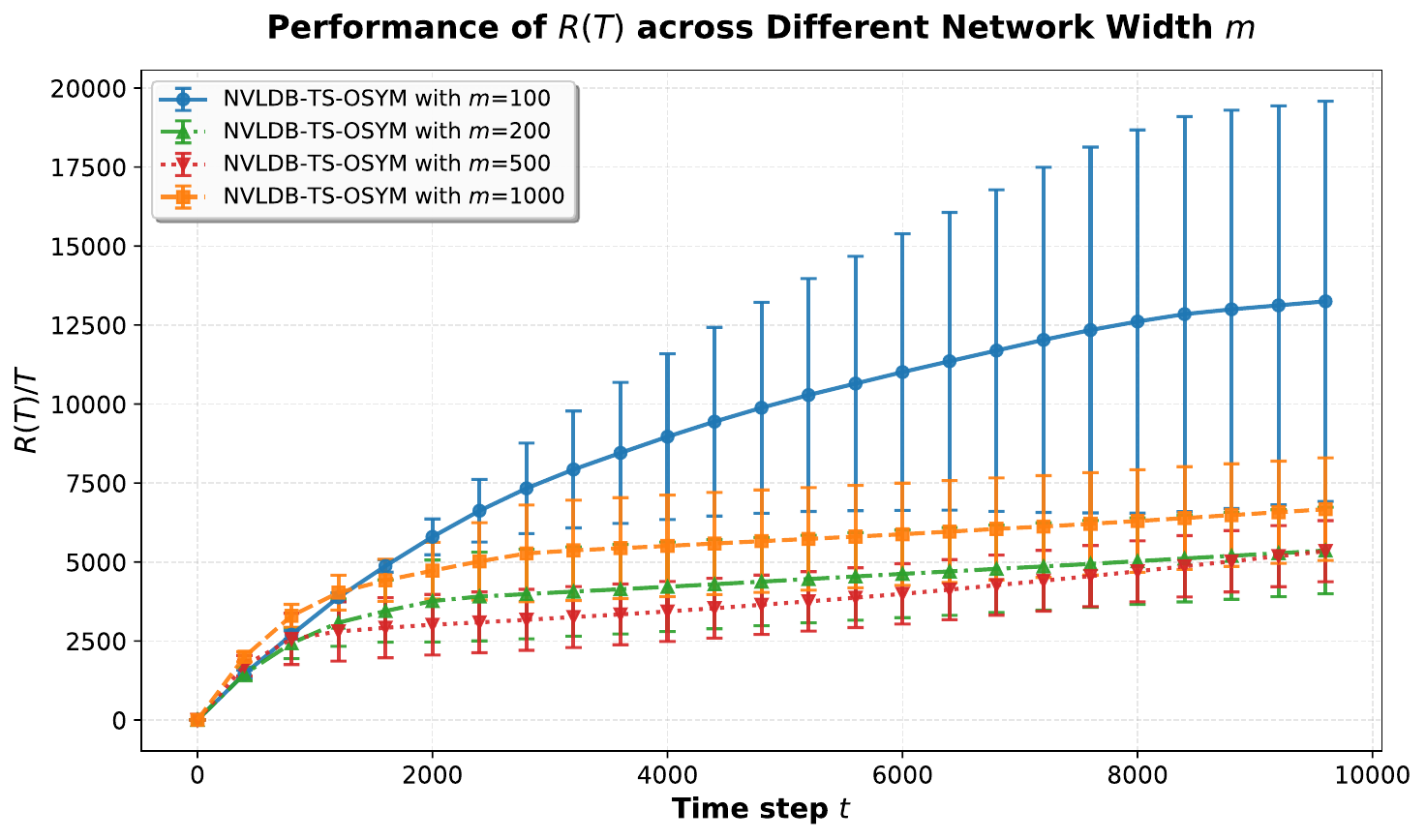}&
    \includegraphics[width=0.3\columnwidth]{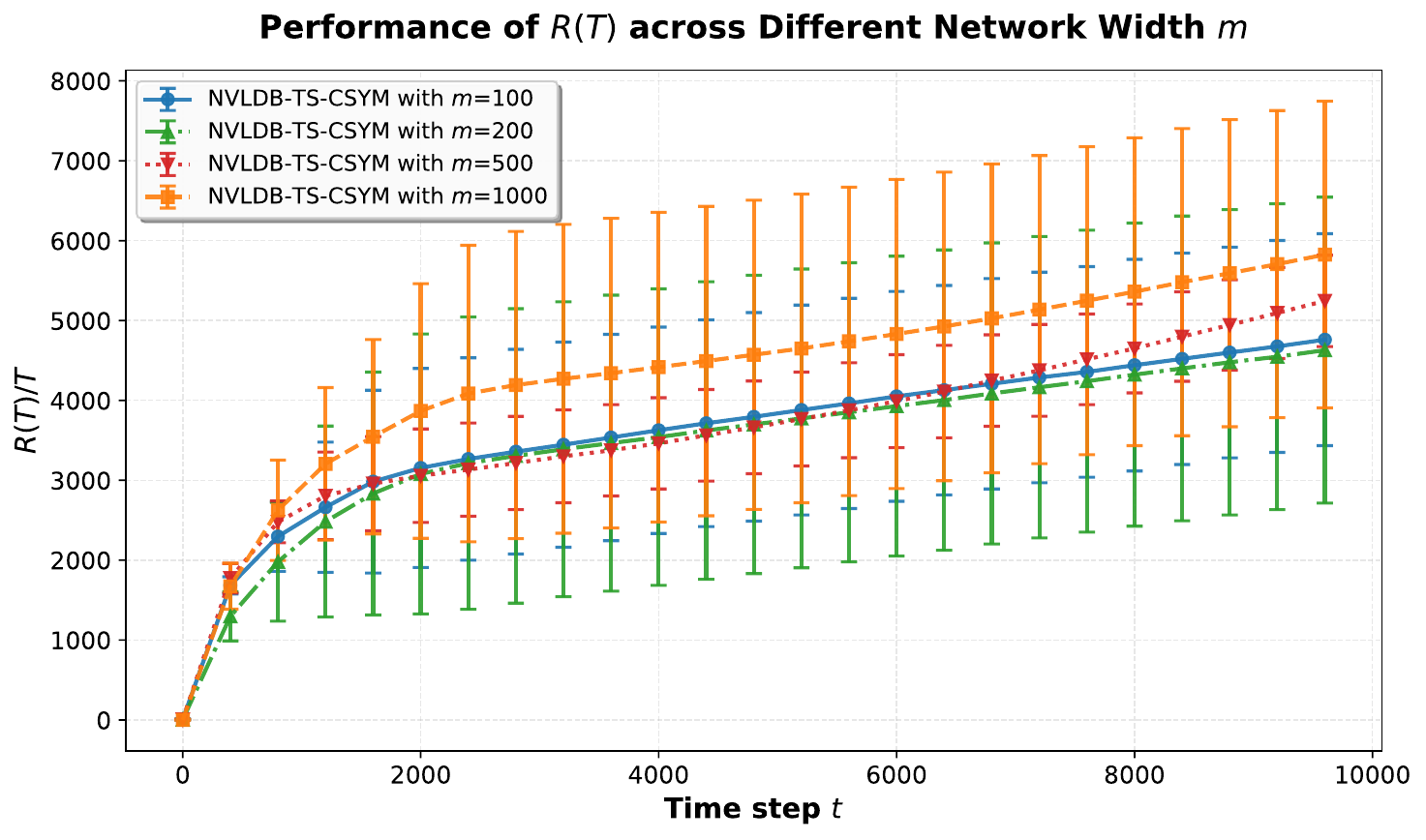}\\
    {\small (d): TS-ASYM} & {\small (e): TS-OSYM} & {\small (f): TS-CSYM}
    \end{tabular}
    \caption{\small Cumulative total regret across different network widths ($\wid \in \{500, 1000\}$) on the square utility task with $\dim=5$ and $K=5$. Each experiment is repeated across 20 random seeds.}
    \label{fig:width_total}
\end{figure*}

\subsection{High-Dimensional Experiments}\label{sec:high_dim}
To evaluate scalability in higher dimensions, we conducted experiments with $\dim=20$ and $K=5$ over $T=10{,}000$ rounds on the square utility task, significantly exceeding the dimensionality used in prior works~\citep{ndb, bui2024variance}. \cref{tab:high_dim} reports the $R(T)/T$ values, and \cref{fig:high_dim} visualizes the cumulative average and total regret.

\begin{table}[h]
\centering
\caption{$R(T)/T$ values with $\dim=20$ and $K=5$ over $T=10{,}000$ rounds on the square utility task. All tested algorithms exhibit steadily decreasing $R(T)/T$, following $\widetilde{\mathcal{O}}(T^{-1/2})$ convergence. Each experiment is repeated across 20 random seeds.}
\label{tab:high_dim}
\resizebox{0.8\textwidth}{!}{
\begin{tabular}{lccc}
\toprule
Algorithm & $R(2000)/2000$ & $R(6000)/6000$ & $R(10000)/10000$ \\
\midrule
NVLDB-UCB-ASYM & $9.33 \pm 1.89$ & $6.79 \pm 1.09$ & $6.05 \pm 0.74$ \\
NVLDB-UCB-OSYM & $9.49 \pm 2.05$ & $6.82 \pm 1.11$ & $6.07 \pm 0.83$ \\
NVLDB-UCB-CSYM & $10.76 \pm 1.29$ & $7.69 \pm 0.80$ & $6.73 \pm 0.58$ \\
\bottomrule
\end{tabular}
}
\end{table}

\begin{figure*}[h!]
    \centering
    \begin{tabular}{cc}
    \includegraphics[width=0.4\columnwidth]{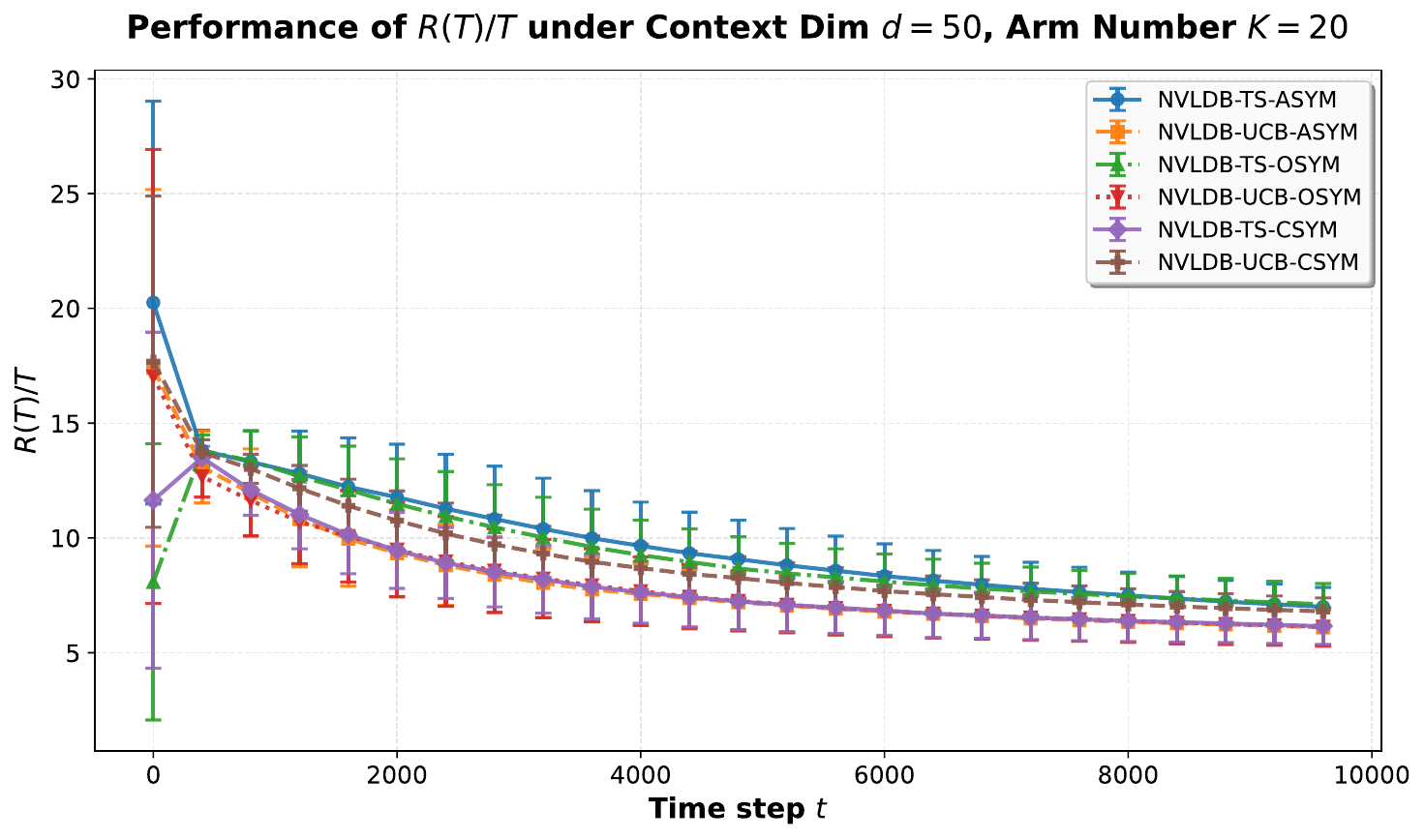}&
    \includegraphics[width=0.4\columnwidth]{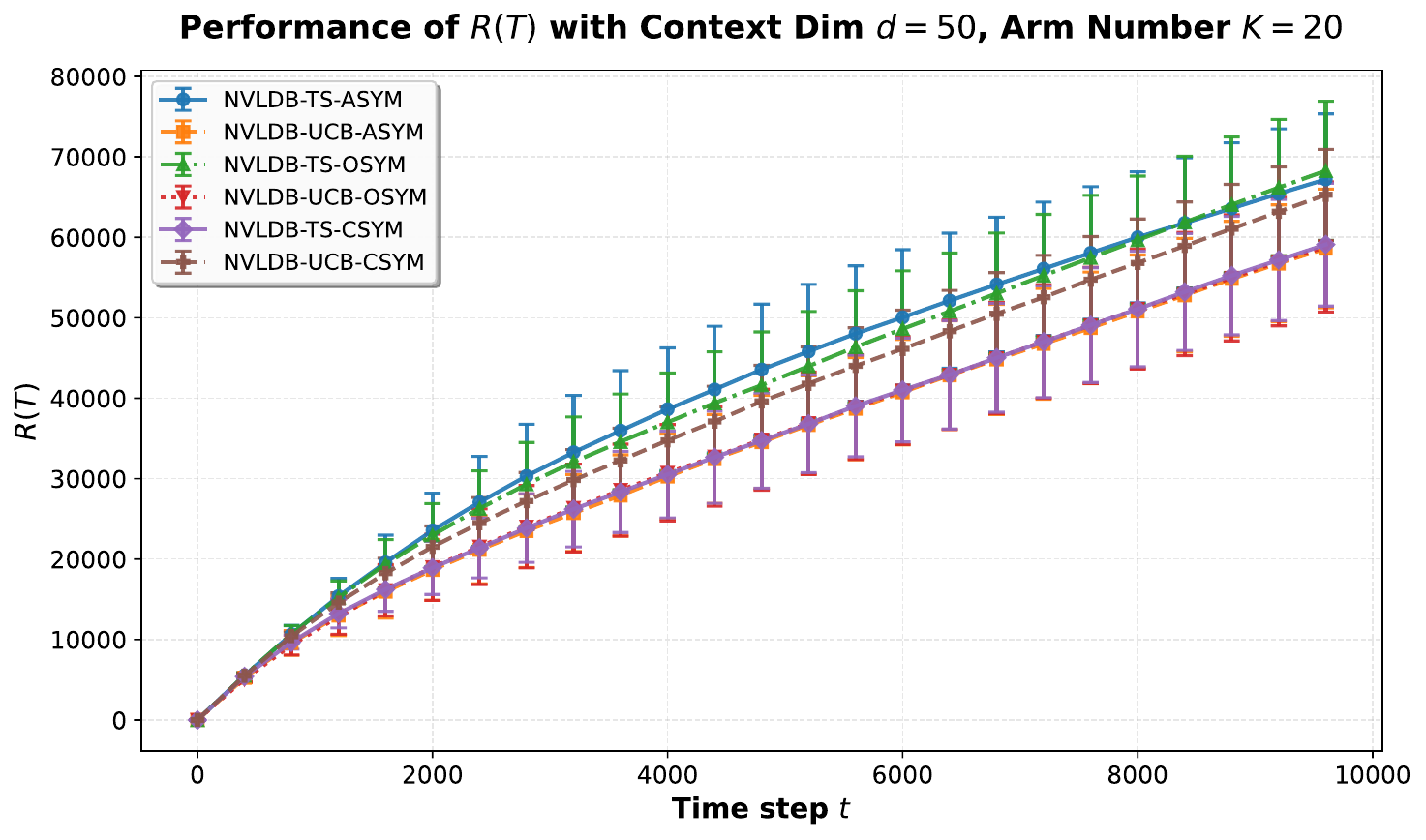}\\
    {\small (a): Cumulative average regret} & {\small (b): Cumulative total regret}
    \end{tabular}
    \caption{\small Cumulative average and total regret with $\dim=20$ and $K=5$ on the square utility task over $T=10{,}000$ rounds. All \term variants maintain sublinear regret in the high-dimensional regime. Each experiment is repeated across 20 random seeds.}
    \label{fig:high_dim}
\end{figure*}
}
\end{document}